\newif\ifdraft\drafttrue
\documentclass{article}

\usepackage[utf8]{inputenc} % allow utf-8 input
\usepackage[T1]{fontenc}    % use 8-bit T1 fonts
\usepackage[hidelinks]{hyperref}       % hyperlinks
\usepackage{url}            % simple URL typesetting
\usepackage{booktabs}       % professional-quality tables
\usepackage{amsfonts}       % blackboard math symbols
\usepackage{nicefrac}       % compact symbols for 1/2, etc.
\usepackage{microtype}      % microtypography
\usepackage{xcolor}         % colors
\usepackage{array}          % 
\usepackage{svg}

\usepackage[noend]{algorithmic}
\usepackage{amsmath,amsthm,amssymb, bbm}
\usepackage{authblk}
\usepackage{array, multirow, adjustbox}
\usepackage{blindtext}
\usepackage{bm}
\usepackage{booktabs} % commands to create good-looking tables
\usepackage{comment}
\usepackage{enumitem}
\usepackage{fullpage}
\usepackage{graphicx}
\usepackage{hyperref}
\usepackage{multirow}
\usepackage{multicol}
\usepackage[square, comma, numbers,sort&compress]{natbib}
\usepackage{subcaption}
\usepackage{tikz} % nice language for creating drawings and diagrams
\usepackage{url}
\usepackage{xcolor}

\usepackage[ruled]{algorithm2e}

\SetAlFnt{\small}
\SetAlCapFnt{\small}
\SetAlCapNameFnt{\small}
\SetAlCapHSkip{0pt}
\IncMargin{-\parindent}

\newcommand{\RR}{\mathbb{R}} % real numbers
\newcommand{\robx}{x^*}
\newcommand{\xz}{x_0} %x_0
\newcommand{\thetaz}{\theta_0} %\theta_0
\newcommand{\newx}{x} %modficiation to x_0
\newcommand{\newtheta}{\theta} %midification to theta_0
\newcommand{\ei}{\bm{e_i}}

\DeclareMathOperator*{\argmax}{arg\,max}
\DeclareMathOperator*{\argmin}{arg\,min}
\DeclareMathOperator{\sign}{sgn}

\theoremstyle{plain}
\newtheorem{theorem}{Theorem}
\newtheorem{proposition}{Proposition}

\theoremstyle{definition}

\theoremstyle{remark}
\newtheorem{remark}{Remark}

\usepackage{xcolor}

\date{}
\title{Optimal Robust Recourse with $L^p$-Bounded Model Change}
\author{Phone Kyaw, Kshitij Kayastha, Shahin Jabbari}
\affil{Drexel University}

\begin{document}
\maketitle
\begin{abstract}
Recourse provides individuals who received undesirable labels (e.g., denied a loan) from algorithmic decision-making systems with a minimum-cost improvement suggestion to achieve the desired outcome. However, in practice, models often get updated to reflect changes in the data distribution or environment, invalidating the recourse recommendations (i.e., following the recourse will not lead to the desirable outcome). The robust recourse literature addresses this issue by providing a framework for computing recourses whose validity is resilient to slight changes in the model. However, since the optimization problem of computing robust recourse is non-convex (even for linear models), most of the current approaches do not have any theoretical guarantee on the optimality of the recourse. Recent work by~\citet{KayasthaGJ24} provides the first \emph{provably} optimal algorithm for robust recourse with respect to generalized linear models when the model changes are measured using the $L^{\infty}$ norm. However, using the $L^{\infty}$ norm can lead to recourse solutions with a high price. To address this shortcoming, we consider more constrained model changes defined by the $L^p$ norm, where $p\geq 1$ but $p\neq \infty$, and provide a new algorithm that provably computes the optimal robust recourse for generalized linear models. Empirically, for both linear and non-linear models, we demonstrate that our algorithm achieves a significantly lower price of recourse (up to several orders of magnitude) compared to prior work and also exhibits a better trade-off between the implementation cost of recourse and its validity. Our empirical analysis also illustrates that our approach provides more sparse recourses compared to prior work and remains resilient to post-processing approaches that guarantee feasibility.
\end{abstract}

\section{Introduction}
\label{sec:intro}
Algorithmic recourse~\citep{WachterMR18, UstunSL19} provides individuals who received undesirable labels from machine learning systems (e.g., denied loans) with a minimum-cost suggestion to obtain the desired outcome. Most prior work on recourse assumes models are fixed, but in practice, they are periodically updated~\citep{UpadhyayJL21, HorowitzR23}, potentially invalidating recourse~\citep{Dominguez-OlmedoKS22} (i.e., following the recourse will not lead to the desirable outcome). To address this shortcoming,~\citet{UpadhyayJL21} introduced a framework for recourse that is robust to adversarial model changes, and proposed an algorithm called ROAR to compute robust recourse. Empirically, this robustness can substantially increase the price of computing recourse by requiring a high implementation cost to ensure that recourse validity is resilient to model change~\citep{PawelczykDHKL23}.

Since the problem of computing a robust recourse is non-convex (even for linear models), it is unclear whether these high prices are inherent or the result of sub-optimality of existing approaches. Recently,~\citet{KayasthaGJ24} provided the first optimal algorithm for robust recourse for generalized linear models and adversarial model changes measured by bounding the $L^\infty$ norm of the difference between initial and changed models. While they showed this optimal algorithm can lower the price of recourse compared to prior non-optimal approaches~\citep{UpadhyayJL21, NguyenBN+22}, they also showed that, in some cases, this price can be much higher compared to the non-robust recourse. To lower this worst-case price, they studied the robust recourse problem through the lens of the learning-augmented framework. They empirically demonstrated that access to (potentially unreliable) predictions about the realized future model can be used to lower the price of recourse. However, they do not offer any concrete approach on how reasonable predictions about future model changes can be generated. Moreover, models that are close to each other in $L^\infty$ norm can exhibit significant behavior differences, especially for larger models that have a large number of parameters.

\noindent\textbf{Our Contributions and Results}
Our main goal is to understand the true price of recourse for more restricted adversarial model changes. In particular, we measure model changes by bounding the $L^p$ norm of the difference between initial and changed models, where $p\geq 1$ but $ p\neq \infty$. We provide a new algorithm that provably computes the optimal robust recourse for generalized linear models for this type of model change. The key insight in the design of our algorithm is the observation that the optimal solution of the non-convex optimization problem of computing robust recourse can be computed by solving $O(d)$ convex problems with linear constraints, where $d$ is the number of features. 

Empirically, on real-world datasets, we observe that our algorithm can lower the price of recourse (sometimes by several orders of magnitude) compared to when model change is measured using $L^{\infty}$ norm, as well as existing algorithms for the same setting as ours. We also break down the price of recourse and study the frontier of the trade-off between implementation costs of recourse and its validity. For linear models, our results indicate that our algorithm can achieve high validity while prior approaches either do not reach high validity or achieve the same level of validity with a substantially higher implementation cost.  For non-linear models, we show that our algorithm performs on par with and often better than prior approaches. Our empirical analysis also demonstrates that our algorithm provides more sparse recourses compared to prior work and remains resilient to post-processing approaches that guarantee feasibility.
\section{Related Work}
\label{sec:related}
Recourse is a type of post-hoc explanation~\cite{RibeiroSG16, SmilkovTKVW17, LundbergL17}, providing counterfactual explanations by finding the lowest-cost modification that changes the label of the predictive model~\citep{WachterMR18, UstunSL19}. Since its introduction, there have been many formulations for recourse by focusing on different assumptions on cost functions and model classes (see e.g.,~\citep{LooverenK21, RawalL20, SlackHLS21, PawelczykDHKL23, GargNS25, KanamoriTKI24, BewleyAM+24} and~\citep{VermaDH20} for a survey). Subsequent work has addressed many additional aspects of algorithmic recourse, such as assumptions and implications~\citep{BarocasSR20, VenkatasubramanianA20, FokkemaGE24, GaoL23}, attainability~\citep{JoshiKV+19, PawelczykBK20b, KarimiBBV20, UstunSL19}, diversity~\cite{LeofanteP24, MothilalST20}, causality~\citep{KarimiKSV20, KonigFG23, KonigFF+25}, fairness~\citep{GuldoganZS+23, GuptaNR+19, HeidariNG19}, repeated dynamics~\citep{FonsecaBABS23,BellFA+24,EhyaeiSS24}, and temporal data~\citep{BuligaDG+25}. 

The original formulation assumes the model is fixed, though in practice, models get updated, necessitating robustness to these model shifts when generating recourse. The formulation of robust recourse is introduced by~\citet{UpadhyayJL21}, and they also proposed the first algorithm for this formulation called \emph{RObust Algorithmic Recourse} (ROAR). To improve the performance of ROAR for non-linear models, \citet{NguyenBN+22} propose a new framework, \emph{Robust Bayesian Recourse} (RBR), along with an algorithm tailored for data shifts rather than model shifts in this setting. We use the same robust recourse framework as~\cite{UpadhyayJL21} and compare our algorithms with both ROAR and RBR. The closest related work to us is by~\citep{KayasthaGJ24}, which solves our exact problem where the neighborhood around the initial model is defined using $L^{\infty}$ norm. We directly compare our algorithm for the other $L^p$ norms to this algorithm.  Very recently,~\citet{TurbalVS25} proposes provably robust algorithms for recourse under predictive multiplicity~\citep{MarxCU20} when the model class is given by an ellipsoidal approximation of the Rashomon set~\citep{Breiman01}, that is, models with similar predictive performance on the data distribution~\citep{SemenovaR19, XinZC+22}. In contrast, our setting allows adversarial models that may differ substantially in their performance on the data distribution.

Analogous to the literature on algorithmic recourse, robustness has also been studied for different model classes, cost functions, and model change formulation~\citep{LeofanteP24,YetukuriHVUL24, JiangLL+23, JiangLR+24b, NguyenBN23, DuttaLMTM22, MochaourabSG+21, HammanNMMD23, BlackWF22, Dominguez-OlmedoKS22}.  Another closely related algorithm called \emph{Robust ReCourse Neural Network} (RoCourseNet) jointly optimizes the model and robust recourse~\citep{GuoJC+23}. We cannot directly compare our algorithms to this approach, as in our setting, the predictive model is given and cannot be changed.  See~\citep{JiangLR+24a} for a recent survey on robust recourse.

\citet{PawelczykAJUL22} demonstrates the connections between recourse and adversarial training~\cite{MadryMSTV18, WongK18}. In fact, ROAR~\cite{UpadhyayJL21} is inspired by gradient-based approaches that are designed for adversarial training. The convergence properties of such gradient-based algorithms (under specific assumptions) are studied extensively in prior work (see e.g.~\cite{WangM0YZG19}). Theoretical guarantees for adversarial training are studied under specific data distributions and model classes. For example, when data is generated from a mixture of Gaussians, the optimal robust models are linear~\cite{li2019inductive}.~\citet{AwasthiDV19} provides a learning algorithm to learn a robust linear classifier under the realizability assumption and proves computational hardness results for learning robust degree-2 polynomial threshold functions. Beyond linear models, these theoretical studies have been extended to other model classes such as decision trees~\cite{vos2021efficient, vos2021roct}. To the best of our knowledge, none of these directly address our problem, although we empirically compare our approach to ROAR, which employs the typical adversarial training template.
\section{Problem Formulation}
\label{sec:prelim}
Consider an \emph{initial} predictive model $f_{\thetaz}: \mathcal{X} \to \mathcal{Y}$, parameterized by $\thetaz \in \Theta \subseteq \RR^k$, mapping $d$-dimensional instances $\mathcal{X} \subseteq \RR^d$ to binary outcomes $\mathcal{Y}=\{0,1\}$. We assume labels $0$ and $1$ represent undesirable and desirable outcomes (e.g., loan denial/approval). For convenience, we refer to the $\thetaz$ as parameters of the model or the model interchangeably. For any instance $\xz\in X$ with $f_{\thetaz}(\xz)=0$, the goal in \emph{algorithmic recourse} is to compute the least costly modification $\newx$ of $\xz$ such that $f_{\thetaz}(\newx)=1$.  Given
a cost function $c:\mathcal{X}\times\mathcal{X}\to \RR_{+}$ measuring the cost of transforming an instance to another one, the recourse can be computed using the following relaxed optimization problem~\cite{UpadhyayJL21}:
\begin{equation}
    \label{eq:recourse}
    \arg\min_{\newx\in \mathcal{X}} \ell\!\left(f_{\thetaz}(\newx),1\right) + \lambda\, c\!\left(\newx,\xz\right),
\end{equation}
where $\ell:\RR\times\RR\to\RR_+$ is a convex and differentiable loss function, penalizing the difference between the prediction of the model $\thetaz$ on the provided recourse and the desirable label of $1$. We refer to $f_{\thetaz}(\newx)$ as \emph{validity} of recourse with respect to model $\thetaz$. The parameter $\lambda\ge 0$ balances the validity of recourse and its \emph{implementation} cost~\cite{UpadhyayJL21}. Following many prior work~\cite{UpadhyayJL21,RawalL20, KayasthaGJ24}, we use $c(x,x')=\|x-x'\|_1$ (see Section~\ref{sec:discussion}). Following~\citep{KayasthaGJ24}, for a recourse $x$ and model $\theta$, we define 
\begin{equation}
    \label{eq:j}
    J(x,\theta):=\ell\!\left(f_{\theta}(\newx),1\right) + \lambda\, c\!\left(\newx,\xz\right)    
\end{equation}
and refer to $J(x,\theta)$ as the \emph{price} of recourse $x$ with respect to model $\theta$.  Using this notation, the recourse optimization problem can be written as $\arg\min_{\newx\in \mathcal{X}} J(\newx,\thetaz)$.

Equation~\eqref{eq:recourse} assumes a fixed model, but models are often retrained~\citep{UpadhyayJL21}, potentially invalidating prior recourses~\citep{DuttaLMTM22,BlackWF22} (i.e., following the recourse does not lead to the desirable outcome). The formulation of \emph{robust recourse}~\citep{UpadhyayJL21} accounts for these model changes as follows: for any fixed $\alpha > 0$, and $p\geq 1$ we define a neighborhood $\Theta_p^{\alpha}(\thetaz)$ around $\thetaz$ as follows: $\Theta_p^{\alpha}(\thetaz)=\{\theta \text{ such that } \|\theta-\thetaz\|_p\leq \alpha\} \subseteq \Theta$. The goal in robust recourse is to compute a recourse $\robx$ to minimize the price against the worst-case model $\newtheta \in \Theta_p^\alpha(\thetaz)$:
\begin{equation}
\label{eq:xr}
    \robx \in \arg\min_{\newx\in\mathcal{X}} \max_{\newtheta\in \Theta_p^{\alpha}(\thetaz)} J(\newx,\newtheta).
\end{equation}
For any recourse $x$, we use $\theta^*(x)$ to denote the optimal adversarial model, i.e., 
\begin{equation}
\label{eq:adversary}
\theta^*(x)=\max_{\newtheta\in \Theta_p^{\alpha}(\thetaz)} J(x, \newtheta).    
\end{equation}
We use $J^*$ to denote the optimal price of recourse i.e., $J^*=J(\robx, \theta^*(\robx))$ where $\robx$ is defined in Equation~\ref{eq:xr}.
\section{Algorithms and Analysis}
\label{sec:alg-analysis}
In this section, we present algorithms to solve the optimization problem in Equation~\ref{eq:xr} for various $L^p$ norms used to define the neighborhood around the initial model. Note that this optimization problem is non-convex for all $p\geq 1$ values even when the initial model $\thetaz$ is a linear function (see Proposition~\ref{pro:non-convexity} in Appendix~\ref{sec:app-algo} for a formal statement and proof).

We begin by providing an algorithm for the case where $ p \geq 1$ but $ p \neq \infty$. Let $\ei$ denote the $d$-dimensional unit vector and for each $\theta\in\Theta$ define 
\begin{equation}
    \label{eq:theta-pm}
    \Theta^{\pm}(\theta)=\{\theta' \mid \theta' = \theta + \alpha \ei \text{ or } \theta' = \theta - \alpha \ei, \forall i \in [d]\},
\end{equation}
i.e., the set of models that are different than $\theta$ only in a dimension, and this difference is exactly $\alpha$ (either positively or negatively). Note that $\Theta^{\pm}(\theta)\subset \Theta_{p}^\alpha(\theta)$ for all $p\geq 1$. 

\begin{algorithm}[ht!]
\caption{$p\geq 1, p\ne\infty$\label{alg:l1-l1}
}
\textbf{Input}  : $\xz$, $\thetaz$, $\ell$, $c$, $\alpha$\\
\textbf{Output}: $\newx$ 
\begin{algorithmic}[1] %S[1] enables line numbers
\STATE $\theta \gets $ linear approximation of $f_{\thetaz}$ at $\xz$
\STATE $\Theta^{\pm}(\theta)\gets$ According to Equation~\ref{eq:theta-pm}
\STATE $\newx \gets \xz$ \hfill{$\triangleright$Initialize the recourse}
\STATE $J^* \gets +\infty$ \hfill{$\triangleright$Initialize the price}
\FOR{$\theta'\in \Theta^{\pm}(\theta)$}
\STATE $x' \gets $ solution of optimization problem in Equation~\ref{eq:optimize-inside} using projected subgradient descent
\IF{$J(x', \theta') < J^*$}
    \STATE $\newx\gets x'$ \hfill{$\triangleright$Update the recourse}
    \STATE $J^*\gets J(x', \theta')$\hfill{$\triangleright$Update the price}
\ENDIF
\ENDFOR
\RETURN $\newx$
\end{algorithmic}
\end{algorithm}

The high-level idea of our algorithm is as follows. The algorithm first approximates the initial model $f_{\thetaz}$ at $\xz$ by a linear function. This can be done, for example, by using LIME~\cite{RibeiroSG16} and is commonly used in prior work~\cite{UpadhyayJL21, KayasthaGJ24}. Let $\theta$ denote the parameters of this linear approximation. The algorithm then solves $2d$ optimization problems separately, one for each $\theta'\in \Theta^{\pm}(\theta)$. These optimization problems have the form 
\begin{align}
x'\in &\arg\min_{\newx\in \mathcal{X}} \ell\!\left(f_{\theta'}(\newx),1\right) + \lambda\, c\!\left(\newx,\xz\right),\nonumber\\
&\text{subject to } \theta'\in \theta^*(x)\label{eq:optimize-inside}.
\end{align}
In these optimization problems, the model is assumed to be fixed, but the constraint restricts the recourse such that $\theta'$ belongs to the set of optimal adversarial models for the recourse. This idea is formalized in Algorithm~\ref{alg:l1-l1}.

We next analyze the theoretical properties of Algorithm~\ref{alg:l1-l1} by focusing on generalized linear models. A model $f_\theta$ is generalized linear if $f_\theta(x):= g \circ h_\theta(x)$, where $h_\theta(x)=\theta^{\top}x$ is a linear function and $g: \RR \to [0,1]$ is a non-decreasing function mapping the outputs of $h_\theta$ to probabilities. For example, setting $g$ to the sigmoid function will recover logistic regression. Our main result for the analysis of Algorithm~\ref{alg:l1-l1} for generalized linear models is as follows. 

\begin{theorem}
\label{thm:opt-l1}
    If $f_{\thetaz}$ is a generalized linear model, then Algorithm~\ref{alg:l1-l1} returns a recourse $x$ that minimizes Equation~\ref{eq:xr} for $p\geq 1$ and $p\ne\infty$ in time polynomial in the number of dimensions $d$.
\end{theorem}
\begin{proof}[Sketch of the Proof]
    First note that, for any linear model $\thetaz$, the approximation in line 1 of the algorithm returns $\theta=\thetaz$. Moreover, for any recourse $x$, since $f_{\thetaz}$ is a generalized linear model, the adversarial $\theta^*(x)$ aims minimizes the dot product between $\thetaz$ and $x$ with the constraint that the selected adversarial model lies in $\Theta_{p}^{\alpha}(\thetaz)$. Avoiding tie-breaking, a simple strategy to compute this optimal is to select the dimension $i$ in which $|x[i]|$ is maximum and modify $\thetaz$ by adding $\alpha \sign(x[i])$ in that dimension. Therefore, $\theta^*(x)\in \Theta^{\pm}(\thetaz)$ for all $x$, meaning that it suffices to narrow down the choice of optimal adversarial models to $\Theta^{\pm}(\thetaz)$.

    Now observe that Algorithm~\ref{alg:l1-l1} computes the best recourse for each $\theta'\in \Theta^{\pm}$ and returns the recourse with the smallest price among these $2d$ optimization problems. So to complete the proof, it suffices to show that the optimization problem in Equation~\ref{eq:optimize-inside} can be solved efficiently by projected subgradient descent. First note that the objective in Equation~\ref{eq:optimize-inside} is convex (since $\ell$ is convex, $f_{\thetaz}$ is generalized linear, and the cost function $c$ is convex), though this optimization problem is not differentiable at all points due to the non-differentiability of the cost function $c$. Moreover, the constraints in Equation~\ref{eq:optimize-inside} are linear (because if $\theta'=\thetaz\pm \alpha e_i$ then, $\theta'\in \theta^*(x)$ is equivalent to $|x_i|\geq |x_j|$ for all $j\in [d]$ while ensuring $x_i$ has the correct sign). For this kind of optimization problem, projected subgradient descent will converge to the optimal solution~\cite{BV2014}. 

    Note that the running time is polynomial in $d$, since there are $2d$ optimization problems, and for each optimization problem, both the gradient computation and projection can be solved in time polynomial in $d$.
\end{proof}

We next focus on the case where $p=\infty$. Note that implementing a similar strategy as in Algorithm~\ref{alg:l1-l1} does not lead to an efficient algorithm since there are now $2^d$ possible candidates for the adversarial $\theta$. To overcome this difficulty,~\citet{KayasthaGJ24} proposes a greedy algorithm where, in each round, the algorithm selects a dimension to update that lowers the $J$ value the most, along with the degree of update in that dimension. If this update causes the recourse to change sign in that dimension, the algorithm only updates that dimension up to $0$. The algorithm then adjusts the adversarial $\theta$ if needed and continues until no other improvement is possible. Algorithm~\ref{alg:l1-linf} is a reproduction of this idea using our notation, where $\sign$ denotes the sign function.

\begin{algorithm}[ht!]
\caption{$p=\infty$~\cite{KayasthaGJ24} \label{alg:l1-linf}
}
\textbf{Input}  : $\xz$, $\thetaz$, $\ell$, $c$, $\alpha$\\
\textbf{Output}: $\newx$ 
\begin{algorithmic}[1] %S[1] enables line numbers
\STATE $\theta \gets $ linear approximation of $f_{\thetaz}$ at $\xz$
\STATE $\newx\gets \xz$ \hfill{$\triangleright$ Initialize the recourse}
\STATE $\theta'\gets \theta-\alpha \sign(\newx)$ \hfill{$\triangleright$ Initialize the adversarial model}
\STATE \textsc{Active}=$[d]$  \hfill{$\triangleright$Initialize the set of coordinates to update}
\WHILE{$\textsc{Active}\neq \emptyset$}
    \STATE $i \gets \argmax_{j\in \textsc{Active}}  |\theta'[j]|$\hfill{$\triangleright$ \hspace{-2mm} Next coordinate to update}
    \STATE $\Delta \gets \argmin_{\delta} J(\newx+\delta e_i,\theta')-J(\newx, \theta')$ \hfill{$\triangleright$ The best update for coordinate $i$ if we were allowed to only change the coordinate $i$ of $\newx$ for the current $\theta'$}
    \IF{$\Delta=0$}
        \STATE break \hfill{$\triangleright$ Terminate}
    \ENDIF
    \IF{$\sign(\newx[i]+\Delta) = \sign(\newx[i])$}
        \STATE $\newx[i]\gets \newx[i]+\Delta$ \hfill{$\triangleright$ Fully update the coordinate}
        \STATE break \hfill{$\triangleright$ Terminate}
    \ELSE
        \STATE $\newx[i] \gets 0$ \hfill{$\triangleright$ Update the coordinate but only until it reaches 0}
        \IF{$|\theta[i]| > \alpha$}
            \STATE $\theta'[i] \gets \theta[i] + \alpha \cdot \sign(\xz[i])$ \hfill{$\triangleright$ Modify $\theta'$} 
        \ELSE
            \STATE \textsc{Active} $\gets$ \textsc{Active} $\setminus \{i\}$ \hfill{$\triangleright$ Remove $i$ from the \textsc{Active} set}
        \ENDIF
    \ENDIF
\ENDWHILE
\RETURN $\newx$
\end{algorithmic}
\end{algorithm}

\citet{KayasthaGJ24} show that Algorithm~\ref{alg:l1-linf} efficiently computes the optimal recourse for generalized linear models. We restate their statement using our notation.
\begin{theorem}[\citep{KayasthaGJ24}]
\label{thm:opt-linf}
If $f_{\thetaz}$ is a generalized linear model, then Algorithm~\ref{alg:l1-linf} returns a recourse $\newx$ that minimizes Equation~\ref{eq:xr} for $p=\infty$ in polynomial time in $d$.
\end{theorem}
We refer the reader to the proof of Theorem 3 in~\citep{KayasthaGJ24} for the details.~\citet{KayasthaGJ24} show that the algorithm runs for $O(d)$ iterations. The computational complexity of each iteration is dominated by computing $\Delta$ in line 7, which corresponds to solving a one-dimensional convex problem and can also be solved analytically for specific loss functions. Hence, Algorithm~\ref{alg:l1-linf} runs in time polynomial in the number of dimensions. We empirically compare the running times of Algorithm~\ref{alg:l1-l1}~and~\ref{alg:l1-linf} (as well as baselines) in Appendix~\ref{sec:app-exp-trade-off}.

\begin{remark}
    While prior gradient-based approaches~\citep{UpadhyayJL21,NguyenBN+22} cannot guarantee optimality, Algorithm~\ref{alg:l1-linf}~\citep{KayasthaGJ24} is the first optimal algorithm for any robust recourse formulation with respect to $L^{\infty}$ model changes. Algorithm~\ref{alg:l1-l1} is the first optimal robust algorithm with respect to $L^{p}$ model changes where $p\geq 1$ and $p < \infty$. The optimality of both algorithms only holds for generalized linear models. For non-linear models, both algorithms first approximate the non-linear model locally, in the neighborhood of $\xz$, with a linear model. This idea was used in prior work~\citep{UpadhyayJL21, RawalL20} for algorithmic recourse. We evaluate the performance of both algorithms for non-linear models in Section~\ref{sec:exp}. 
\end{remark}

\begin{remark}
    Additionally, there may be constraints on the feasibility or actionability of recourse, and the data may contain categorical features. While neither of the algorithms (and many prior works such as ROAR~\cite{UpadhyayJL21}, RBR~\cite{NguyenBN23}, and RoCourseNet~\citep{GuoJC+23}) directly handle these constraints, the recourse output by these algorithms can be post-processed (e.g., by projecting them to the set of feasible values) to guarantee feasibility or valid categorical values. We study the effect of these post-processing approaches on the quality of recourse in Section~\ref{sec:exp-feasibility}.
\end{remark}
\section{Experiments}
\label{sec:exp}
\noindent\textbf{Datasets: }
We experimented on two real-world datasets: the Small Business Administration dataset~\citep{sba} and the German Credit dataset~\citep{germanuci}. The Small Business Administration (SBA) dataset contains the small business loans approved by the State of California from 1989 to 2004. The dataset includes 1159 data points, each with 28 features (such as business category, zip code, and number of jobs created by the business). The German Credit dataset contains information about loan applicants and binary labels to determine creditworthiness. The dataset consists of 1000 data points, each with 7 features (such as age, marital status, income, and credit duration).  

\noindent\textbf{Implementation Details: }
We used 5-fold cross-validation in all experiments and reported average values. 4 folds were used to train the initial model $\thetaz$, and the last fold was used to compute recourse (only for instances with label 0 under $\thetaz$). We trained two models as $\thetaz$: logistic regression (LR) and a 3-level neural network (NN) with 50, 100, and 200 nodes in each successive layer (which is the same architecture as ROAR~\citep{UpadhyayJL21}).

\renewcommand{\arraystretch}{1.6}
\begin{table*}[t]
\centering
\begin{adjustbox}{max width=\linewidth}
\huge{
\begin{tabular}{|c|cccc|cccc|}
    \hline
  \multicolumn{1}{|c|}{} & \multicolumn{4}{c|}{German (LR)} & \multicolumn{4}{c|}{Small Business Administration (LR)} \\ \hline
     \multicolumn{1}{|c}{$\alpha$} & \multicolumn{2}{|c}{$0.1$} & \multicolumn{2}{|c}{$0.5$} & \multicolumn{2}{|c}{$0.1$} &\multicolumn{2}{|c|}{$0.5$} \\ \hline
     \multicolumn{1}{|c}{$\lambda$} & \multicolumn{1}{|c}{$0.1$} & \multicolumn{1}{|c}{$0.01$} & \multicolumn{1}{|c}{$0.1$} & \multicolumn{1}{|c}{$0.01$} & \multicolumn{1}{|c}{$0.1$} & \multicolumn{1}{|c}{$0.01$} & \multicolumn{1}{|c}{$0.1$} & \multicolumn{1}{|c|}{$0.01$} \\ \hline
    Alg1 ($L^1$) 
        & $\mathbf{0.68} \pm 0.07$ 
        & $\mathbf{0.14} \pm 0.03$ 
        & $\mathbf{0.85} \pm 0.09$ 
        & $\mathbf{0.20} \pm 0.05$ 
        & $\mathbf{0.27} \pm 0.04$ 
        & $\mathbf{0.03} \pm 0.00$ 
        & $\mathbf{0.32} \pm 0.10$ 
        & $\mathbf{0.04} \pm 0.01$  \\ \cline{1-1}
    ROAR ($L^1$)
        & \begin{tabular}{c}$0.83 \pm 0.10$ \\ (+22.1\%)\end{tabular}
        & \begin{tabular}{c}$0.28 \pm 0.04$ \\ (+100.0\%)\end{tabular}
        & \begin{tabular}{c}$1.03 \pm 0.11$ \\ (+21.2\%)\end{tabular}
        & \begin{tabular}{c}$0.40 \pm 0.06$ \\ (+100.0\%)\end{tabular}
        & \begin{tabular}{c}$0.44 \pm 0.06$ \\ (+63.0\%)\end{tabular}
        & \begin{tabular}{c}$0.26 \pm 0.04$ \\ (+766.7\%)\end{tabular}
        & \begin{tabular}{c}$1.25 \pm 1.76$ \\ (+290.6\%)\end{tabular}
        & \begin{tabular}{c}$1.04 \pm 1.75$ \\ (+2500.0\%)\end{tabular} \\ \cline{1-1} 
    Alg2 ($L^\infty$) 
        & \begin{tabular}{c}$0.80 \pm 0.08$ \\ (+17.6\%)\end{tabular}
        & \begin{tabular}{c}$0.16 \pm 0.03$ \\ (+14.3\%)\end{tabular}
        & \begin{tabular}{c}$1.12 \pm 0.05$ \\ (31.7\%)\end{tabular}
        & \begin{tabular}{c}$0.62 \pm 0.03$ \\ (+210.0\%)\end{tabular}
        & \begin{tabular}{c}$0.31 \pm 0.04$ \\ (+14.8\%)\end{tabular}
        & \begin{tabular}{c}$0.04 \pm 0.00$ \\ (+33.3\%)\end{tabular}
        & \begin{tabular}{c}$0.53 \pm 0.06$ \\ (+65.6\%)\end{tabular}
        & \begin{tabular}{c}$0.06 \pm 0.01$ \\ (+50.0\%)\end{tabular} \\ \cline{1-1}
    ROAR ($L^\infty)$ 
        & \begin{tabular}{c}$0.99 \pm 0.11$ \\ (+45.6\%)\end{tabular}
        & \begin{tabular}{c}$0.33 \pm 0.05$ \\ (+135.7\%)\end{tabular}
        & \begin{tabular}{c}$1.54 \pm 0.46$ \\ (+81.2\%)\end{tabular}
        & \begin{tabular}{c}$0.65 \pm 0.01$ \\ (+225.0\%)\end{tabular}
        & \begin{tabular}{c}$0.62 \pm 0.08$ \\ (+129.6\%)\end{tabular}
        & \begin{tabular}{c}$ 0.25 \pm 0.03$ \\ (733.3\%)\end{tabular}
        & \begin{tabular}{c}$2.10 \pm 2.23$ \\ (+556.2\%)\end{tabular}
        & \begin{tabular}{c}$1.28 \pm 2.28$ \\ (+3100.0\%)\end{tabular} \\    
        \cline{1-1}  \hline 
\end{tabular}
}
\end{adjustbox}
\caption{The price of recourse for logistic regression models. The columns correspond to combinations of $\alpha$ and $\lambda$ for each of the datasets. Each row represents the price of recourse returned by each of the algorithms, averaged over all the test instances in each dataset. The smallest price is shown in bold in each column, and percentages indicate the increase in price compared to the smallest value.  \label{table:price_lr}}
\end{table*}
\renewcommand{\arraystretch}{1}

\renewcommand{\arraystretch}{2.0}
\begin{table*}[t]
\centering
\begin{adjustbox}{max width=\linewidth}
\huge{\begin{tabular}{|c|cccc|cccc|}
    \hline
  \multicolumn{1}{|c|}{} & \multicolumn{4}{c|}{German (NN)} & \multicolumn{4}{c|}{Small Business Administration (NN)} \\ \hline
     \multicolumn{1}{|c}{$\alpha$} & \multicolumn{2}{|c}{$0.1$} & \multicolumn{2}{|c}{$0.5$} & \multicolumn{2}{|c}{$0.1$} &\multicolumn{2}{|c|}{$0.5$} \\ \hline
     \multicolumn{1}{|c}{$\lambda$} & \multicolumn{1}{|c}{$0.7$} & \multicolumn{1}{|c}{$0.3$} & \multicolumn{1}{|c}{$0.7$} & \multicolumn{1}{|c}{$0.3$} & \multicolumn{1}{|c}{$0.1$} & \multicolumn{1}{|c}{$0.01$} & \multicolumn{1}{|c}{$0.1$} & \multicolumn{1}{|c|}{$0.01$} \\ \hline 
    Alg1 ($L^1$) 
        & $\mathbf{1.60} \pm 0.19$ 
        & $\mathbf{1.33} \pm 0.18$ 
        & $\mathbf{4.25} \pm 0.39$ 
        & $\mathbf{3.71} \pm 0.28$ 
        & $\mathbf{1.12} \pm 1.86$ 
        & $\mathbf{0.19} \pm 0.36$ 
        & $\mathbf{3.17} \pm 1.94$ 
        & $\mathbf{0.66} \pm 1.03$  \\ \cline{1-1}
    ROAR ($L^1$)
        & \begin{tabular}{c}$1.61 \pm 0.19$ \\ (+0.6\%)\end{tabular}
        & \begin{tabular}{c}$1.41 \pm 0.17$ \\ (+6.0\%)\end{tabular}
        & \begin{tabular}{c}$4.51 \pm 0.45$ \\ (+6.1\%)\end{tabular}
        & \begin{tabular}{c}$4.03 \pm 0.40$ \\ (+8.6\%)\end{tabular}
        & \begin{tabular}{c}$1.90 \pm 2.02$ \\ (+69.6\%)\end{tabular}
        & \begin{tabular}{c}$1.12 \pm 1.96$ \\ (+489.5\%)\end{tabular}
        & \begin{tabular}{c}$5.73 \pm 1.60$ \\ (+80.8\%)\end{tabular}
        & \begin{tabular}{c}$3.79 \pm 1.90$ \\ (+474.2\%)\end{tabular} \\ \cline{1-1} 
    Alg2 ($L^\infty$) 
        & \begin{tabular}{c}$65.52 \pm 4.90$ \\ (+3995.0\%)\end{tabular}
        & \begin{tabular}{c}$63.97 \pm 3.69$ \\ (+4709.8\%)\end{tabular}
        & \begin{tabular}{c}$83.85 \pm 5.56$ \\ (+1872.9\%)\end{tabular}
        & \begin{tabular}{c}$84.74 \pm 1.94$ \\ (+2184.1\%)\end{tabular}
        & \begin{tabular}{c}$86.46 \pm 2.76$ \\ (+7619.6\%)\end{tabular}
        & \begin{tabular}{c}$85.66 \pm 2.59$ \\ (+44984.2\%)\end{tabular}
        & \begin{tabular}{c}$94.64 \pm 5.60$ \\ (+2885.5\%)\end{tabular}
        & \begin{tabular}{c}$95.35 \pm 4.38$ \\ (+14347.0\%)\end{tabular} \\ \cline{1-1}
    ROAR ($L^\infty)$ 
        & \begin{tabular}{c}$64.26 \pm 3.13$ \\ (+3916.3\%)\end{tabular}
        & \begin{tabular}{c}$62.67 \pm 2.79$ \\ (+4612.0\%)\end{tabular}
        & \begin{tabular}{c}$86.95 \pm 6.55$ \\ (+1945.9\%)\end{tabular}
        & \begin{tabular}{c}$90.72 \pm 5.51$ \\ (+2345.3\%)\end{tabular}
        & \begin{tabular}{c}$81.41 \pm 11.74$ \\ (+7168.7\%)\end{tabular}
        & \begin{tabular}{c}$84.42 \pm 10.67$ \\ (+44331.6\%)\end{tabular}
        & \begin{tabular}{c}$100.88 \pm 0.10$ \\ (+3082.3\%)\end{tabular}
        & \begin{tabular}{c}$98.14 \pm 4.44$ \\ (+14769.7\%)\end{tabular} 
    \\\cline{1-1}  \hline 
\end{tabular}
}
\end{adjustbox}
\caption{The price of recourse for neural network models using LIME approximation. The columns correspond to combinations of $\alpha$ and $\lambda$ for each of the datasets. Each row represents the price of recourse returned by each of the algorithms, averaged over all the test instances in each dataset. The smallest price is shown in bold in each column, and percentages indicate the increase in price compared to the smallest value. \label{table:price_nn}}
\end{table*}
\renewcommand{\arraystretch}{1}

To generate \emph{optimal} recourse, we implemented Algorithms~\ref{alg:l1-l1}~and~\ref{alg:l1-linf}. For the latter, we used code from~\cite{KayasthaGJ24}. We used the code from~\citep{UpadhyayJL21} for ROAR, and~\cite{NguyenBN+22} for RBR as baselines. For ROAR, we use two variants: one for which the model change is measured using $L^1$ norm and another where it is measured with $L^{\infty}$ norm. We refer to these variants as ROAR ($L^1$) and ROAR ($L^{\infty}$), respectively. For neural network models, all approaches except RBR first approximate the non-linear models locally with LIME~\citep{RibeiroSG16}.\footnote{For differentiable models, local linearization can be done using SmoothGrad~\cite{SmilkovTKVW17}. We provided analysis using this linearization in Appendix~\ref{sec:app-exp-smoothgrad}, which is consistent with our findings when using LIME.} We set $\ell$ in Equation~\ref{eq:j} to binary cross-entropy. We use different $\alpha$ and $\lambda$ values in different experiments. We demonstrate these choices and additional parameters (if applicable) for each experiment. Our code is available at \url{https://github.com/PMyatKyaw/Optimal-Robust-Recourse}. See Appendix~\ref{sec:app-exp-details} for additional details.

\subsection{Analysis of Price of Recourse}
\label{sec:exp-price}
We start by studying how different types of $L^p$ norms to define the neighborhood for model change can affect the price of recourse. In particular, for any given $\alpha$, $\lambda$ pair, using Algorithms~\ref{alg:l1-l1}~and~\ref{alg:l1-linf} as well as two ROAR variants~\cite{UpadhyayJL21}, we compute recourse for each of the test instances in our datasets, using both logistic regression and neural network models.  For each computed recourse $x$, we then compute the optimal adversarial model $\theta^*(x)$ using Equation~\ref{eq:adversary}. The price of recourse can be computed as $J(x, \theta^*(x))$.

The results are summarized in Tables~\ref{table:price_lr} (logistic regression models)~and~\ref{table:price_nn} (neural network models). In each table, each column corresponds to a pair of $\alpha$ and $\lambda$ combination (that can vary across datasets and models), and each row corresponds to the average price of recourse computed by each algorithm for each dataset. For each combination of $\alpha$ and $\lambda$, the smallest price of recourse is represented in bold. For all other algorithms, we also include the percentage increase in their price compared to the smallest price.

As suggested by our theory in Section~\ref{sec:alg-analysis}, for linear models, Algorithm~\ref{alg:l1-l1} has the smallest price. So in Table~\ref{table:price_lr}, Algorithm~\ref{alg:l1-l1} has the smallest price for all combinations. Comparing Algorithms~\ref{alg:l1-l1}~and~\ref{alg:l1-linf}, Table~\ref{table:price_lr} indicates that defining the neighborhood for model change using $L^{\infty}$ can increase the price by 14-210\% depending on the values of $\alpha$ and $\lambda$. More strikingly, these price increases are steeper when model change is measured using the $L^1$ norm, but ROAR is used instead of the optimal algorithm. We also observe that the percentage of price increase is generally higher for the Small Business dataset compared to the German Credit dataset.

Even for non-linear models, Table~\ref{table:price_nn} shows that Algorithm~\ref{alg:l1-l1} continues to have the smallest price across all combinations of $\alpha$ and $\lambda$ values for both datasets. Comparing Algorithms~\ref{alg:l1-l1}~and~\ref{alg:l1-linf}, Table~\ref{table:price_nn} indicates that defining the neighborhood for model change using $L^{\infty}$ norm instead of $L^{1}$ norm can increase the price much higher for non-linear models (minimum 18x) compared to linear models (maximum 2.1x).

For non-linear models, Table~\ref{table:price_nn} shows that ROAR ($L^1$) compares more favorably with Algorithm~\ref{alg:l1-l1} compared to linear models; although, even in this case, the sub-optimality in price can be as high as 470+\% in the Small Business dataset. Finally, Table~\ref{table:price_nn} shows that while  Algorithm~\ref{alg:l1-l1} always computes a recourse with a smaller price compared to ROAR ($L^1$), this is not the case when comparing Algorithm~\ref{alg:l1-linf} with ROAR ($L^\infty$) suggesting that the optimal algorithm is more resilient to approximation errors when the neighborhood for model change is defined with respect to the $L^1$ norm.

\begin{figure*}[t!]
    \centering
    % INSTANCE-WISE VALIDITY
    \begin{subfigure}[b]{0.45\textwidth}
        \centering
        \includegraphics[width=0.85\textwidth]{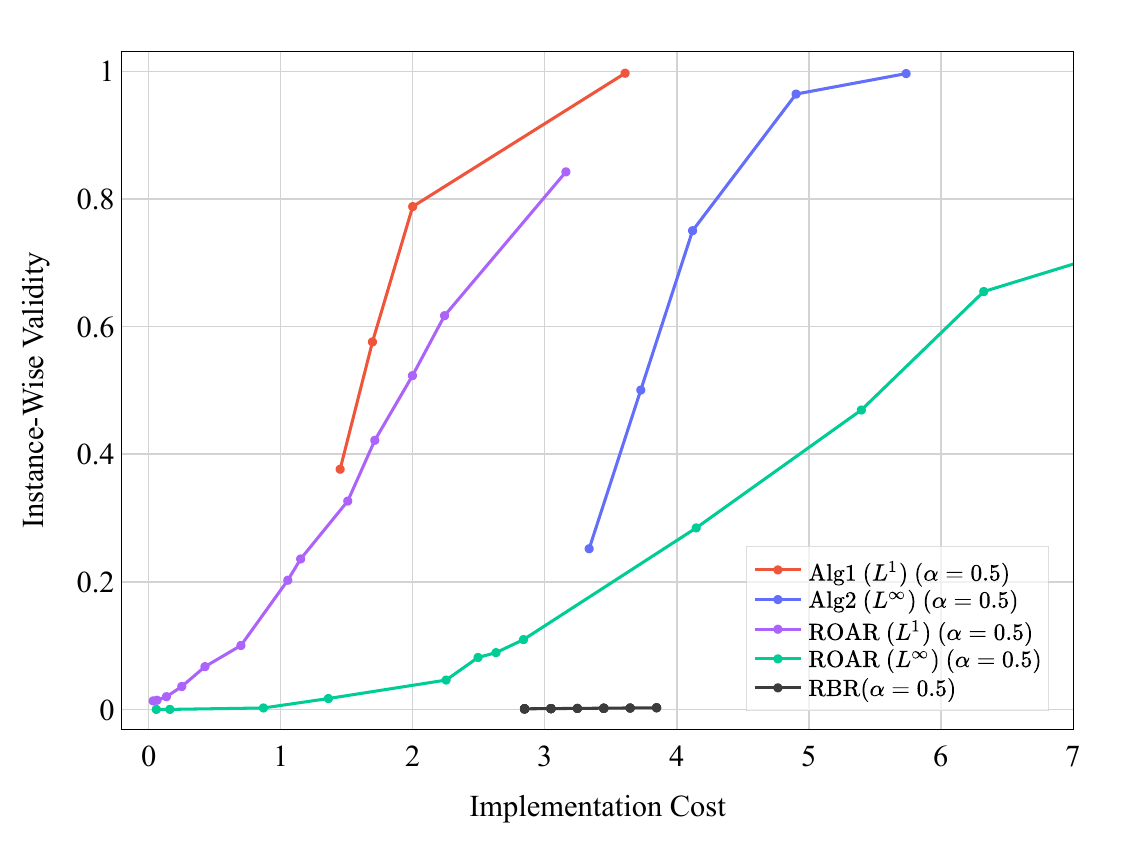}
        \caption{Instance-Wise Validity, Logistic Regression}
        \label{fig:cost_validity_tradeoff_lr_sba_instance_alpha_0.5}
    \end{subfigure}
    % POPULATION-WISE VALIDITY
    \begin{subfigure}[b]{0.45\textwidth}
        \centering
        \includegraphics[width=0.85\textwidth]{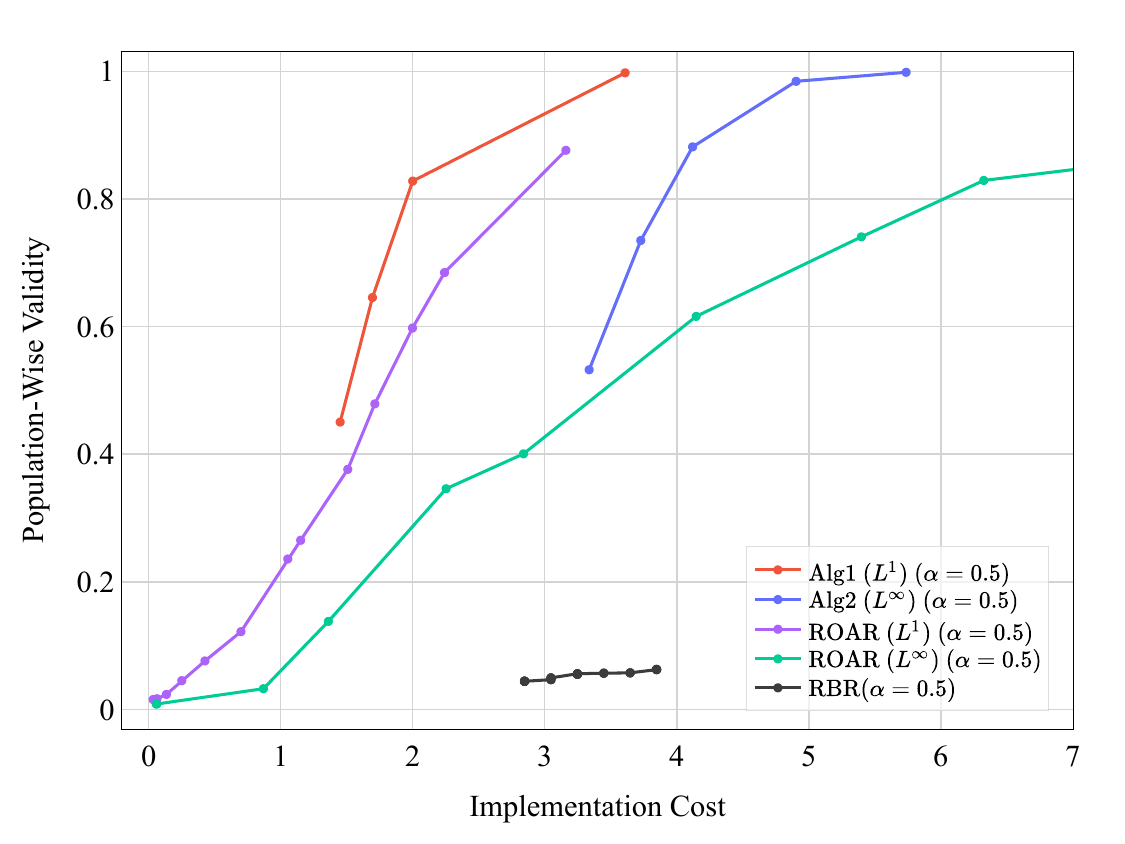}
        \caption{Population-Wise Validity, Logistic Regression}
        \label{fig:cost_validity_tradeoff_lr_sba_population_alpha_0.5}
    \end{subfigure}
    % CURRENT VALIDITY
    \begin{subfigure}[b]{0.45\textwidth}
        \centering
        \includegraphics[width=0.85\textwidth]{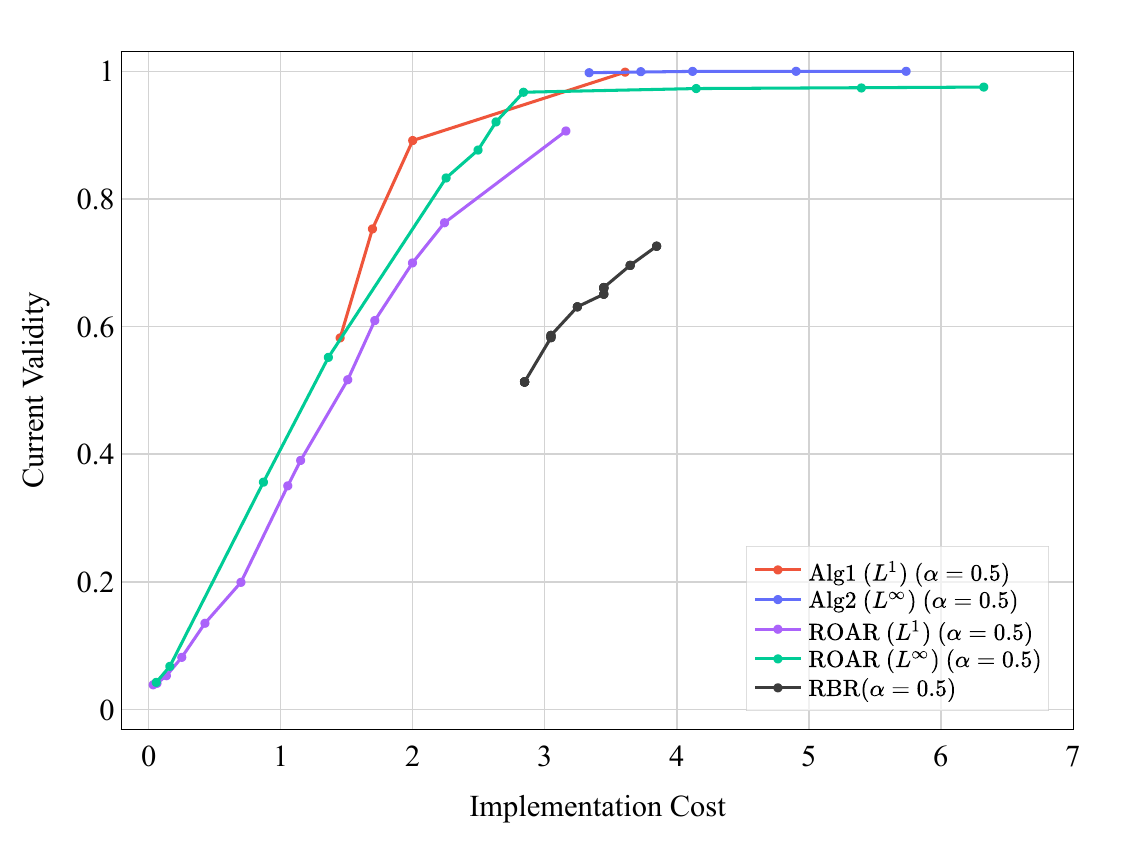}
        \caption{Current Validity, Logistic Regression}
        \label{fig:cost_validity_tradeoff_lr_sba_current_alpha_0.5}
    \end{subfigure}
    % FUTURE VALIDITY
    \begin{subfigure}[b]{0.45\textwidth}
        \centering
        \includegraphics[width=0.85\textwidth]{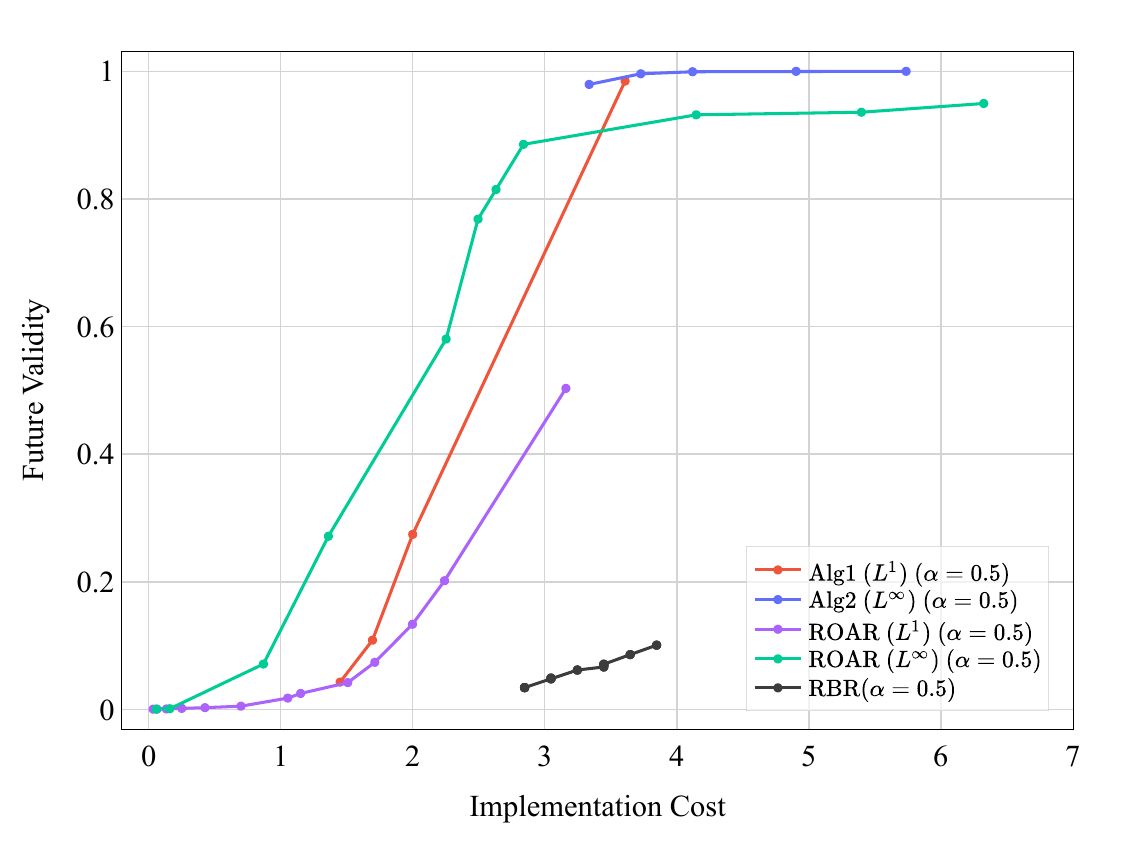}
        \caption{Future Validity, Logistic Regression}
        \label{fig:cost_validity_tradeoff_lr_sba_future_alpha_0.5}
    \end{subfigure}
    \caption{The frontier of the trade-off between validity and implementation cost on the Small Business Administration dataset and logistic regression models with $\alpha=0.5$. Each subfigure corresponds to a different measure of validity. In each subfigure, curves show the trade-off for different algorithms. 
    \label{fig:cost_validity_tradeoff_sba_alpha_0.5_LR}}
\end{figure*}

\begin{figure*}[t!]
    \centering
    % INSTANCE-WISE VALIDITY
    \begin{subfigure}[b]{0.45\textwidth}
        \centering
        \includegraphics[width=0.85\textwidth]{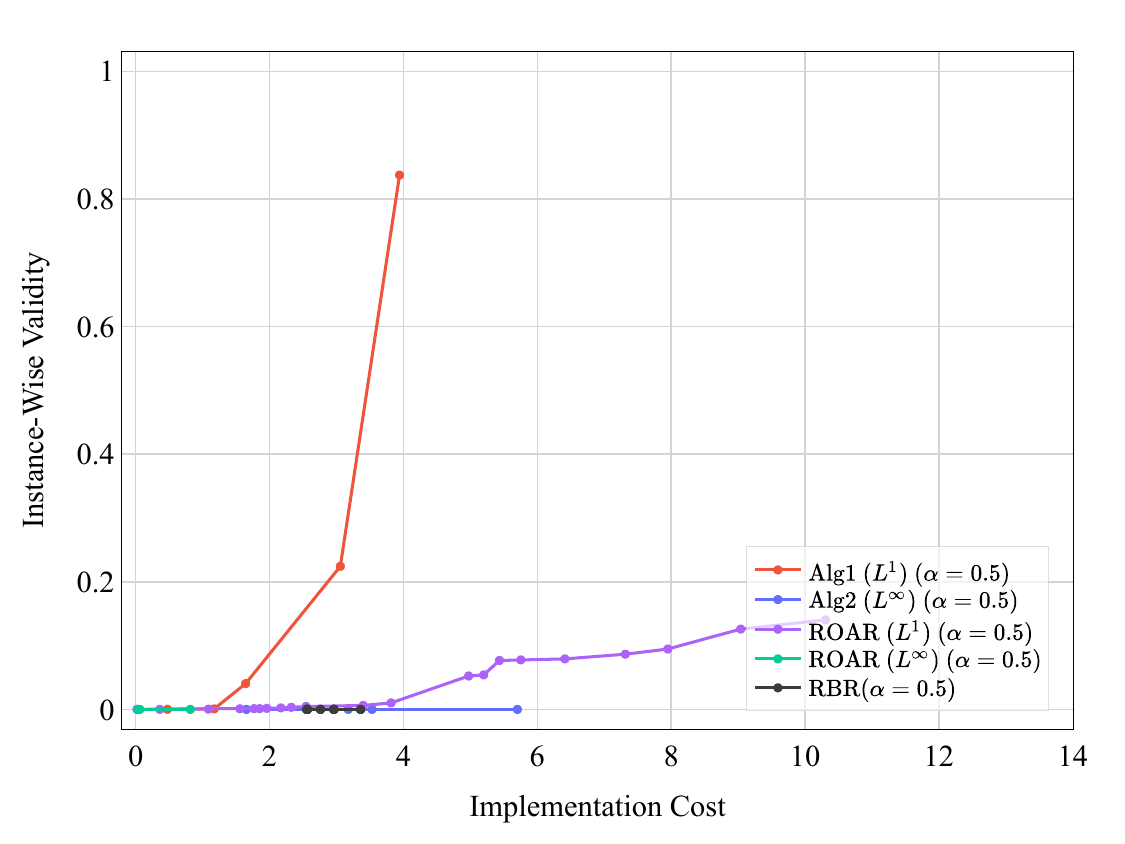}
        \caption{Instance-Wise Validity, Neural Network}
        \label{fig:cost_validity_tradeoff_nn_sba_instance_alpha_0.5}
    \end{subfigure}
    % POPULATION-WISE VALIDITY
    \begin{subfigure}[b]{0.45\textwidth}
        \centering
        \includegraphics[width=0.85\textwidth]{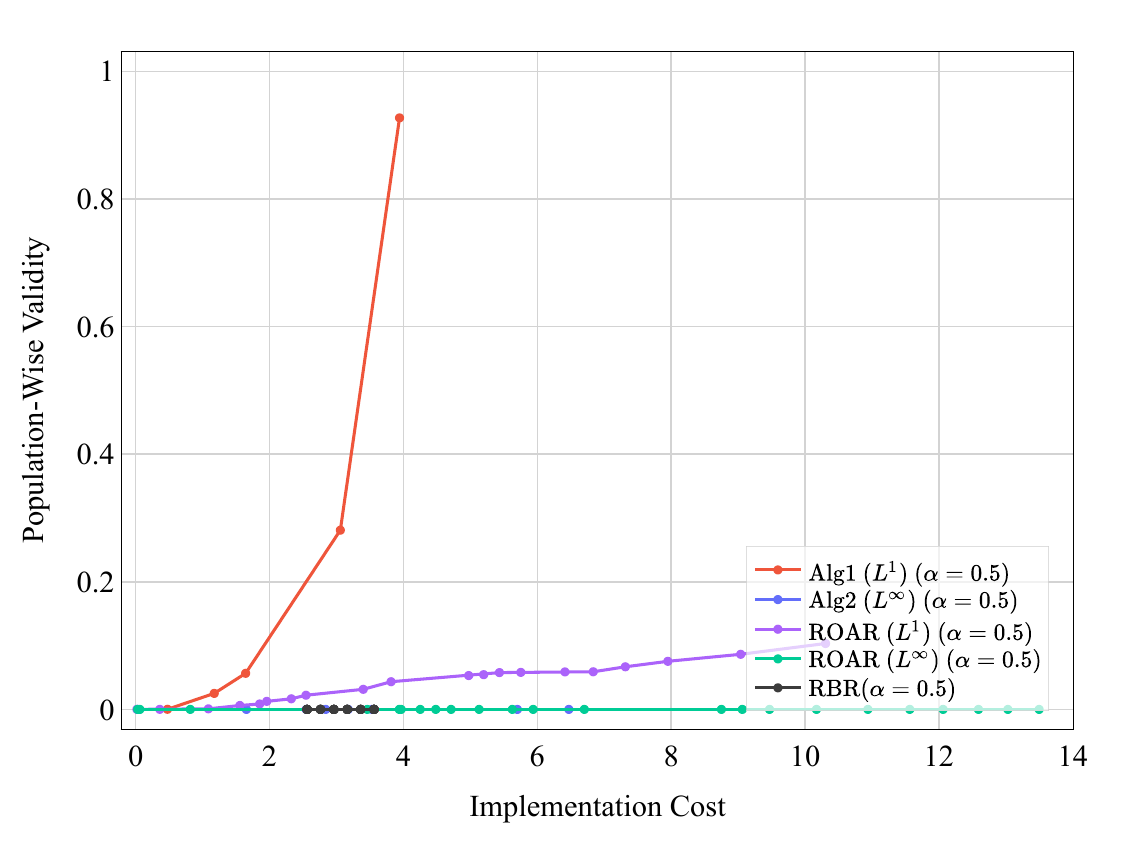}
        \caption{Population-Wise Validity, Neural Network}
        \label{fig:cost_validity_tradeoff_nn_sba_population_alpha_0.5}
    \end{subfigure}
    % CURRENT VALIDITY
    \begin{subfigure}[b]{0.45\textwidth}
        \centering
        \includegraphics[width=0.85\textwidth]{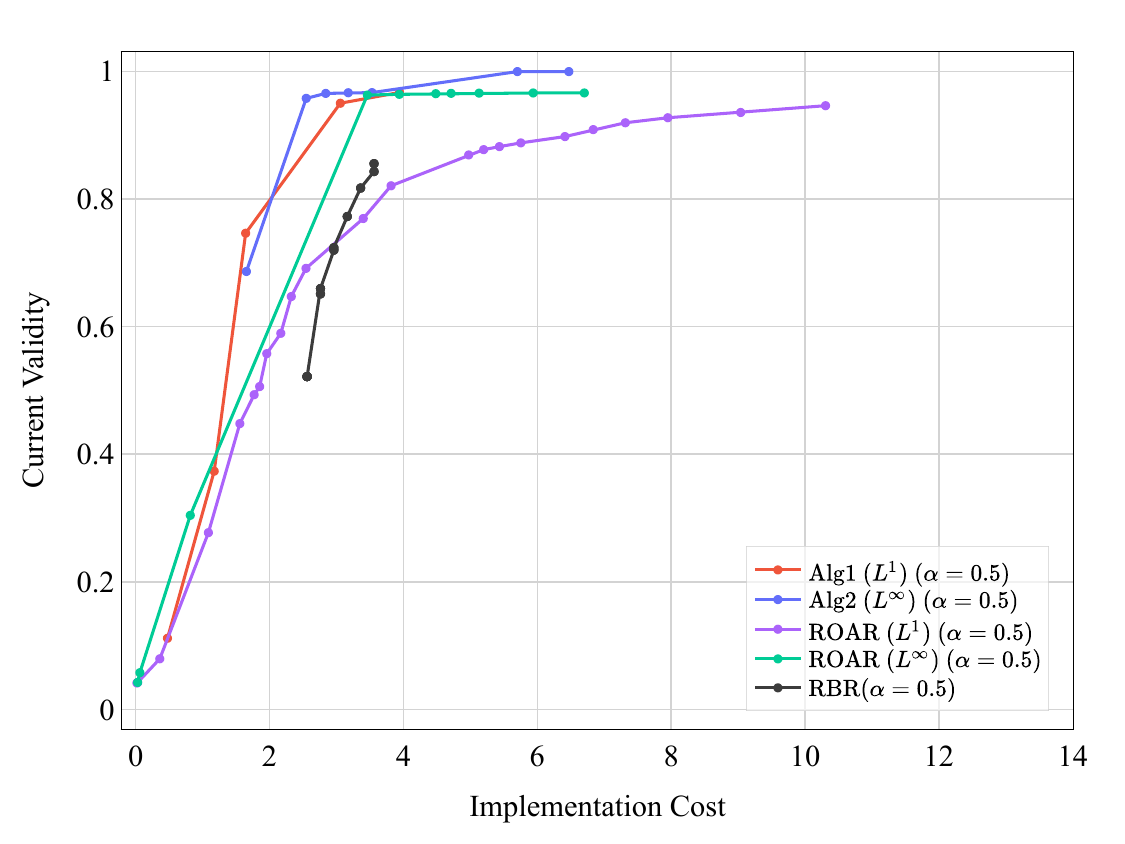}
        \caption{Current Validity, Neural Network}
        \label{fig:cost_validity_tradeoff_nn_sba_current_alpha_0.5}
    \end{subfigure}
    % FUTURE VALIDITY
    \begin{subfigure}[b]{0.45\textwidth}
        \centering
        \includegraphics[width=0.85\textwidth]{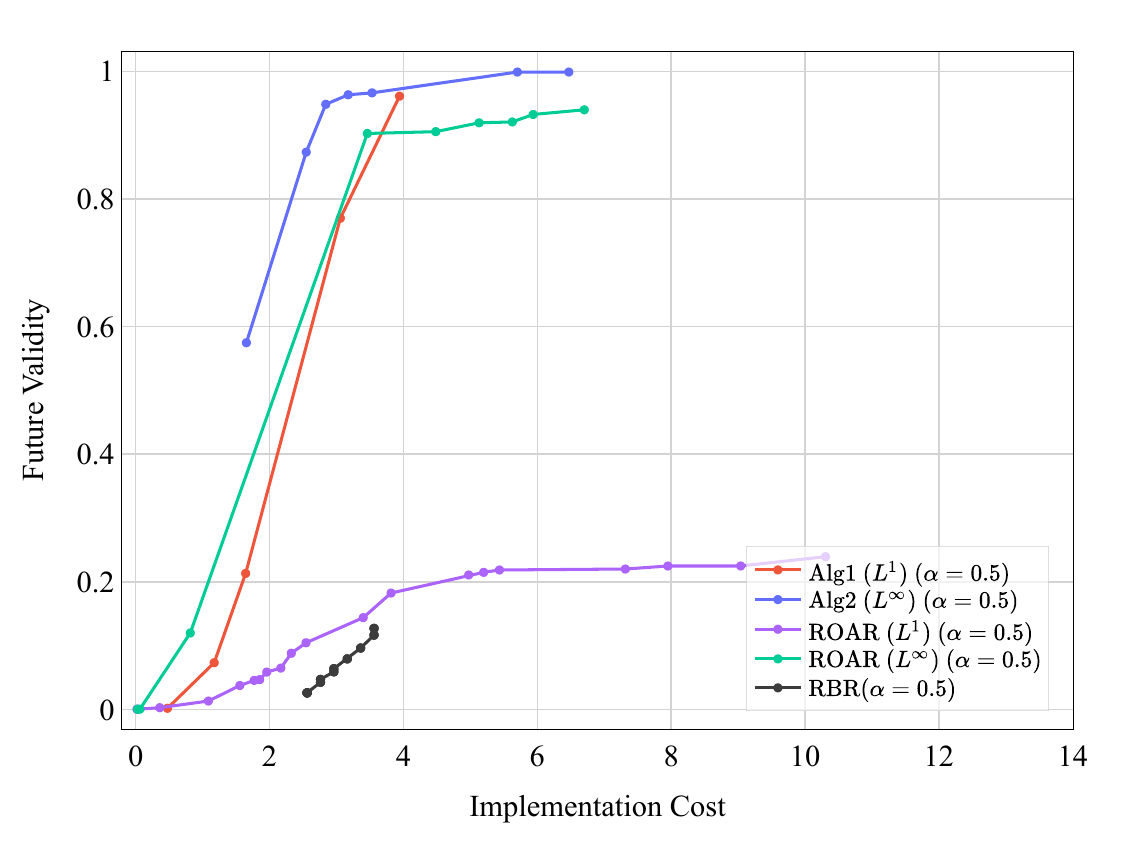}
        \caption{Future Validity, Neural Network}
        \label{fig:cost_validity_tradeoff_nn_sba_future_alpha_0.5}
    \end{subfigure}
    \caption{The frontier of the trade-off between validity and implementation cost on the German Credit dataset and logistic regression models with $\alpha=0.5$. Each subfigure corresponds to a different measure of validity. In each subfigure, curves show the trade-off for different algorithms. 
    \label{fig:cost_validity_tradeoff_sba_alpha_0.5_NN}}
\end{figure*}

\subsection{Trade-off Between Implementation Cost and Validity}
\label{sec:exp-trade-off}

While in Section~\ref{sec:exp-price}, we compared our algorithms and baselines by focusing on the price of recourses provided at specific combinations of $\alpha$ and $\lambda$, in this section, we aim to break down the price of recourse and study the achievable trade-off between validity and implementation cost of recourse for each of the algorithms. In Equation~\eqref{eq:j} for the price, the first term is a proxy for validity measured with respect to a given model $\theta$ and the second term is the implementation cost of recourse, i.e., the cost of modifying $\xz$ to $\newx$. 

To generate the trade-offs for each dataset and model combination, we vary $\lambda$ in the same range for all algorithms, and compute the recourses for all test instances using both $\alpha=0.1$ and $\alpha=0.5$ (details about the exact range of $\lambda$ used can be found in Appendix~\ref{sec:app-exp-trade-off}). We also compute recourses with RBR~\cite{NguyenBN+22}. While the formulation of RBR does not have the same parameters as us, we replicate their experiments by setting the ambiguity sizes to $\epsilon_0, \epsilon_1 \in [0, 1]$ with increments of 0.5, and the maximum recourse cost $\delta = \|\xz - \robx\|_1 + \delta_+$ to $\delta_+ \in [0, 1]$ with increments of 0.2. These choices are the same as their paper, and we also use the same datasets as RBR in our experiments.

Once recourse is computed, we can break down the quality of recourse along two axes. One is the implementation cost measured as the $L^1$ difference between the initial instance and the provided recourse. The other is validity, which we define as the probability that the recourse leads to the desirable outcome, and it can be measured with respect to several models as is done in prior work~\cite{UpadhyayJL21, NguyenBN+22}. One choice is to compute the worst-case model for each instance (defined as the model $\theta$ maximizing $J(x,\theta)$ for any given recourse $x$). This is in line with our theoretical results. We refer to this quantity as \emph{instance-wise validity}. As a less powerful adversarial model, instead of computing worst-case models per instance, we can compute a single worst-case model per dataset. This is defined as the model $\theta$ maximizing $\Sigma_{x} J(x,\theta)$ where the sum is over all instances for which recourse is provided. We refer to this quantity as \emph{population-wise validity}. We can also measure validity with respect to the initial model, assuming no model change. We refer to this quantity as \emph{current validity}. Finally, motivated by reasons for model shift and given access to different versions of the datasets, prior work also computes a model on the alternate shifted version of each of the datasets~\cite{UpadhyayJL21}. This shifted dataset for the German Credit dataset is the result of a correction shift. For the Small Business Administration dataset, the shifted dataset is the result of a temporal shift. We refer to this quantity as \emph{future validity}. When the worst-case model cannot be computed analytically, we use projected gradient ascent to compute it. 

Figures~\ref{fig:cost_validity_tradeoff_sba_alpha_0.5_LR}~and~\ref{fig:cost_validity_tradeoff_sba_alpha_0.5_NN} depict the frontier of the trade-off between different types of validity (Y-axis) and implementation cost (X-axis) for the Small Business Administration dataset for Logistic Regression and neural networks models with $\alpha = 0.5$. Each subfigure corresponds to a different measure of validity, and curves show the frontier of the trade-off between implementation cost and validity for different algorithms.

For logistic regression models in Figure~\ref{fig:cost_validity_tradeoff_sba_alpha_0.5_LR}, we observe that our algorithm Pareto dominates other algorithms for instance-wise validity as suggested by theory. Surprisingly, this dominance continues for population-wise and current validity as well. For future validity, while our algorithm outperforms at higher validity regimes, in lower validity regimes, ROAR ($L^{\infty}$) returns the best trade-off. One reason can be that the model used to measure feature validity is much further away from the original model compared to the $\alpha$ values we use for our algorithm. Moreover, except for Algorithm~\ref{alg:l1-linf}, no ROAR variant reaches the validity of 1, while our algorithm is always able to reach perfect validity. In all experiments, RBR is dominated by all other algorithms.

The comparison between the models becomes more complex for neural network models as depicted in Figure~\ref{fig:cost_validity_tradeoff_sba_alpha_0.5_NN}. First of all, for the most stringent measures of validity (instance-wise and population-wise), the validity of neither of the algorithms reaches 1, though our algorithm achieves much higher validity compared to all others. Not surprisingly, not only does the validity of all models drop when using neural network models compared to logistic regression models, but also the implementation cost of recourse increases (compare the X axis in Figures~\ref{fig:cost_validity_tradeoff_sba_alpha_0.5_LR}~and~\ref{fig:cost_validity_tradeoff_sba_alpha_0.5_NN}). For less stringent adversaries (current and future), the validity of all algorithms improves. Even in these cases, our algorithm performs only slightly worse than the best-performing approach (Algorithm~\ref{alg:l1-linf}) and on par with or better than other baselines.

In Appendix~\ref{sec:app-exp-trade-off}, we provide results for $\alpha=0.1$ for the Small Business dataset for both logistic regression and neural network models (Figure~\ref{fig:cost_validity_tradeoff_sba_alpha_0.1_app}). In the same section, we also provide results for the German Credit dataset for both models using $\alpha=0.1$ (Figure~\ref{fig:cost_validity_tradeoff_german_alpha_0.1_app}) and $\alpha=0.5$ (Figure~\ref{fig:cost_validity_tradeoff_german_alpha_0.5_app}). While the observations presented in this section generally hold for the German Credit dataset, compared to the Small Business Administration dataset, the validity of all algorithms is lower.
\begin{figure*}[ht!]
    \centering
    \begin{subfigure}[b]{0.46\textwidth}
        \centering
        \includegraphics[width=0.85\textwidth]{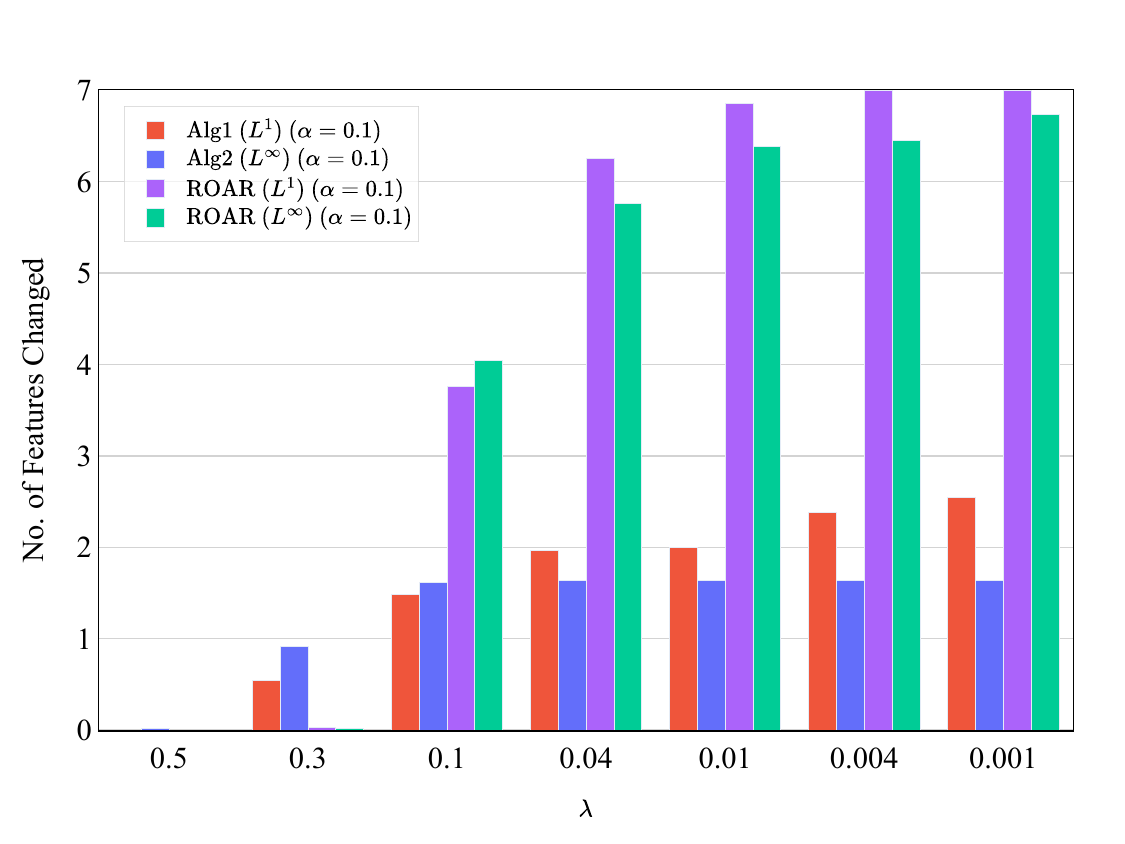}
        \caption{German Credit dataset, Logistic Regression, $\alpha = 0.1$}
        \label{fig:sparsity_add_lr_german_0.1}
    \end{subfigure}    
    \begin{subfigure}[b]{0.46\textwidth}
        \centering
        \includegraphics[width=0.85\textwidth]{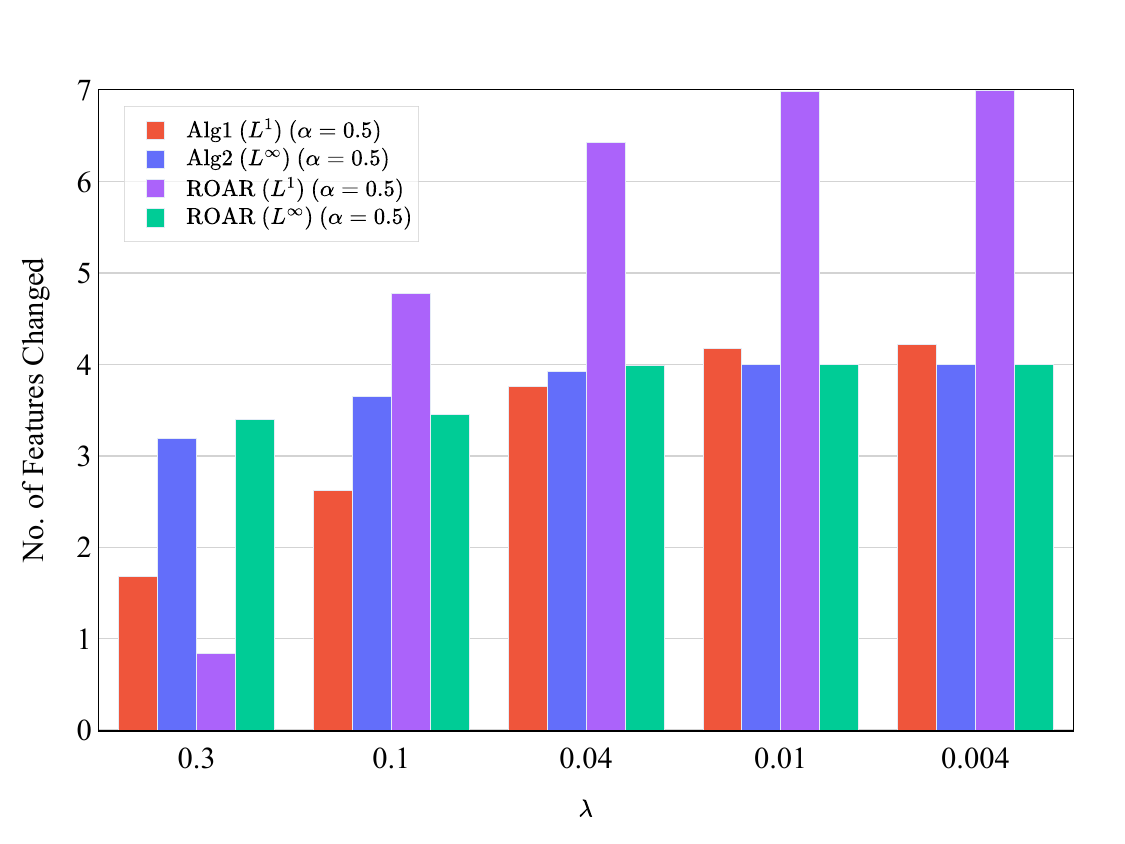}
        \caption{German Credit dataset, Logistic Regression, $\alpha = 0.5$}
        \label{fig:sparsity_add_lr_german_0.5}
    \end{subfigure}
    \begin{subfigure}[b]{0.46\textwidth}
        \centering
        \includegraphics[width=0.85\textwidth]{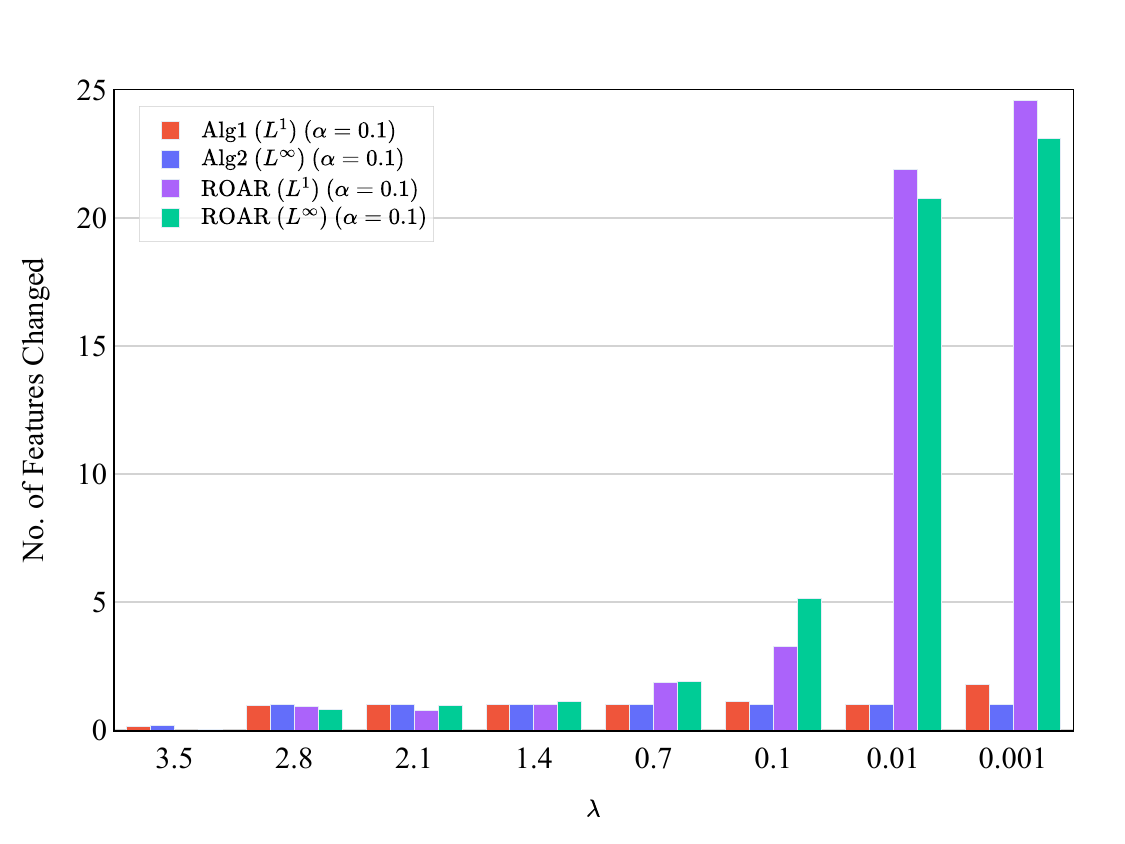}
        \caption{SBA Dataset, Logistic Regression, $\alpha = 0.1$}
        \label{fig:sparsity_add_lr_sba_0.1}
    \end{subfigure}
    \begin{subfigure}[b]{0.46\textwidth}
        \centering
        \includegraphics[width=0.85\textwidth]{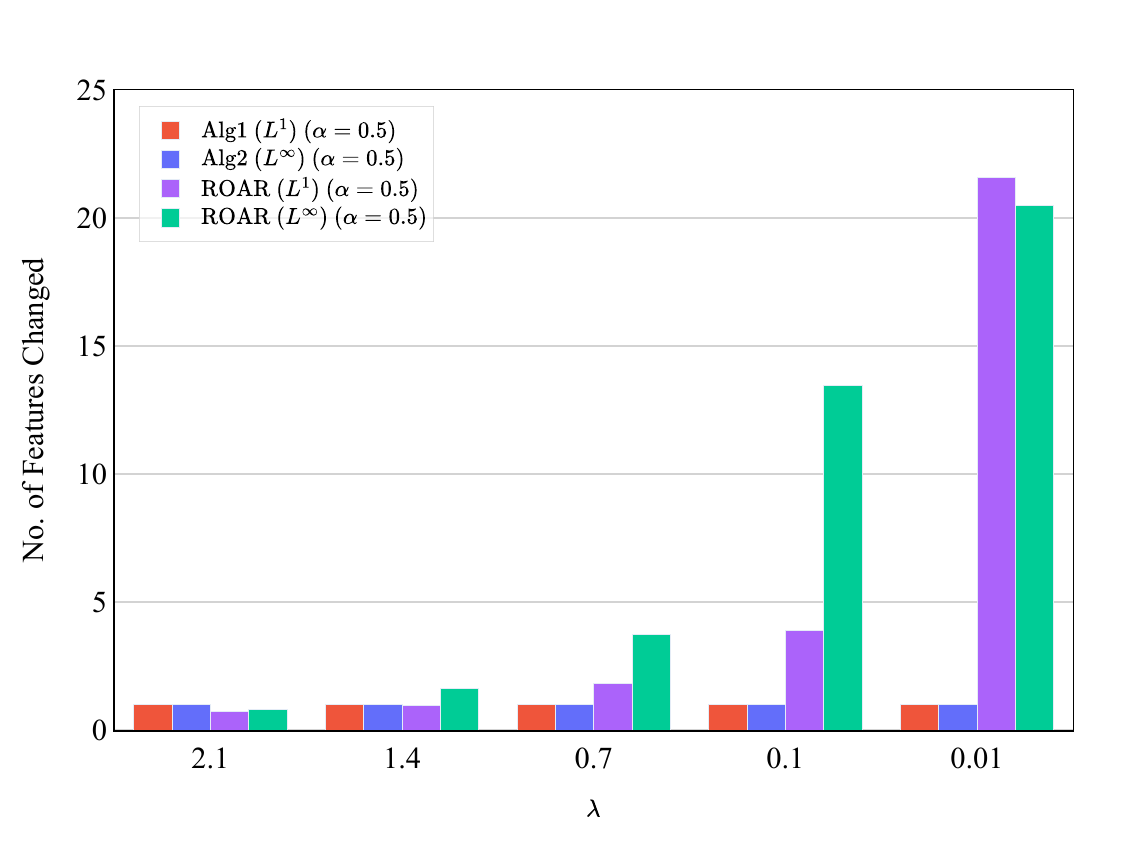}
        \caption{SBA Dataset, Logistic Regression, $\alpha=0.5$}
        \label{fig:sparsity_add_lr_sba_0.5}
    \end{subfigure}
    \caption{Number of changed features for the German and Small Business Datasets for logistic regression models. Left and right columns correspond to $\alpha=0.1$ and $\alpha=0.5$, respectively. The top row corresponds to the German Credit dataset, while the bottom row corresponds to the Small Business Administration dataset. In each subfigure, bars depict the number of changed features for each of the algorithms at different $\lambda$ values.\label{fig:sparsity}}
\end{figure*}

\subsection{Sparsity}
\label{sec:exp-sparsity}
Sparsity is a desirable property of counterfactual explanations, such as recourse~\cite{VermaSBDS23, TominagaYK24, KanchetiVMB25}, requiring the recourse to change a small number of features. Prior works encourage sparsity by regularization. However, neither our approaches nor the baselines use regularization. Nonetheless, we next aim to compare the sparsity of the provided recourse by each of these approaches. 

\begin{figure*}[ht!]
    \centering
    \begin{subfigure}[b]{0.48\textwidth}
        \centering
        \includegraphics[width=\textwidth]{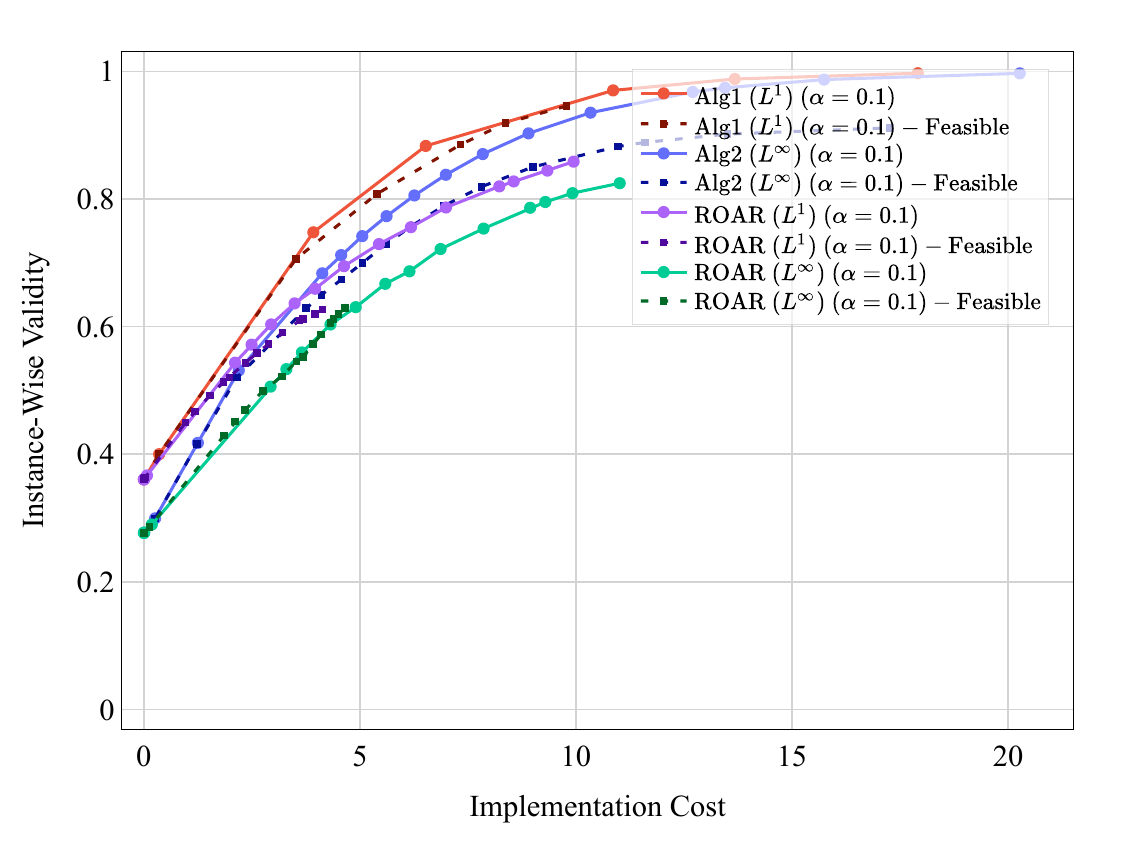}
        \caption{German Credit dataset, Logistic Regression, $\alpha = 0.1$}
        \label{fig:feasibility_lr_german_alpha_0.1}
    \end{subfigure}
    \begin{subfigure}[b]{0.48\textwidth}
        \centering
        \includegraphics[width=\textwidth]{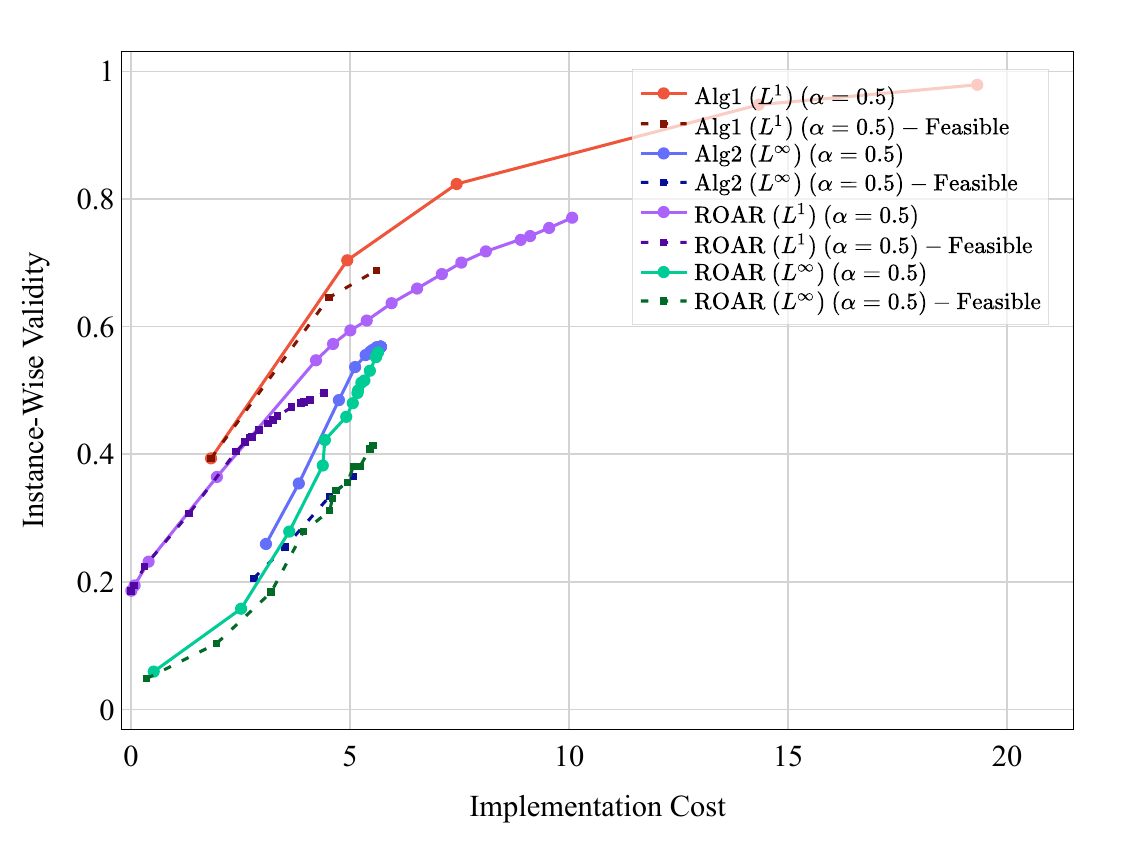}
        \caption{German Credit dataset, Logistic Regression, $\alpha = 0.5$}
        \label{fig:feasibility_lr_german_alpha_0.5}
    \end{subfigure}
    \caption{The frontier of the trade-off between validity and implementation cost on the Small Business Administration dataset and logistic regression models after post-processing. The left and right columns correspond to $\alpha=0.1$ and $\alpha=0.5$. In each subfigure, curves show the trade-off for different algorithms. For each algorithm, solid and dashed lines depict the performance before and after hardmax post-processing is applied.\label{fig:feasibility}}
\end{figure*}

In particular, for Algorithms~\ref{alg:l1-l1}~and~\ref{alg:l1-linf} and the two variants of ROAR, we first compute the recourse and then measure the number of features that are changed in the recourse compared to the original instance. Since gradient-based methods can add small perturbations to the features, we consider a feature to be changed when the value of the feature is modified by at least $\epsilon$ (in an additive manner). We use $\epsilon=0.01$ in all experiments.\footnote{The feature change can also be defined in a multiplicative manner, requiring the value of the feature to be modified by some predefined percentage. In our experiments, this does not change the results significantly. See Figures~\ref{fig:sparsity_alpha_0.1_app}~and~\ref{fig:sparsity_alpha_0.5_app} in Appendix~\ref{sec:app-exp-sparse} for more details}.

The results are presented in Figure~\ref{fig:sparsity}, where the top row corresponds to the German Credit dataset, while the bottom row corresponds to the Small Business dataset. The left and right columns represent $\alpha=0.1$ and $\alpha=0.5$, respectively. In each subfigure, the bars depict the average number of features changed by each of the approaches for varying $\lambda$ values (decreasing from left to right). All the results presented in Figure~\ref{fig:sparsity} are for logistic regression models. The results for neural network models are deferred to Figures~\ref{fig:sparsity_alpha_0.1_app}~and~\ref{fig:sparsity_alpha_0.5_app} in Appendix~\ref{sec:app-exp-sparse}.

In Figure~\ref{fig:sparsity}, we observe that both Algorithms~\ref{alg:l1-l1}~and~\ref{alg:l1-linf} change a much smaller number of features compared to the ROAR variants. In addition, the number of features changed for all approaches increases as $\lambda$ decreases. This is because $\lambda$ controls the weight on the implementation costs, and lower $\lambda$ values penalize higher implementation costs less. However, the increase in the number of features changed in the variants of ROAR is much more significant compared to the optimal algorithms. In particular, as $\lambda$ gets very small, ROAR appears to be changing \emph{all} the features while the number of features changed by Algorithms~\ref{alg:l1-l1}~and~\ref{alg:l1-linf} either remains the same (in the Small Business Dataset) or increases by 1 or 2 (in the German Credit dataset). Finally, the number of feature changes increases as $\alpha$ gets larger, indicating that when facing more powerful adversarial model changes, all algorithms, especially ROAR variants, rely on modifying more features as opposed to modifying the same features by bigger magnitudes.

\subsection{Feasibility}
\label{sec:exp-feasibility}

Both the German and the Small Business datasets contain categorical features, and there are constraints on the space of feasible values for each feature. For example, in the German Credit dataset, the feature Personal Status and Sex can take one of five mutually exclusive values: male and divorced/separated, female and divorced/separated/married, male and single, male and married/widowed, or female and single. In the Small Business dataset, the feature Revolving Line of Credit is binary, taking values Yes or No. Neither Algorithms~\ref{alg:l1-l1}~and~\ref{alg:l1-linf} nor the ROAR variants take into account such feasibility constraints when computing recourse. 

A common approach to enforce feasibility is to post-process the recourse by projecting it to the feasible set~\cite{UpadhyayJL21,NguyenBN23,GuoJC+23, KayasthaGJ24}. In this section, we are interested in understanding the effect of such post-processing on the quality of the recourse. More specifically, after computing the recourse, we apply a hardmax operation to the one-hot encoded categorical features so that exactly one entry is set to 1 and all others are set to 0. In the German Credit dataset, to ensure actionability, we further constrain the age feature such that it may only increase by at most two years.

In Figure~\ref{fig:feasibility}, we display the frontier of the trade-off between implementation cost and instance-wise validity of recourse before and after post-processing the recourse for logistic regression models on the German Credit dataset. The subfigures correspond to different $\alpha$ values as indicated in the caption. In each subfigure, the curves show the trade-off for different algorithms. For each algorithm, solid and dashed lines depict the performance before and after hardmax post-processing is applied.

Figure~\ref{fig:feasibility} shows that post-processing does not significantly change the frontier of the trade-off for any of the algorithms at low implementation costs, which corresponds to low to medium validity levels. However, after post-processing, points with high validity and implementation costs cannot be achieved on the frontier. However, this effect can probably be mitigated by using smaller $\lambda$ values to account for degradation in performance after post-processing.

The results for other datasets and model combinations are provided in Figures~\ref{fig:feasibility_alpha_0.1_app}~and~\ref{fig:sparsity_alpha_0.5_app} in Appendix~\ref{sec:app-exp-feasibility}. For these combinations, our results indicate that the post-processing even has a milder effect on the frontier of the trade-off between validity and implementation cost compared to the results presented in this section.
\section{Conclusion and Discussion}
\label{sec:discussion}

The literature on robust recourse provides many different formulations by considering various model classes, cost functions, and methods for formulating model changes. Each of these combinations results in distinctive optimization problems requiring different technical solutions (see~\citep{JiangLR+24a} for a survey). More explicitly, there are alternative proposals for capturing model change beyond the $L^p$ norms, such as naturally occurring model changes~\cite{DuttaLMTM22} or small changes to initial training conditions~\cite{BlackWF22}. In addition, using any $L^p$ norm to measure the implementation cost does not allow for capturing feature dependencies present in many applications. Computing optimal robust recourse by considering these different combinations is an interesting area for future work. 

Moreover, some prior works explicitly consider the feasibility of recourse by ensuring recourse solutions satisfy feature modification constraints~\citep{KarimiBBV20, JoshiKV+19, PawelczykBK20b}. While we empirically study the effect of imposing such feasibility constraints by post-processing, extending our algorithm to guarantee optimality in the presence of these constraints is left as future work. 

Finally, the main contribution of this work is to provide a better understanding of the \emph{true price} of robustness when providing recourse by studying optimal algorithms for $L^p$-bounded model changes. Studying alternative robustness frameworks such as distributionally robust optimization~\cite{distributionrobust, robusttextbook} and beyond-worst-case analysis~\cite{Roughgarden20, MitzenmacherV20} to further lower the price of recourse is an interesting direction for future work.

\section*{Acknowledgment} We would like to thank Vasilis Gkatzelis for insightful discussions on earlier stages of this work. We also thank the anonymous reviewers for their suggestion on using SmoothGrad for linearization.

\bibliographystyle{plainnat}
\bibliography{bib}
\appendix
\section{Omitted Proofs from Section~\ref{sec:alg-analysis}}
\label{sec:app-algo}

\begin{proposition}
\label{pro:non-convexity}    
The optimization problem in Equation~\ref{eq:xr} is non-convex for all $p\geq 1$ even if the model $\thetaz$ is linear.
\end{proposition}
\begin{proof}
The proof for $p=\infty$ can be found in Appendix C1 of~\cite{KayasthaGJ24}. We provide an example to demonstrate the non-convexity for $p\geq 1$ and $p\ne \infty$. Consider a one-dimensional instance $\xz = [1,1]$, where the second dimension is the unchangeable intercept. Let $\thetaz = [0,0]$, $\ell$ to be the squared loss, and set $\alpha = 0.5$, and $\lambda = 1$. For any recourse, $[x,1]$ (note that the intercept cannot change), the worst-case $\newtheta$ is of the form $[0.5 \text{sign}(x),0]$ when $|x|\geq1$ and $[0, -0.5]$ when $|x|<1$. This is because $\alpha$ is 0.5 and $\thetaz$ is 0 in both dimensions. The price of recourse can be written as a function of $x$ as follows:
$$
J(x)=\begin{cases}
			1/\left(e^{0.5x \text{sign}(x)}\right)^2 + |x-1|, & |x|\geq 1,\\
            1/\left(e^{-0.5}\right)^2 + |x-1|, & |x| < 1.
		 \end{cases}
$$
Plotting this function proves its non-convexity.
\end{proof}

% ========== KK: BEGIN SMOOTHGRAD RESULTS ==========
\renewcommand{\arraystretch}{2.0}
\begin{table*}[ht!]
\centering
\begin{adjustbox}{max width=\linewidth}
\huge{\begin{tabular}{|c|cccc|cccc|}
    \hline
  \multicolumn{1}{|c|}{} & \multicolumn{4}{c|}{German (NN)} & \multicolumn{4}{c|}{Small Business Administration (NN)} \\ \hline
     \multicolumn{1}{|c}{$\alpha$} & \multicolumn{2}{|c}{$0.1$} & \multicolumn{2}{|c}{$0.5$} & \multicolumn{2}{|c}{$0.1$} &\multicolumn{2}{|c|}{$0.5$} \\ \hline
     \multicolumn{1}{|c}{$\lambda$} & \multicolumn{1}{|c}{$0.7$} & \multicolumn{1}{|c}{$0.3$} & \multicolumn{1}{|c}{$0.7$} & \multicolumn{1}{|c}{$0.3$} & \multicolumn{1}{|c}{$0.1$} & \multicolumn{1}{|c}{$0.01$} & \multicolumn{1}{|c}{$0.1$} & \multicolumn{1}{|c|}{$0.01$} \\ \hline 
    Alg1 ($L^1$) 
        & $\mathbf{1.62} \pm 0.34$ 
        & $\mathbf{1.29} \pm 0.34$ 
        & $\mathbf{4.23} \pm 0.65$ 
        & $\mathbf{3.75} \pm 0.42$ 
        & $\mathbf{0.30} \pm 0.10$ 
        & $\mathbf{0.03} \pm 0.02$ 
        & $\mathbf{2.91} \pm 2.04$ 
        & $\mathbf{0.42} \pm 0.56$  \\ \cline{1-1}
    ROAR ($L^1$)
        & \begin{tabular}{c}$1.62 \pm 0.32$ \\ (+0.3\%)\end{tabular}
        & \begin{tabular}{c}$1.42 \pm 0.28$ \\ (+10.8\%)\end{tabular}
        & \begin{tabular}{c}$4.50 \pm 0.72$ \\ (+6.5\%)\end{tabular}
        & \begin{tabular}{c}$3.97 \pm 0.59$ \\ (+5.9\%)\end{tabular}
        & \begin{tabular}{c}$1.19 \pm 0.25$ \\ (+299.1\%)\end{tabular}
        & \begin{tabular}{c}$0.32 \pm 0.05$ \\ (+1001.9\%)\end{tabular}
        & \begin{tabular}{c}$5.37 \pm 1.27$ \\ (+84.5\%)\end{tabular}
        & \begin{tabular}{c}$3.64 \pm 1.56$ \\ (+760.9\%)\end{tabular} \\ \cline{1-1} 
    Alg2 ($L^\infty$) 
        & \begin{tabular}{c}$64.64 \pm 5.63$ \\ (+3896.1\%)\end{tabular}
        & \begin{tabular}{c}$60.49 \pm 2.06$ \\ (+4606.6\%)\end{tabular}
        & \begin{tabular}{c}$84.91 \pm 0.48$ \\ (+1908.7\%)\end{tabular}
        & \begin{tabular}{c}$85.54 \pm 1.15$ \\ (+2183.9\%)\end{tabular}
        & \begin{tabular}{c}$87.79 \pm 4.95$ \\ (+29275.2\%)\end{tabular}
        & \begin{tabular}{c}$87.11 \pm 3.54$ \\ (+300873.4\%)\end{tabular}
        & \begin{tabular}{c}$98.33 \pm 2.93$ \\ (+3278.4\%)\end{tabular}
        & \begin{tabular}{c}$100.05 \pm 0.03$ \\ (+23547.4\%)\end{tabular} \\ \cline{1-1}
    ROAR ($L^\infty)$ 
        & \begin{tabular}{c}$61.41 \pm 0.05$ \\ (+3696.2\%)\end{tabular}
        & \begin{tabular}{c}$62.02 \pm 1.68$ \\ (+4725.8\%)\end{tabular}
        & \begin{tabular}{c}$85.26 \pm 0.84$ \\ (+1917.1\%)\end{tabular}
        & \begin{tabular}{c}$92.91 \pm 4.64$ \\ (+2380.6\%)\end{tabular}
        & \begin{tabular}{c}$88.36 \pm 2.70$ \\ (+29275.15\%)\end{tabular}
        & \begin{tabular}{c}$82.22 \pm 10.84$ \\ (+283970.1\%)\end{tabular}
        & \begin{tabular}{c}$96.96 \pm 5.58$ \\ (+3231.1\%)\end{tabular}
        & \begin{tabular}{c}$100.13 \pm 0.05$ \\ (+23567.2\%)\end{tabular} 
    \\\cline{1-1}  \hline 
\end{tabular}
}
\end{adjustbox}
\caption{The price of recourse for neural network models using Smoothgrad approximation. The columns correspond to combinations of $\alpha$ and $\lambda$ for each of the datasets. Each row represents the price of recourse returned by each of the algorithms, averaged over all the test instances in each dataset. The smallest price is shown in bold in each column, and percentages indicate the increase in price compared to the smallest value. \label{table:price_nn_sg}}
\end{table*}
\renewcommand{\arraystretch}{1}
% ========== KK: END SMOOTHGRAD RESULTS ==========

\section{Omitted Details from Section~\ref{sec:exp}}
\label{sec:app-exp}

\subsection{Additional Experimental Details}
\label{sec:app-exp-details}
The experiments are conducted on three Apple MacBook Pro laptops. The first MacBook has an Apple M1 Max chip with 10 cores and 32 GB of memory, the second MacBook has an Apple M1 Pro chip with 8 cores and 16 GB of memory, and the third MacBook has an Apple M3 Max chip with 14 cores and 36 GB of memory. While we use all the instances for the German Credit dataset in both logistic regression and neural network models, for the Small Business Administration dataset, we subsample $24\%$ of instances in the logistic regression models and $16\%$ in the neural network models.

The average time to compute a robust recourse varies between datasets and different algorithms. In the German Credit dataset, the average run times are 188.4 seconds for Algorithm~\ref{alg:l1-l1}, 0.0001 seconds for Algorithm~\ref{alg:l1-linf}, 0.3 seconds for ROAR $(L^1)$, and 0.8 seconds for ROAR $(L^\infty)$. In the Small Business Administration dataset, the average run times are 667.5 seconds for Algorithm~\ref{alg:l1-l1}, 0.0001 seconds for Algorithm~\ref{alg:l1-linf}, 0.3 seconds for ROAR $(L^1)$, and 1.5 seconds for ROAR $(L^\infty)$. We observe a high difference in runtime for Algorithm~\ref{alg:l1-l1} due to the polynomial dependency of the running time on the number of dimensions $d$ as described in Section~\ref{sec:alg-analysis}.

\subsection{Linearization Using SmoothGrad~\cite{SmilkovTKVW17}}
\label{sec:app-exp-smoothgrad}
Instead of using LIME to compute a local surrogate for the model, we can use an alternate linearization. When the models are differentiable, one such approach is SmoothGrad, which averages out the gradients in a local neighborhood of the instance of interest. In Table~\ref{table:price_nn_sg}, we replicated the same set of experiments as in Table~\ref{table:price_nn}, only replacing the LIME approximation with SmoothGrad~\cite{SmilkovTKVW17}. We observe that \emph{(i)} Algorithm~\ref{alg:l1-l1} still outperforms the baselines and \emph{(ii)} the price of recourse barely changes after this modification.

\subsection{Additional Details About the Trade-off Between Validity and Implementation Cost}
\label{sec:app-exp-trade-off}
For the German Credit dataset with the logistic regression model, where $\alpha = 0.1$\, we use $0.001 \le \lambda \le 0.5$ bound, and with $\alpha = 0.5$, we use $0.004 \le \lambda \le 0.5$ bound. For German Credit dataset with neural network models $(\alpha=0.1, 0.5)$, we use $0.01 \le \lambda \le 3.0$ bound. For the Small Business Administration dataset, the logistic regression model uses $0.01 \le \lambda \le 2.1$ bound, and the neural network model uses $0.01 \le \lambda \le 3.5$ bound $(\alpha = 0.1, 0.5)$.

% ========== German Credit dataset ==========
\begin{figure*}[ht!]
    \centering
    % INSTANCE-WISE VALIDITY
    \begin{subfigure}[b]{0.45\textwidth}
        \centering
        \includegraphics[width=0.85\textwidth]{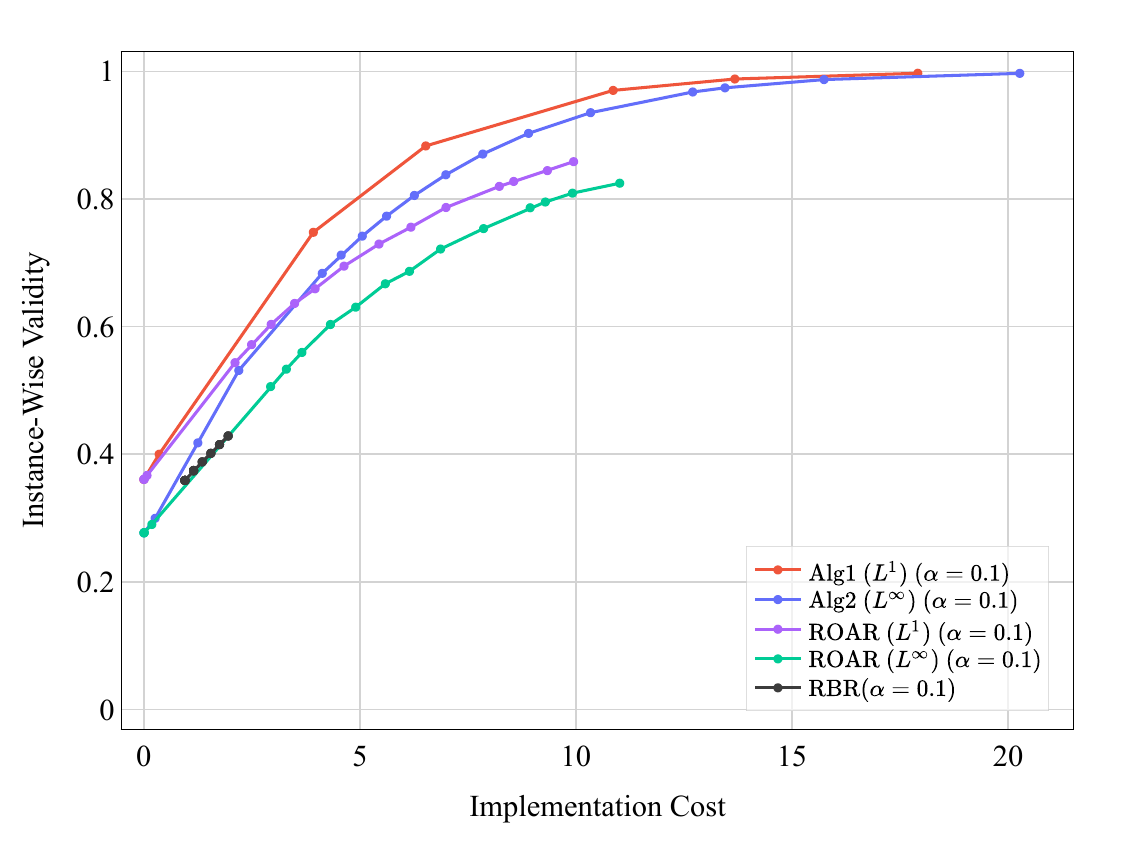}
        \caption{Instance-Wise Validity, Logistic Regression}
        \label{fig:cost_validity_tradeoff_lr_german_instance_app}
    \end{subfigure}
    \begin{subfigure}[b]{0.45\textwidth}
        \centering
        \includegraphics[width=0.85\textwidth]{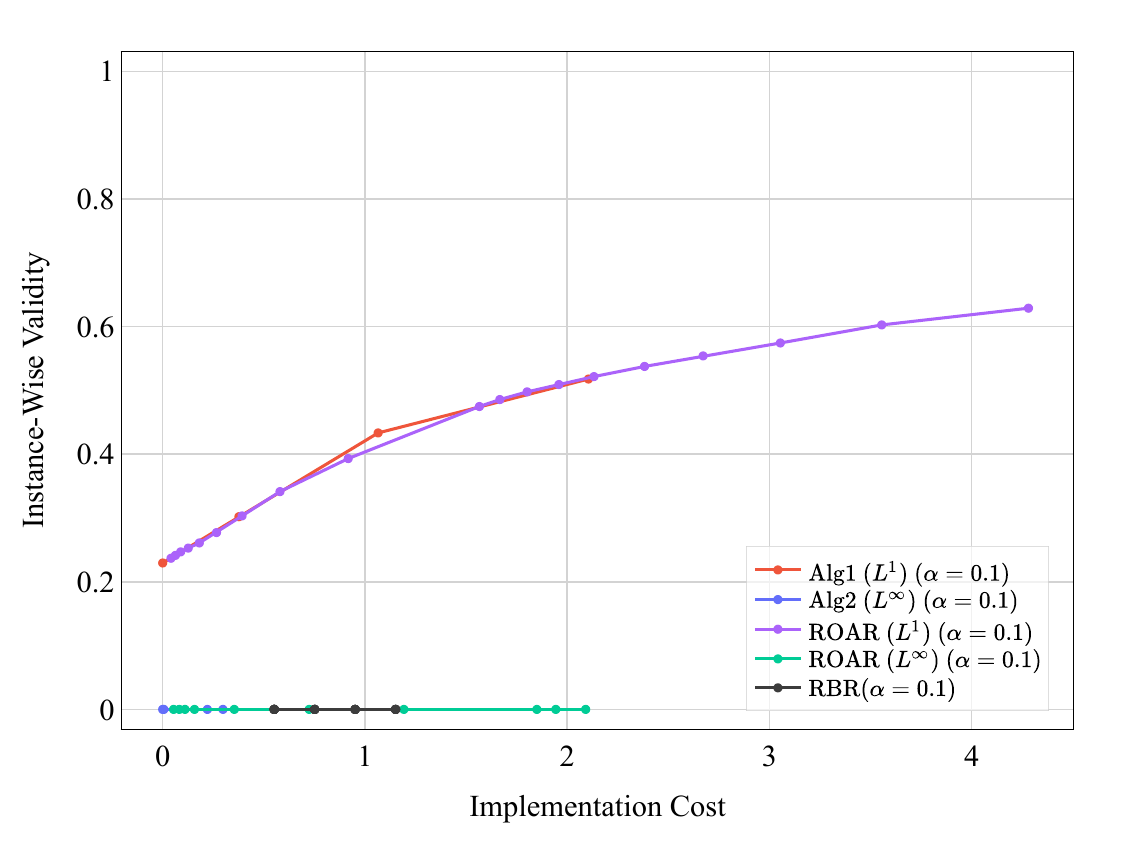}
        \caption{Instance-Wise Validity, Neural Network}
        \label{fig:cost_validity_tradeoff_nn_german_instance_app}
    \end{subfigure}
    % POPULATION-WISE VALIDITY
    \begin{subfigure}[b]{0.45\textwidth}
        \centering
        \includegraphics[width=0.85\textwidth]{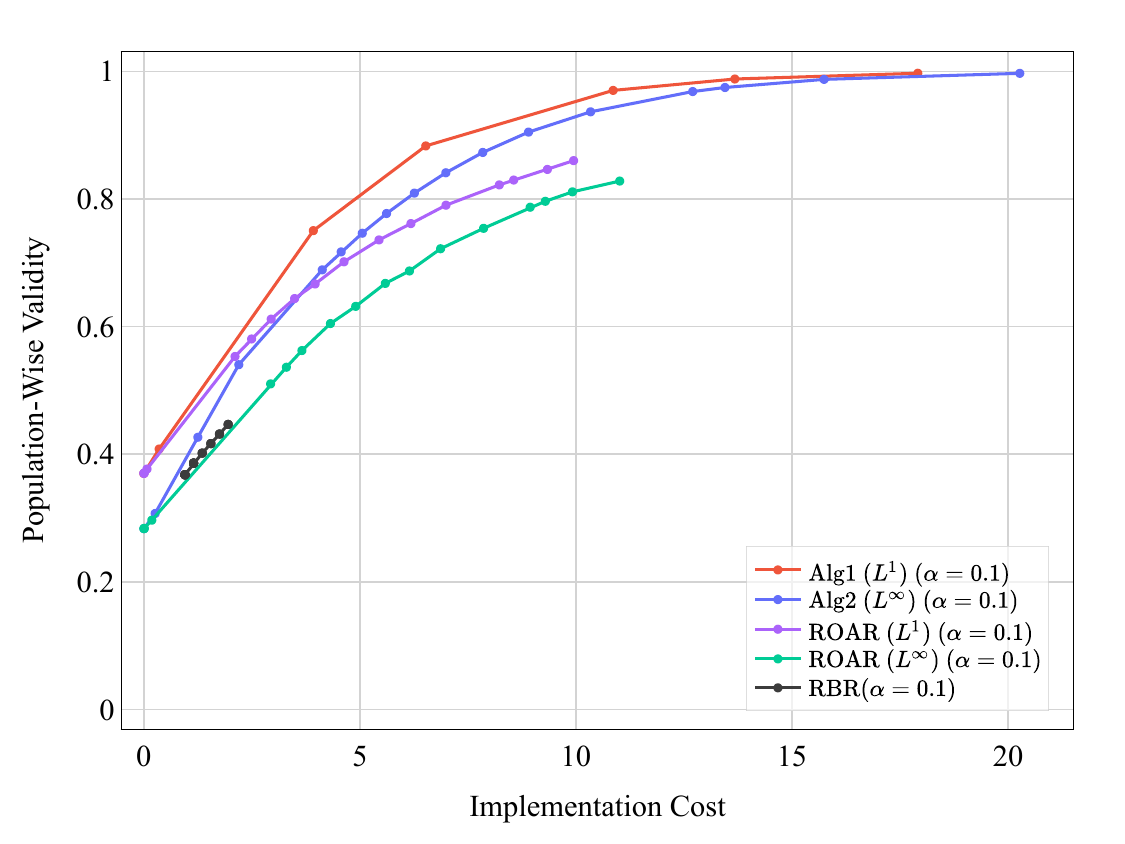}
        \caption{Population-Wise Validity, Logistic Regression}
        \label{fig:cost_validity_tradeoff_lr_german_population_app}
    \end{subfigure}
    \begin{subfigure}[b]{0.45\textwidth}
        \centering
        \includegraphics[width=0.85\textwidth]{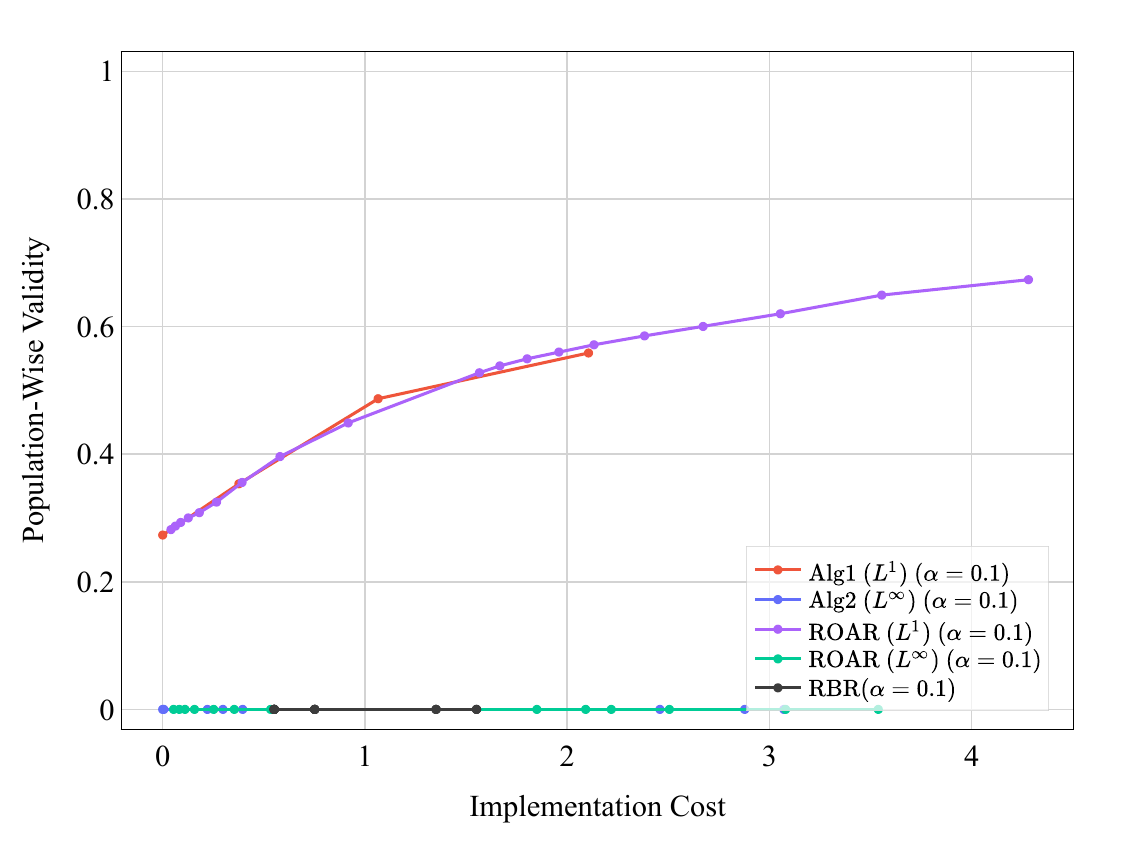}
        \caption{Population-Wise Validity, Neural Network}
        \label{fig:cost_validity_tradeoff_nn_german_population_app}
    \end{subfigure}
    % CURRENT VALIDITY
    \begin{subfigure}[b]{0.45\textwidth}
        \centering
        \includegraphics[width=0.85\textwidth]{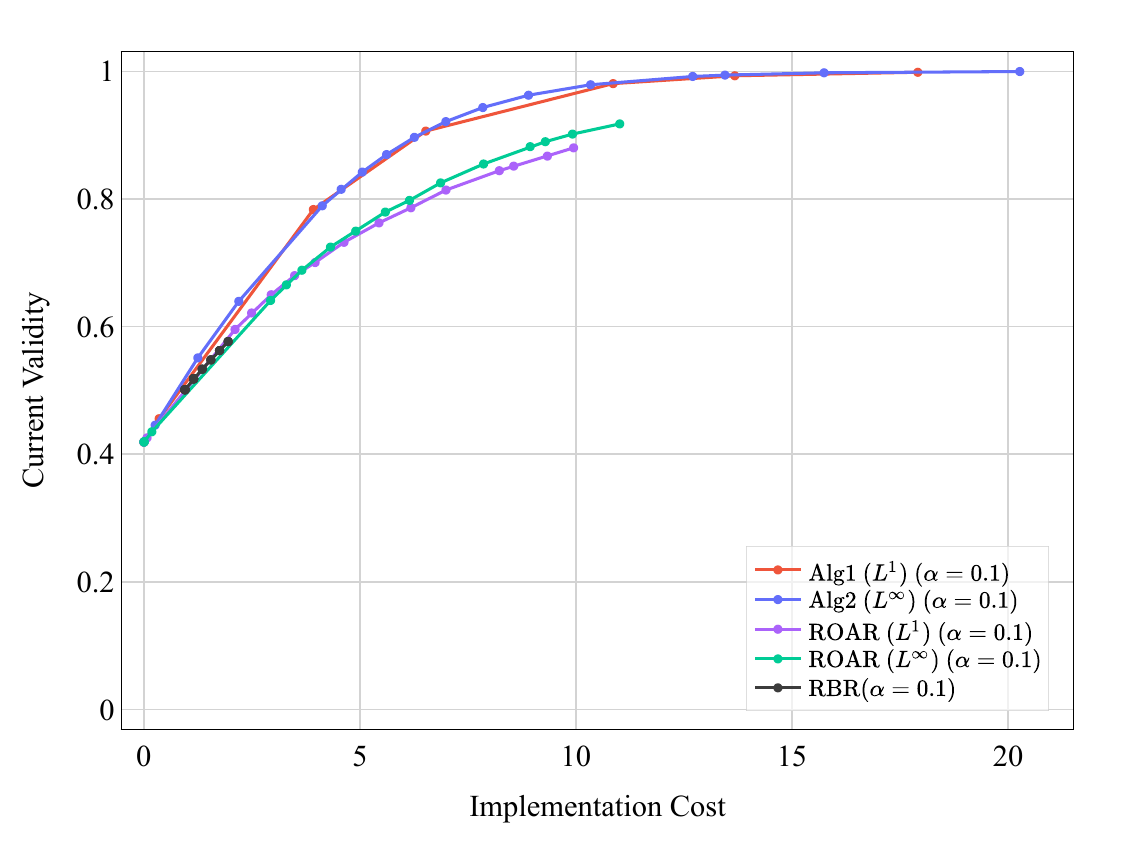}
        \caption{Current Validity, Logistic Regression}
        \label{fig:cost_validity_tradeoff_lr_german_current_app}
    \end{subfigure}
    \begin{subfigure}[b]{0.45\textwidth}
        \centering
        \includegraphics[width=0.85\textwidth]{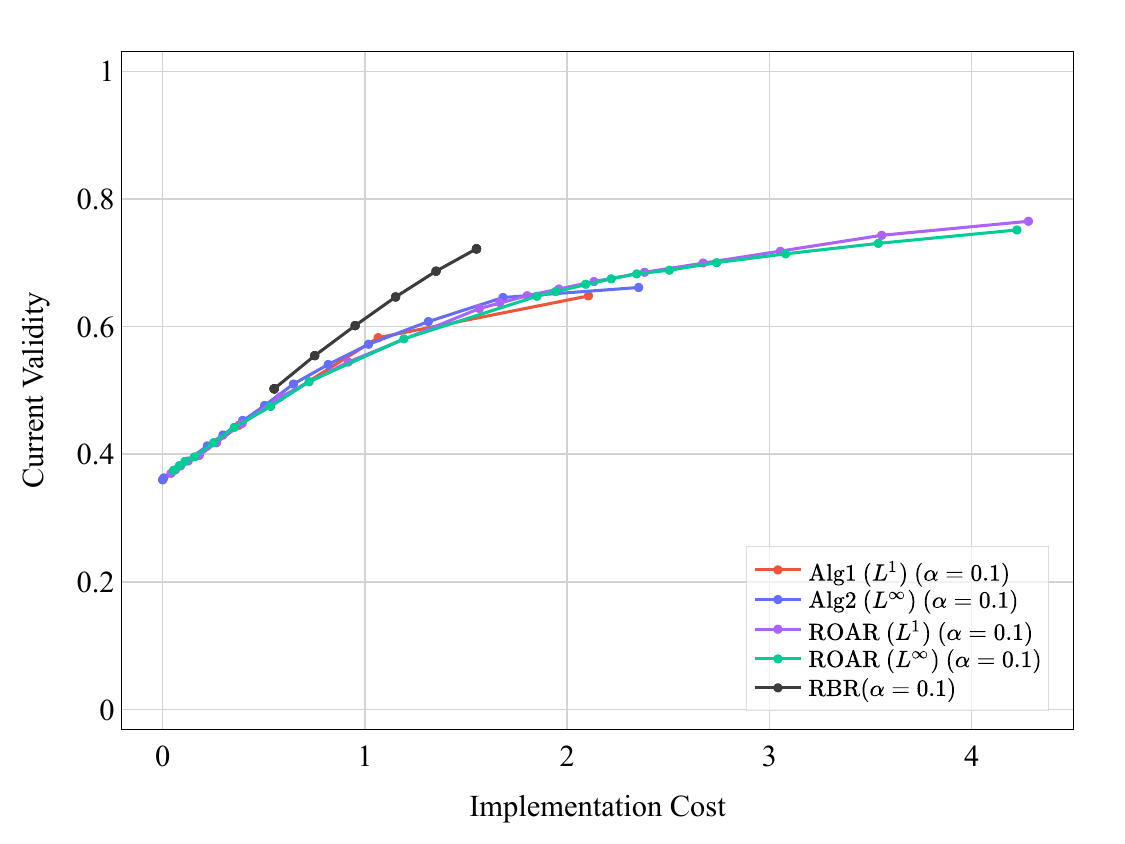}
        \caption{Current Validity, Neural Network}
        \label{fig:cost_validity_tradeoff_nn_german_current_app}
    \end{subfigure}
    % FUTURE VALIDITY
    \begin{subfigure}[b]{0.45\textwidth}
        \centering
        \includegraphics[width=0.85\textwidth]{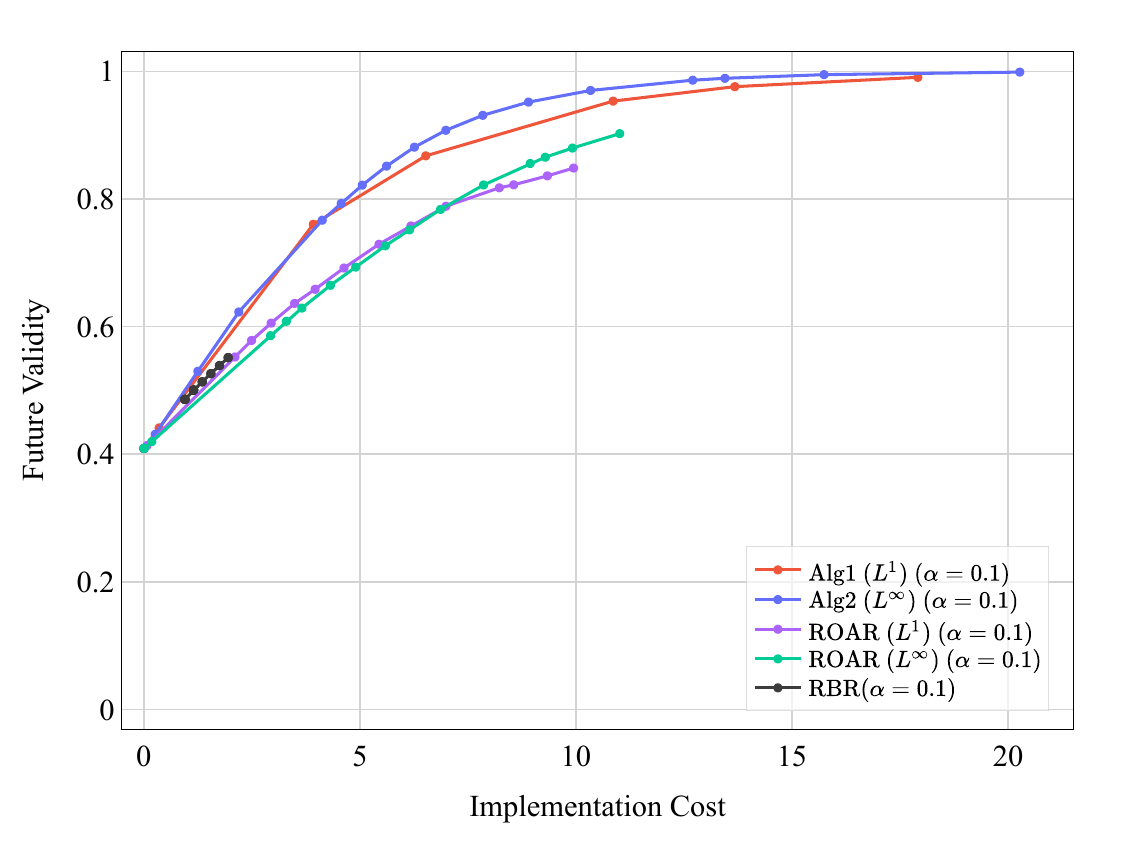}
        \caption{Future Validity, Logistic Regression}
        \label{fig:cost_validity_tradeoff_lr_german_future_app}
    \end{subfigure}
    \begin{subfigure}[b]{0.45\textwidth}
        \centering
        \includegraphics[width=0.85\textwidth]{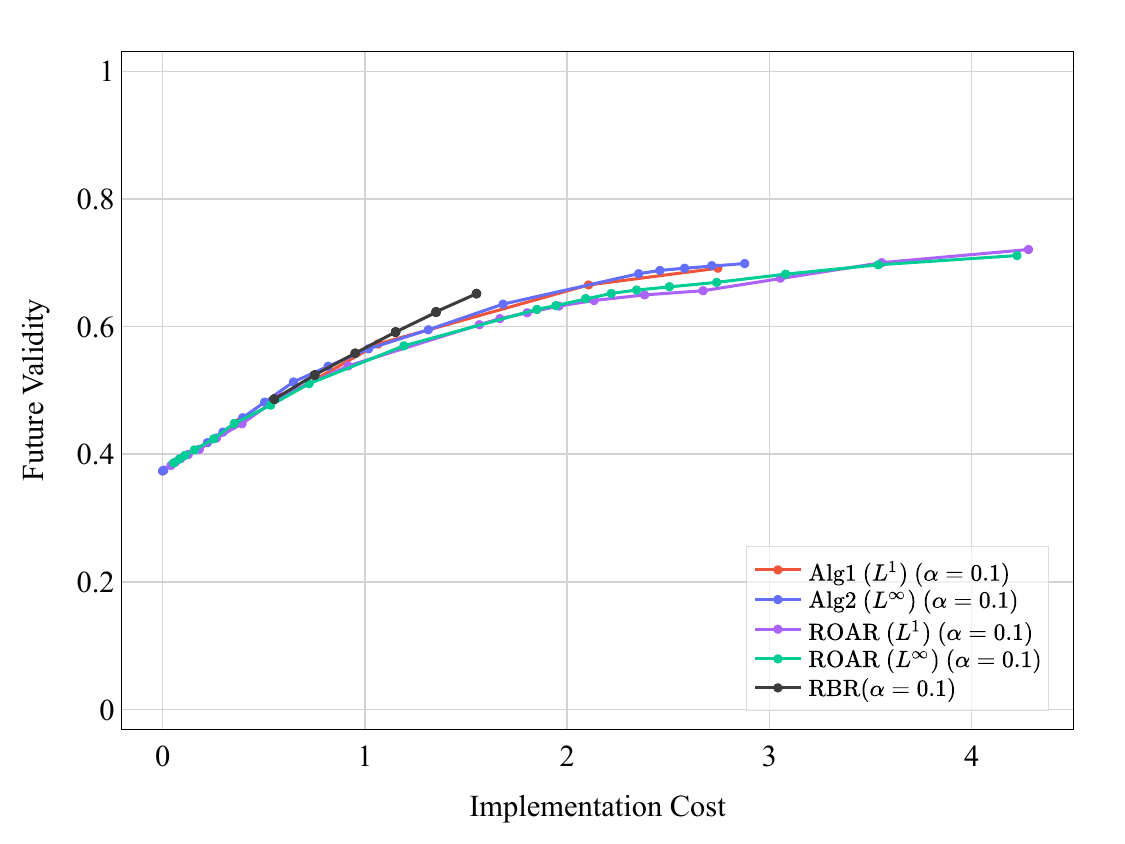}
        \caption{Future Validity, Neural Network}
        \label{fig:cost_validity_tradeoff_nn_german_future_app}
    \end{subfigure}
    \caption{The frontier of the trade-off between validity and implementation cost on the German Credit dataset with $\alpha=0.1$. The left and right columns correspond to logistic regression and neural network models. Each row corresponds to a different measure of validity. In each subfigure, curves show the trade-off for different algorithms.
    \label{fig:cost_validity_tradeoff_german_alpha_0.1_app}}
\end{figure*}

\begin{figure*}[ht!]
    \centering
    % INSTANCE-WISE VALIDITY
    \begin{subfigure}[b]{0.45\textwidth}
        \centering
        \includegraphics[width=0.85\textwidth]{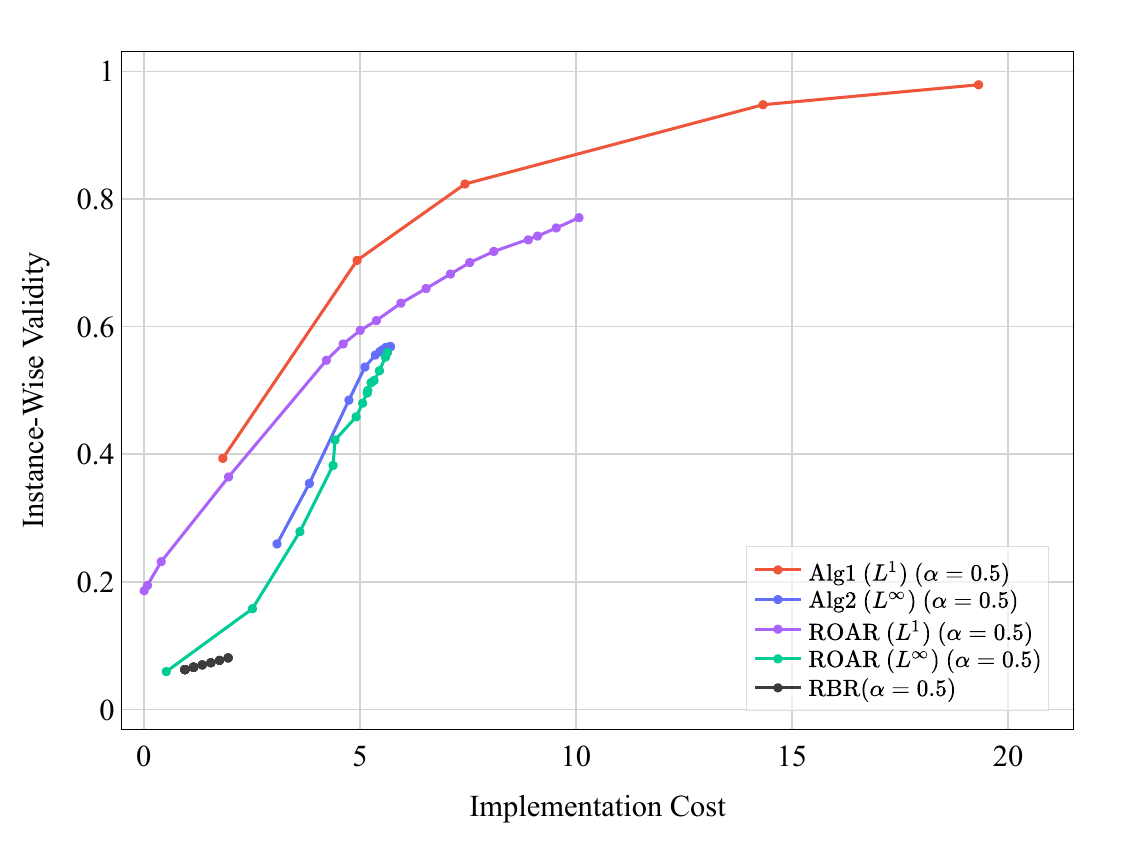}
        \caption{Instance-Wise Validity, Logistic Regression}
        \label{fig:cost_validity_tradeoff_lr_german_instance_alpha_0.5_app}
    \end{subfigure}
    \begin{subfigure}[b]{0.45\textwidth}
        \centering
        \includegraphics[width=0.85\textwidth]{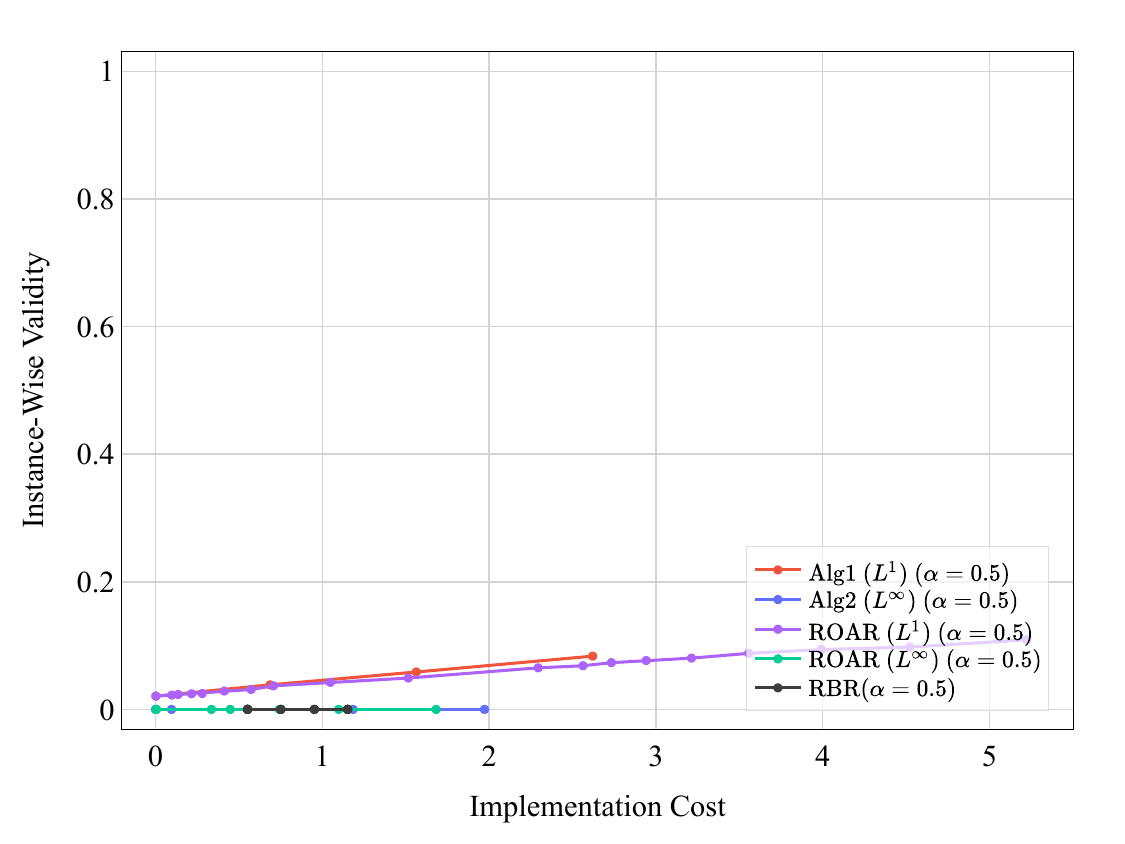}
        \caption{Instance-Wise Validity, Neural Network}
        \label{fig:cost_validity_tradeoff_nn_german_instance_alpha_0.5_app}
    \end{subfigure}
    % POPULATION-WISE VALIDITY
    \begin{subfigure}[b]{0.45\textwidth}
        \centering
        \includegraphics[width=0.85\textwidth]{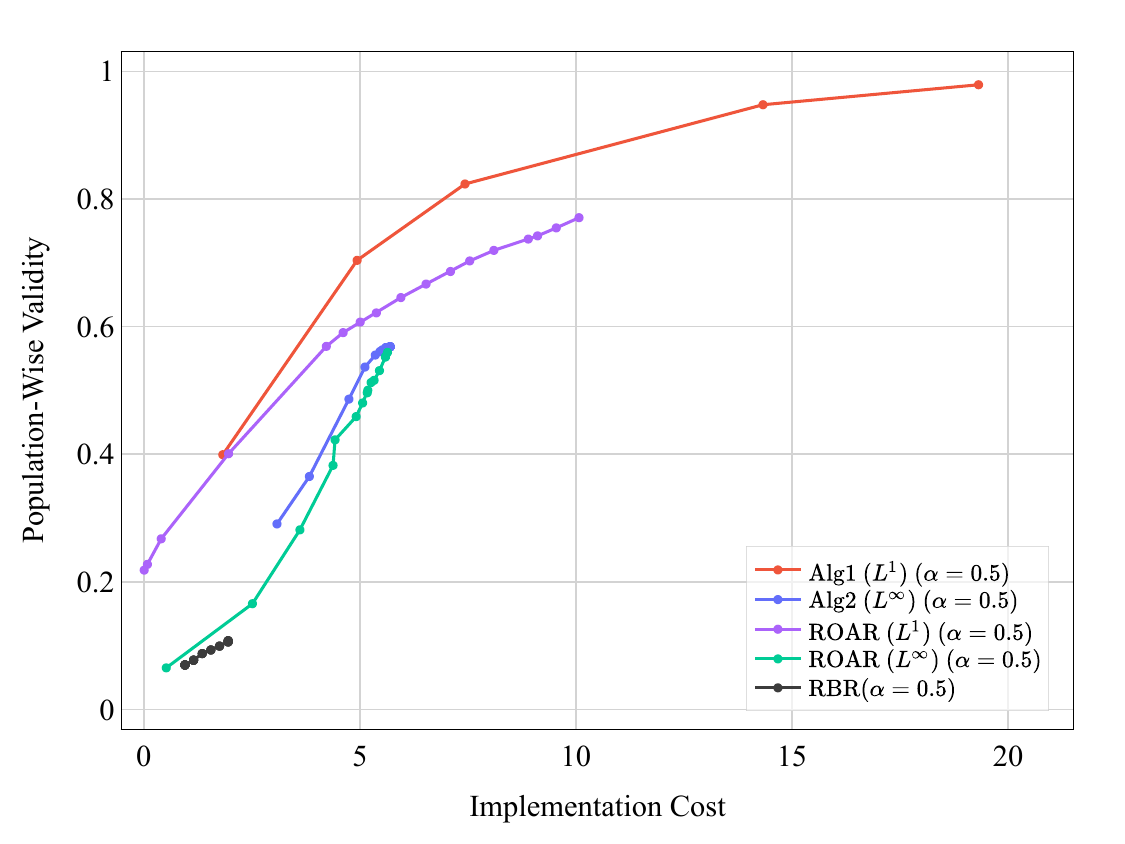}
        \caption{Population-Wise Validity, Logistic Regression}
        \label{fig:cost_validity_tradeoff_lr_german_population_alpha_0.5_app}
    \end{subfigure}
    \begin{subfigure}[b]{0.45\textwidth}
        \centering
        \includegraphics[width=0.85\textwidth]{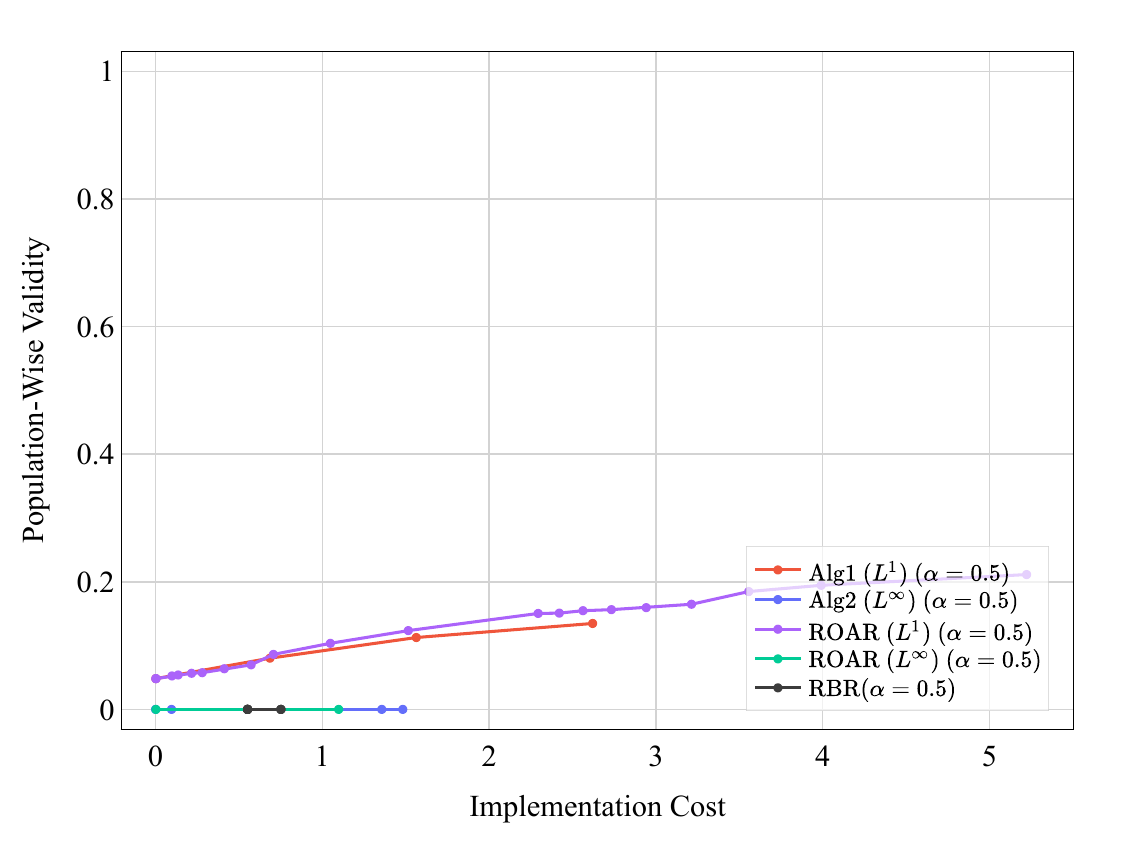}
        \caption{Population-Wise Validity, Neural Network}
        \label{fig:cost_validity_tradeoff_nn_german_population_alpha_0.5_app}
    \end{subfigure}
    % CURRENT VALIDITY
    \begin{subfigure}[b]{0.45\textwidth}
        \centering
        \includegraphics[width=0.85\textwidth]{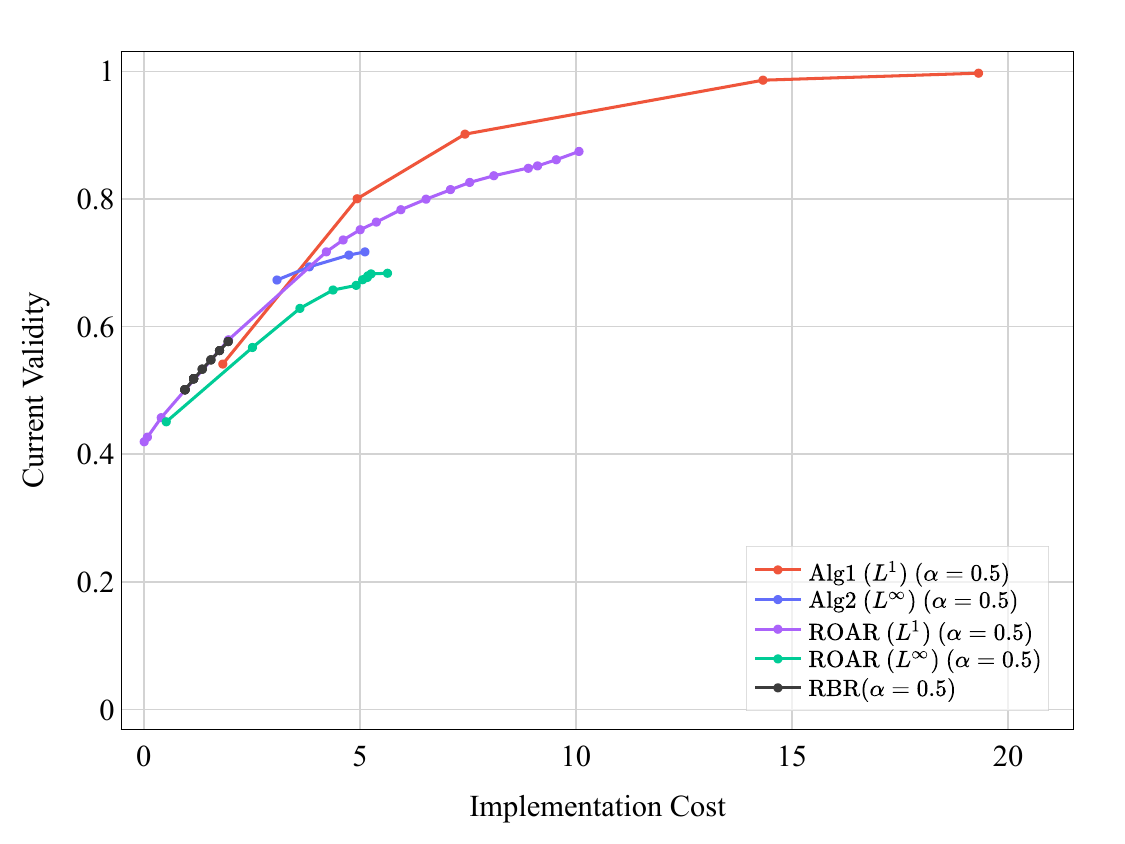}
        \caption{Current Validity, Logistic Regression}
        \label{fig:cost_validity_tradeoff_lr_german_current_alpha_0.5_app}
    \end{subfigure}
    \begin{subfigure}[b]{0.45\textwidth}
        \centering
        \includegraphics[width=0.85\textwidth]{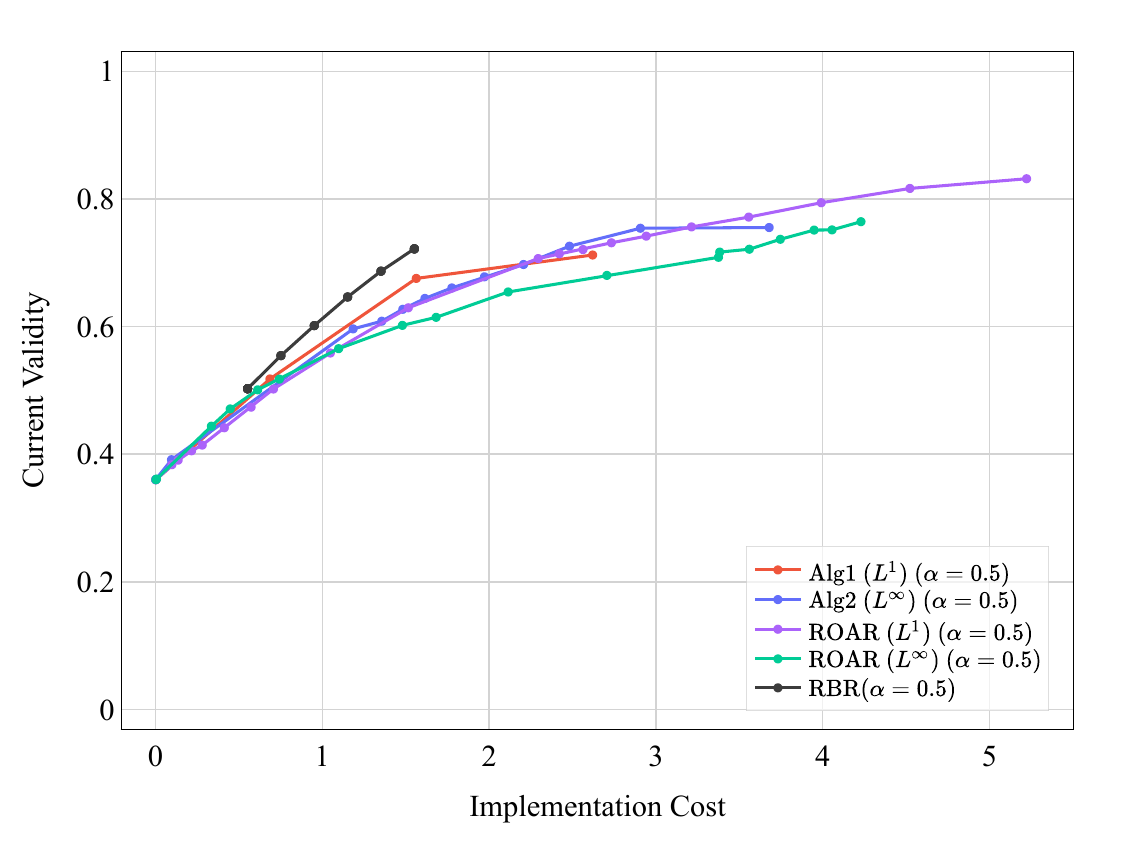}
        \caption{Current Validity, Neural Network}
        \label{fig:cost_validity_tradeoff_nn_german_current_alpha_0.5_app}
    \end{subfigure}
    % FUTURE VALIDITY
    \begin{subfigure}[b]{0.45\textwidth}
        \centering
        \includegraphics[width=0.85\textwidth]{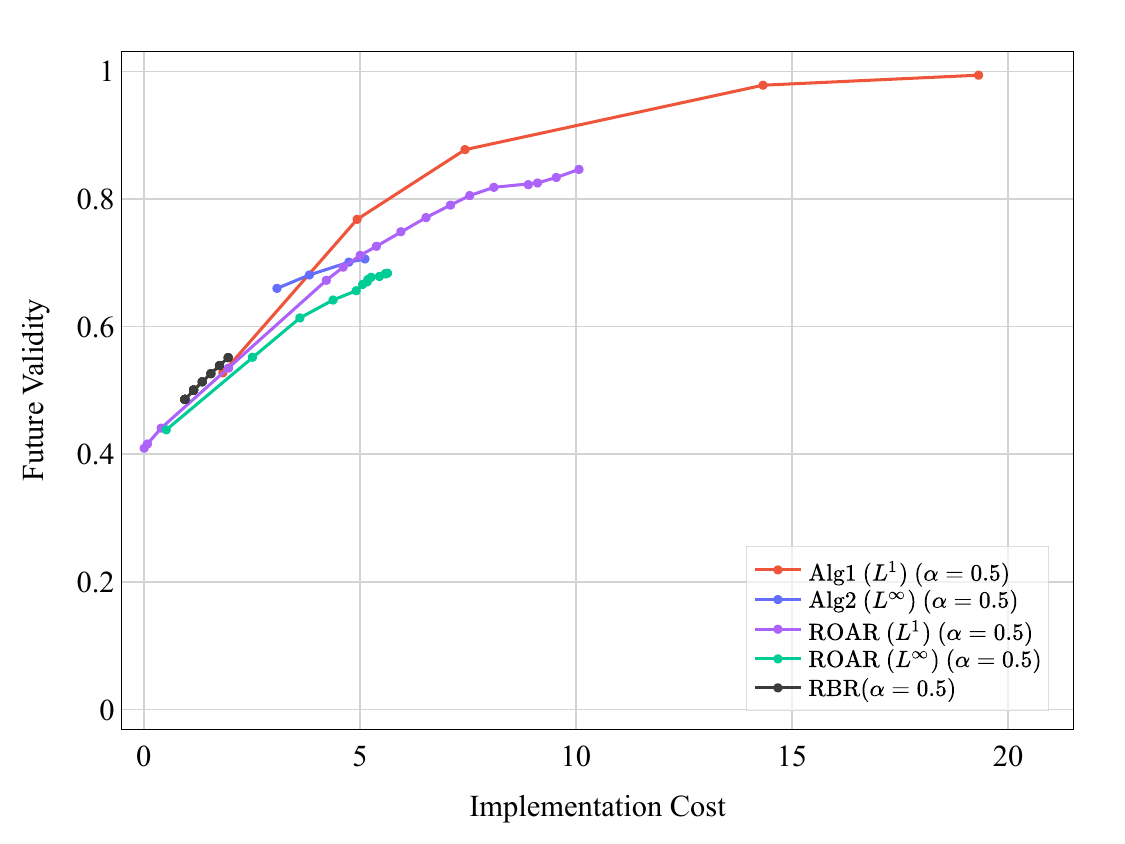}
        \caption{Future Validity, Logistic Regression}
        \label{fig:cost_validity_tradeoff_lr_german_future_alpha_0.5_app}
    \end{subfigure}
    \begin{subfigure}[b]{0.45\textwidth}
        \centering
        \includegraphics[width=0.85\textwidth]{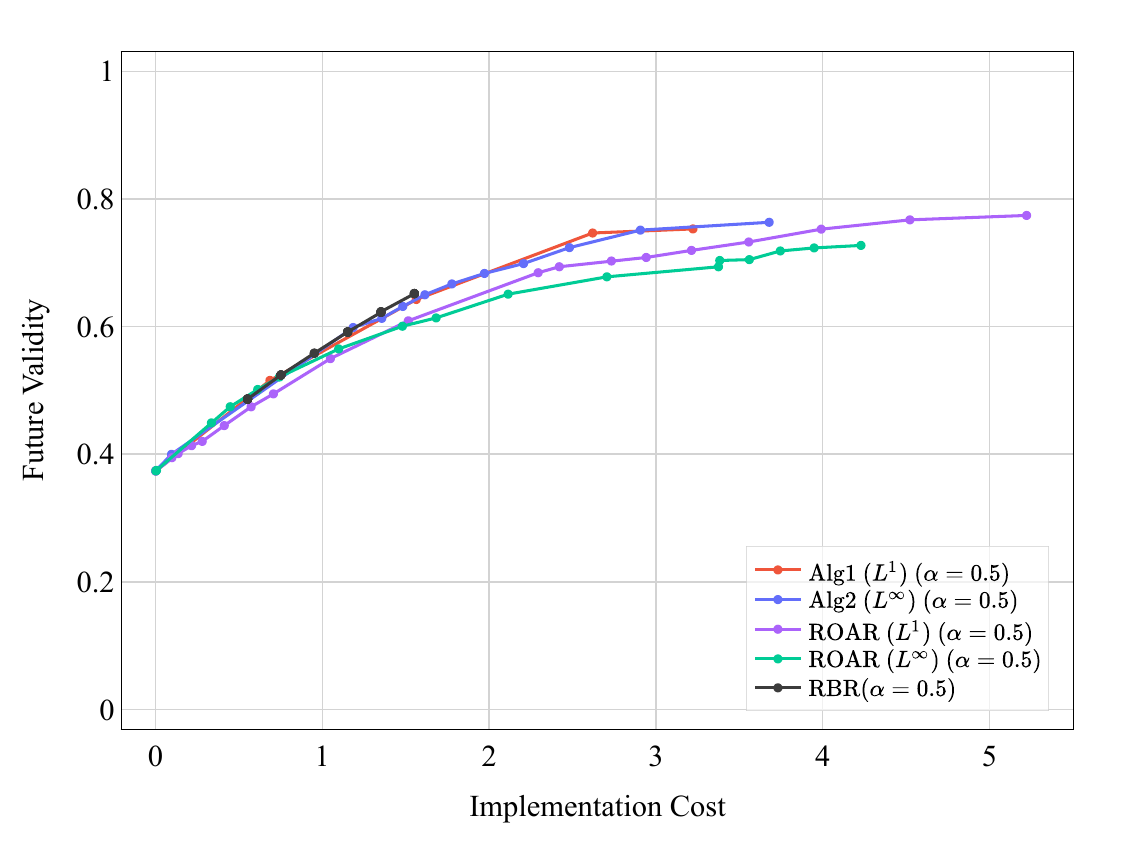}
        \caption{Future Validity, Neural Network}
        \label{fig:cost_validity_tradeoff_nn_german_future_alpha_0.5_app}
    \end{subfigure}
    \caption{The frontier of the trade-off between validity and implementation cost on the German Credit dataset with $\alpha=0.5$. The left and right columns correspond to logistic regression and neural network models. Each row corresponds to a different measure of validity. In each subfigure, curves show the trade-off for different algorithms.
    \label{fig:cost_validity_tradeoff_german_alpha_0.5_app}}
\end{figure*}

% ========== SBA DATASET ==========
\begin{figure*}[ht!]
    \centering
    % INSTANCE-WISE VALIDITY
    \begin{subfigure}[b]{0.45\textwidth}
        \centering
        \includegraphics[width=0.85\textwidth]{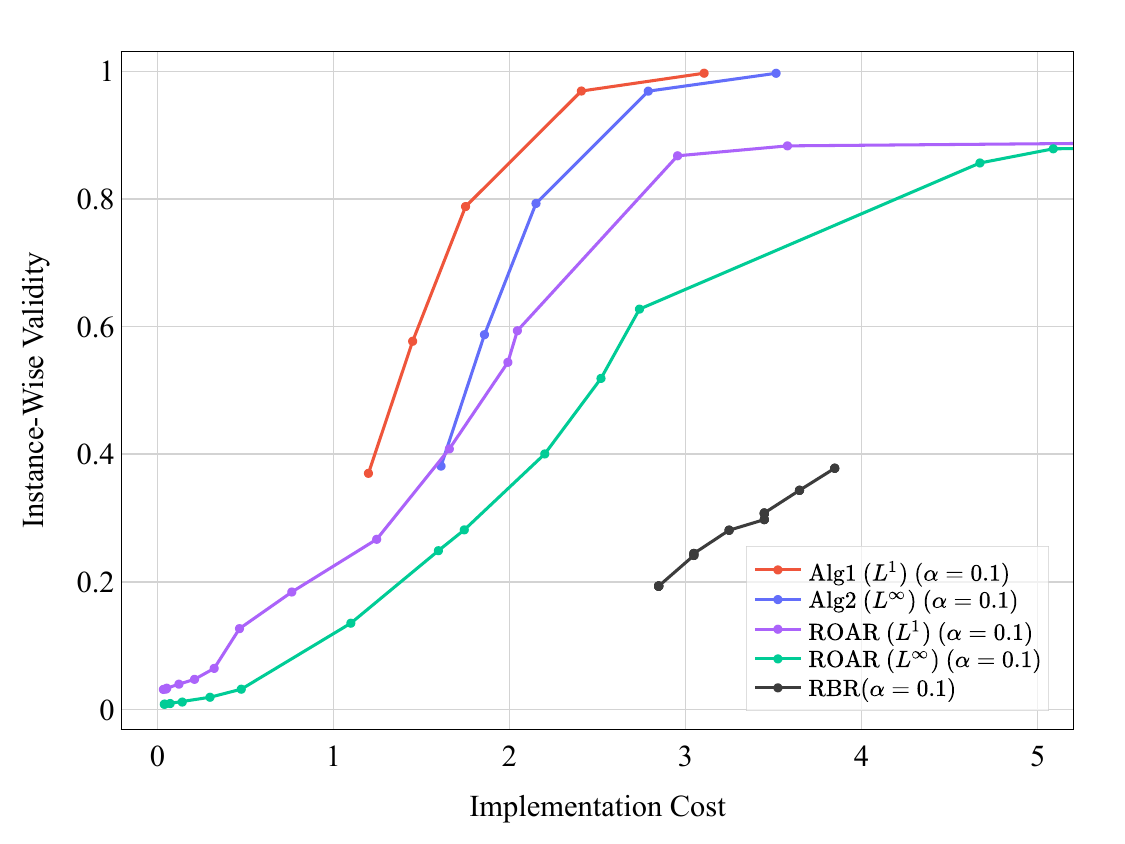}
        \caption{Instance-Wise Validity, Logistic Regression}
        \label{fig:cost_validity_tradeoff_lr_sba_instance_app}
    \end{subfigure}
    \begin{subfigure}[b]{0.45\textwidth}
        \centering
        \includegraphics[width=0.85\textwidth]{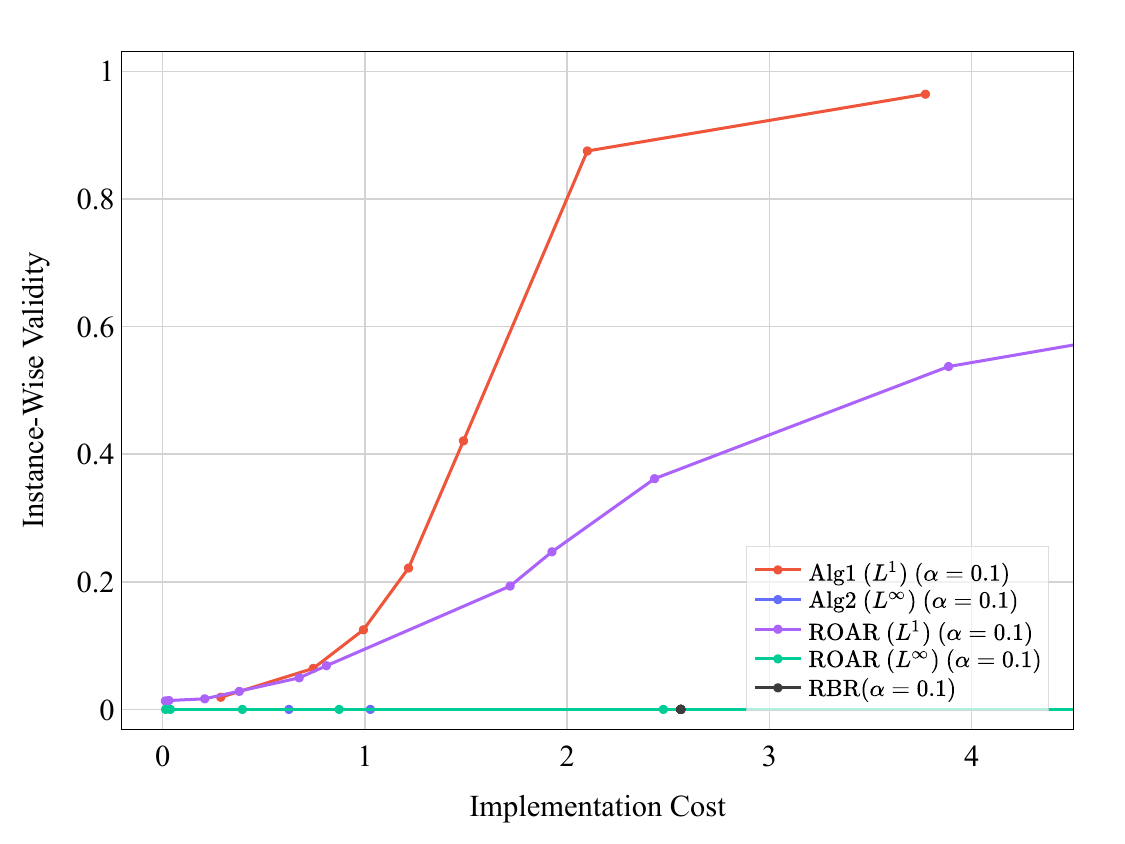}
        \caption{Instance-Wise Validity, Neural Network}
        \label{fig:cost_validity_tradeoff_nn_sba_instance_app}
    \end{subfigure}
    % POPULATION-WISE VALIDITY
    \begin{subfigure}[b]{0.45\textwidth}
        \centering
        \includegraphics[width=0.85\textwidth]{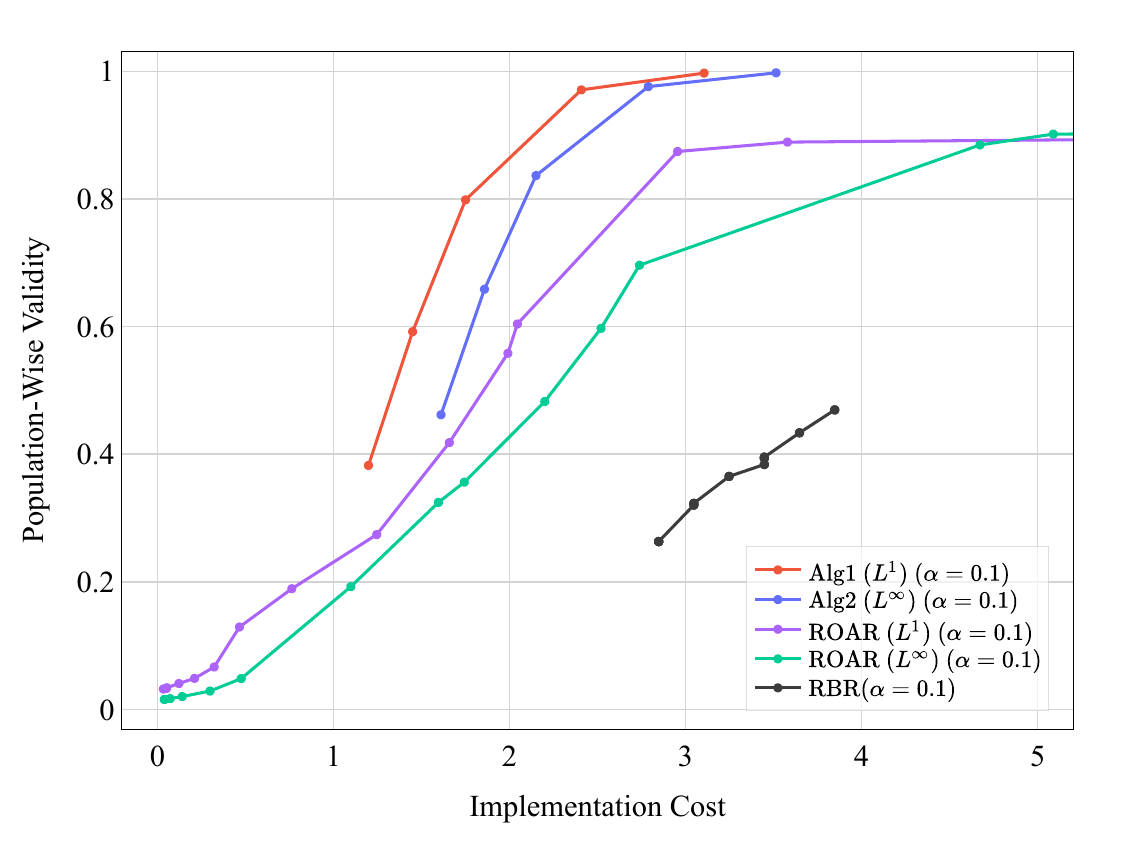}
        \caption{Population-Wise Validity, Logistic Regression}
        \label{fig:cost_validity_tradeoff_lr_sba_population_app}
    \end{subfigure}
    \begin{subfigure}[b]{0.45\textwidth}
        \centering
        \includegraphics[width=0.85\textwidth]{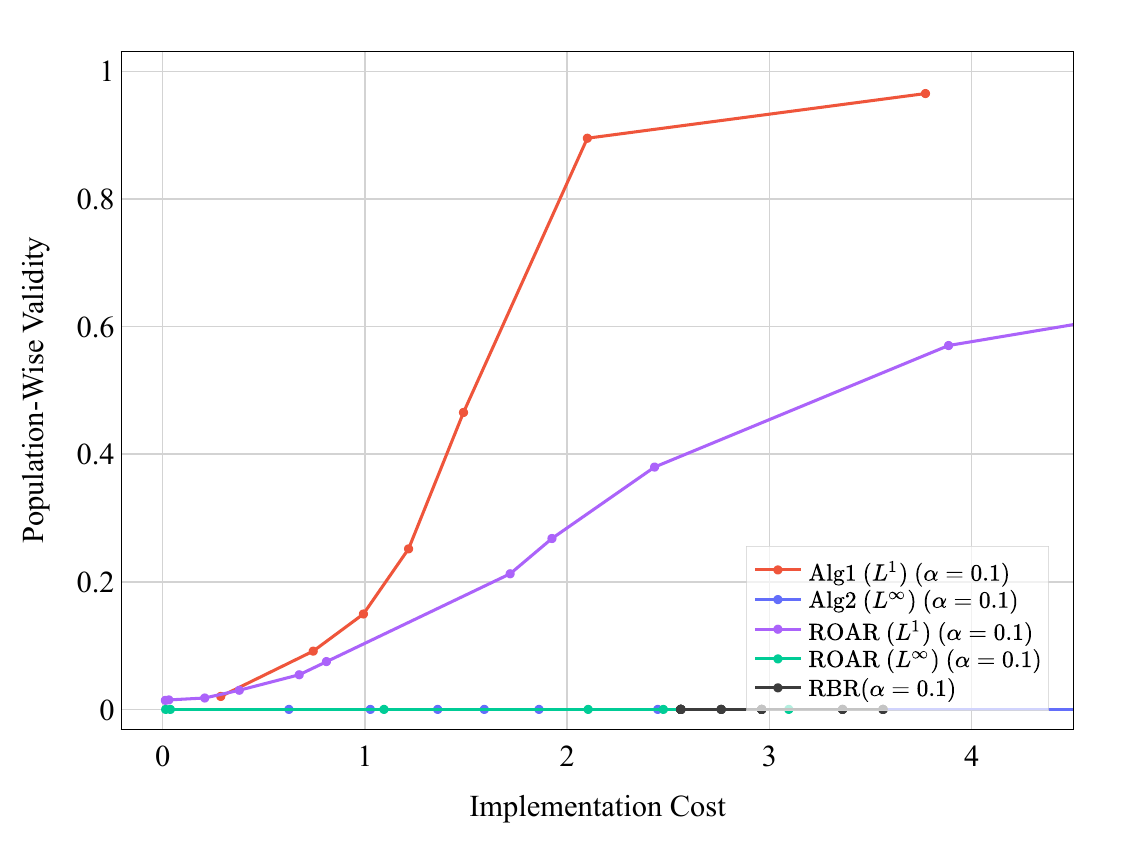}
        \caption{Population-Wise Validity, Neural Network}
        \label{fig:cost_validity_tradeoff_nn_sba_population_app}
    \end{subfigure}
    % CURRENT VALIDITY
    \begin{subfigure}[b]{0.45\textwidth}
        \centering
        \includegraphics[width=0.85\textwidth]{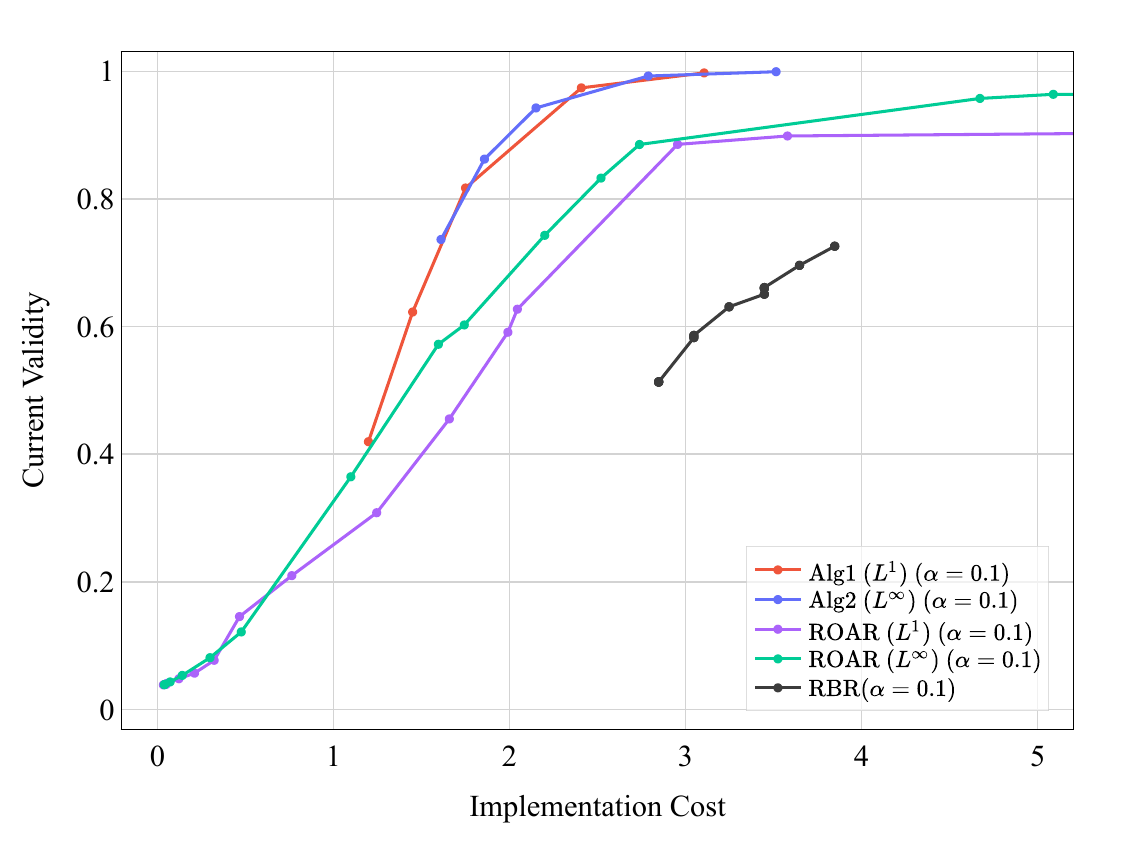}
        \caption{Current Validity, Logistic Regression}
        \label{fig:cost_validity_tradeoff_lr_sba_current_app}
    \end{subfigure}
    \begin{subfigure}[b]{0.45\textwidth}
        \centering
        \includegraphics[width=0.85\textwidth]{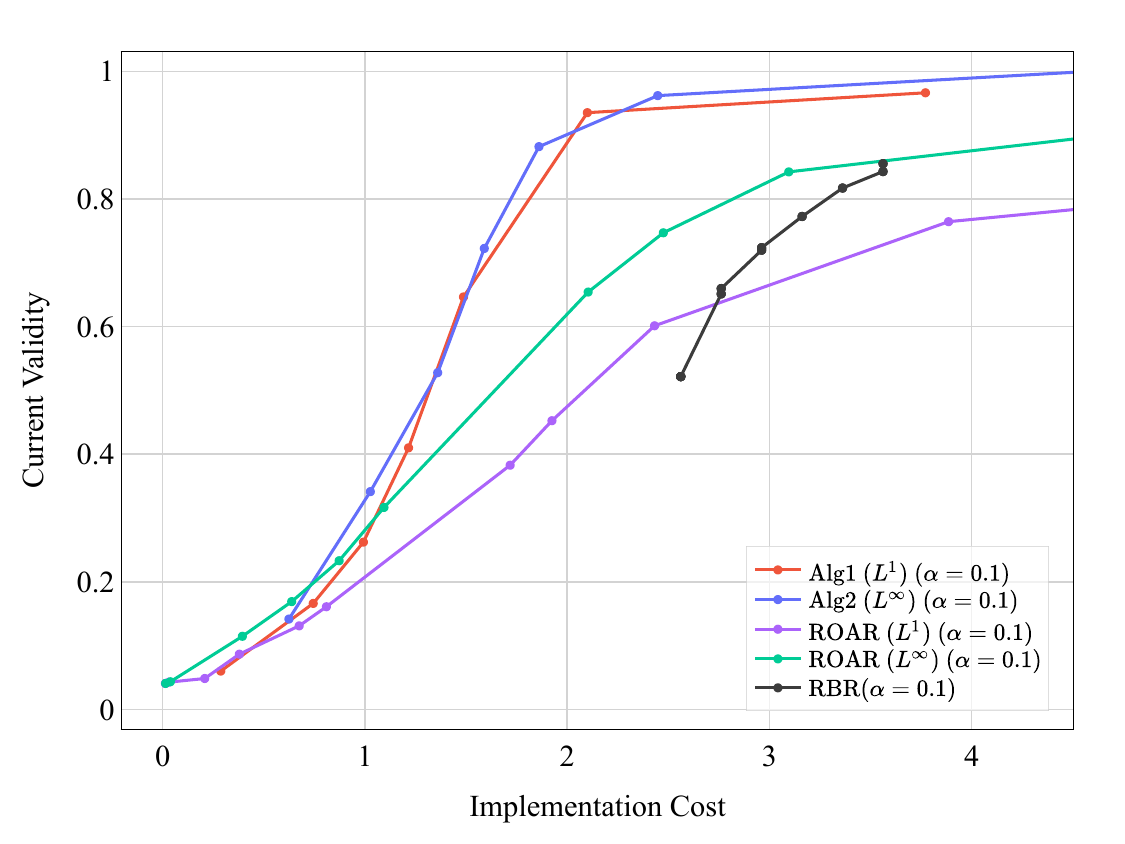}
        \caption{Current Validity, Neural Network}
        \label{fig:cost_validity_tradeoff_nn_sba_current_app}
    \end{subfigure}
    % FUTURE VALIDITY
    \begin{subfigure}[b]{0.45\textwidth}
        \centering
        \includegraphics[width=0.85\textwidth]{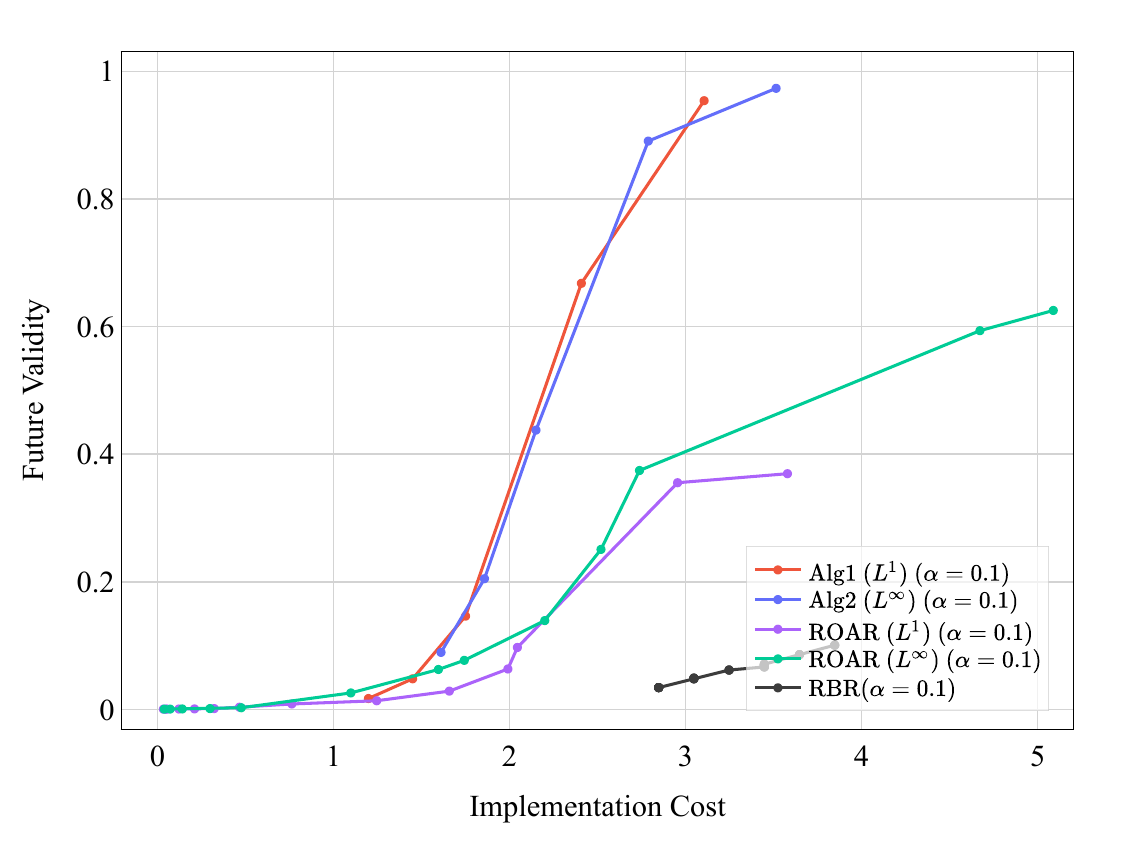}
        \caption{Future Validity, Logistic Regression}
        \label{fig:cost_validity_tradeoff_lr_sba_future_app}
    \end{subfigure}
    \begin{subfigure}[b]{0.45\textwidth}
        \centering
        \includegraphics[width=0.85\textwidth]{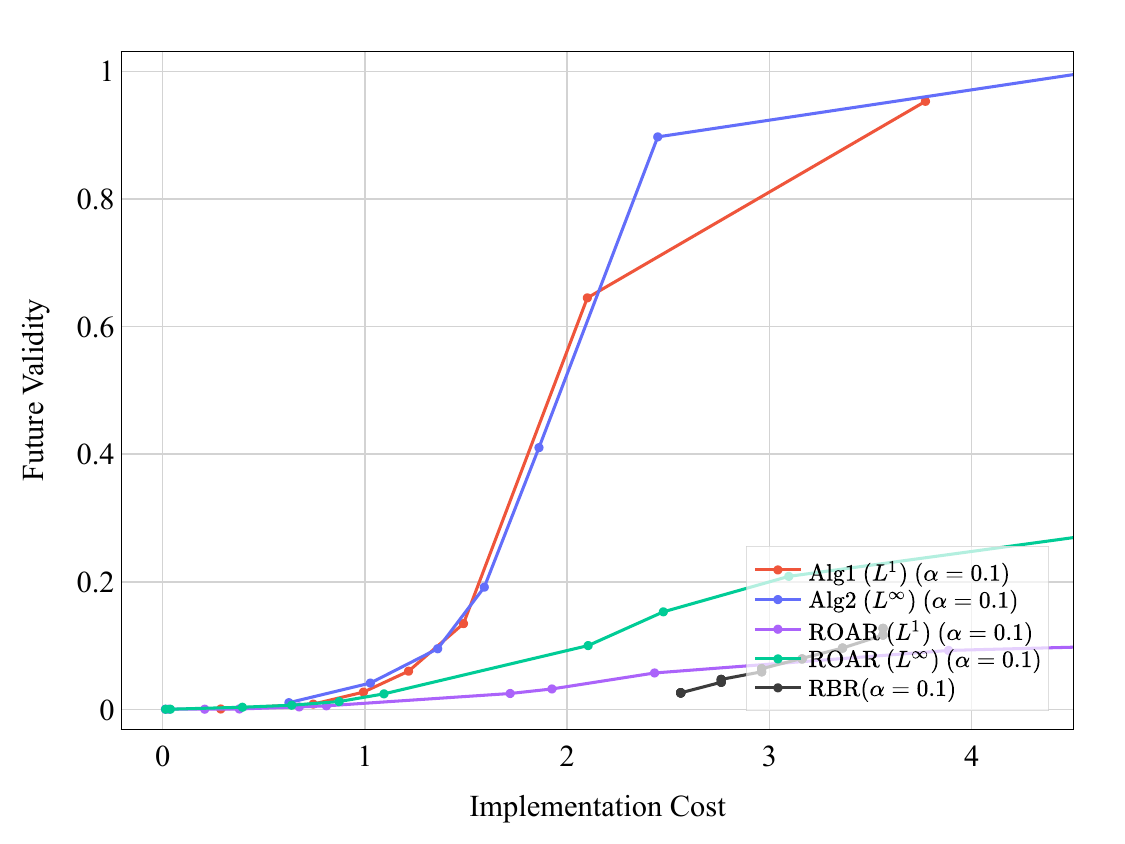}
        \caption{Future Validity, Neural Network}
        \label{fig:cost_validity_tradeoff_nn_sba_future_app}
    \end{subfigure}
    \caption{The frontier of the trade-off between validity and implementation cost on the Small Business Administration dataset with $\alpha=0.1$. The left and right columns correspond to logistic regression and neural network models. Each row corresponds to a different measure of validity. In each subfigure, curves show the trade-off for different algorithms. 
    \label{fig:cost_validity_tradeoff_sba_alpha_0.1_app}}
\end{figure*}

\begin{figure*}[ht!]
    \centering
    % INSTANCE-WISE VALIDITY
    \begin{subfigure}[b]{0.45\textwidth}
        \centering
        \includegraphics[width=0.85\textwidth]{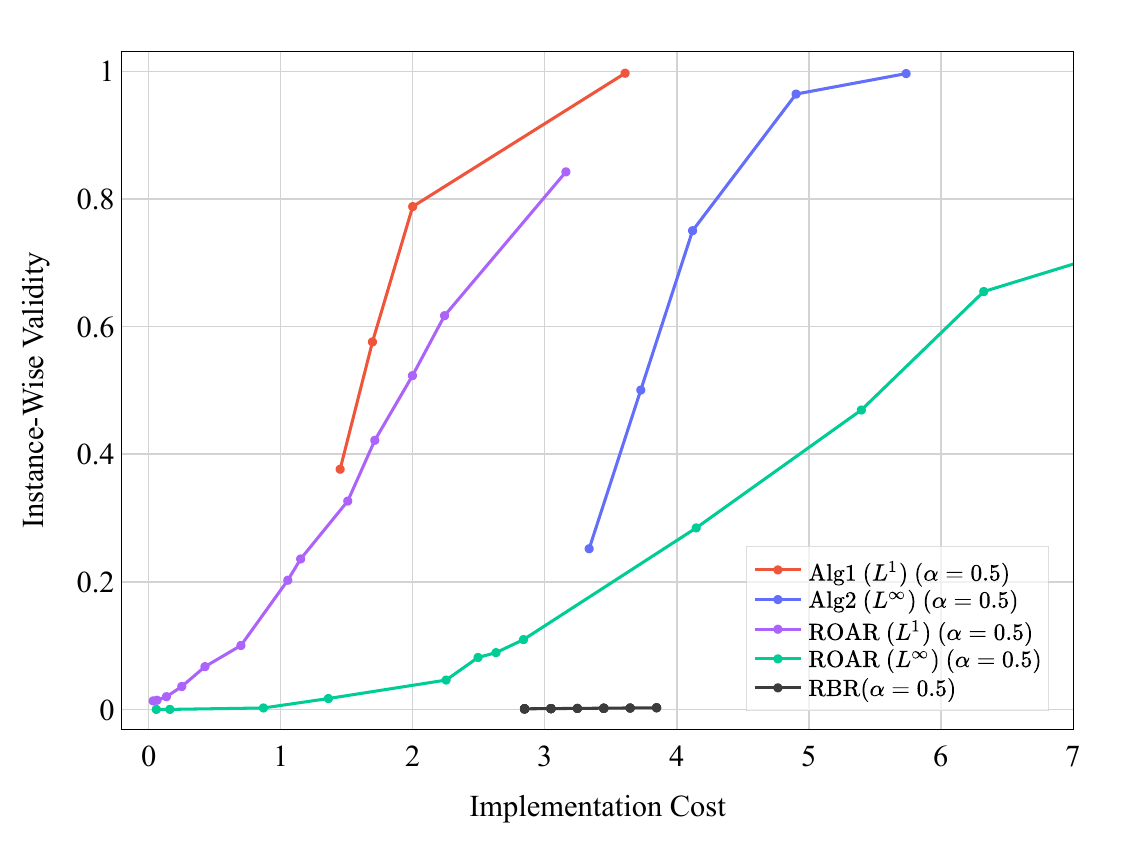}
        \caption{Instance-Wise Validity, Logistic Regression}
        \label{fig:cost_validity_tradeoff_lr_sba_instance_alpha_0.5_app}
    \end{subfigure}
    \begin{subfigure}[b]{0.45\textwidth}
        \centering
        \includegraphics[width=0.85\textwidth]{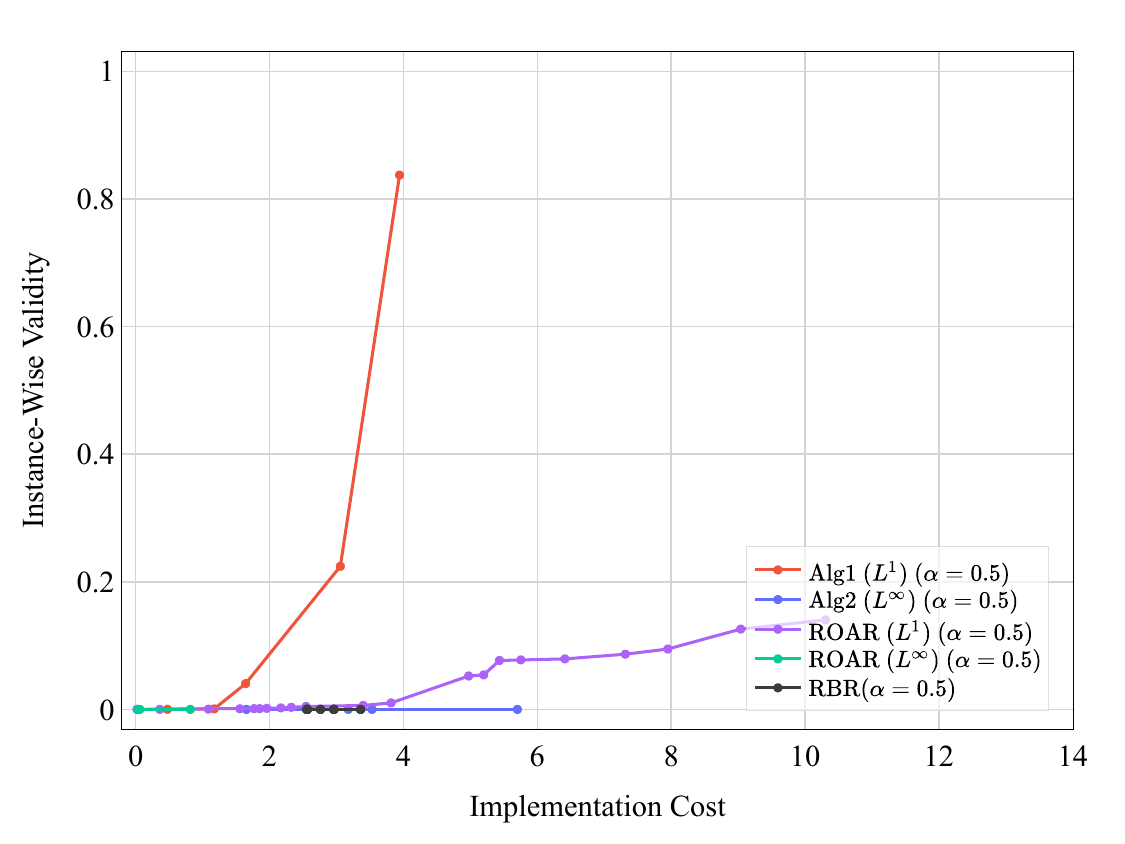}
        \caption{Instance-Wise Validity, Neural Network}
        \label{fig:cost_validity_tradeoff_nn_sba_instance_alpha_0.5_app}
    \end{subfigure}
    % POPULATION-WISE VALIDITY
    \begin{subfigure}[b]{0.45\textwidth}
        \centering
        \includegraphics[width=0.85\textwidth]{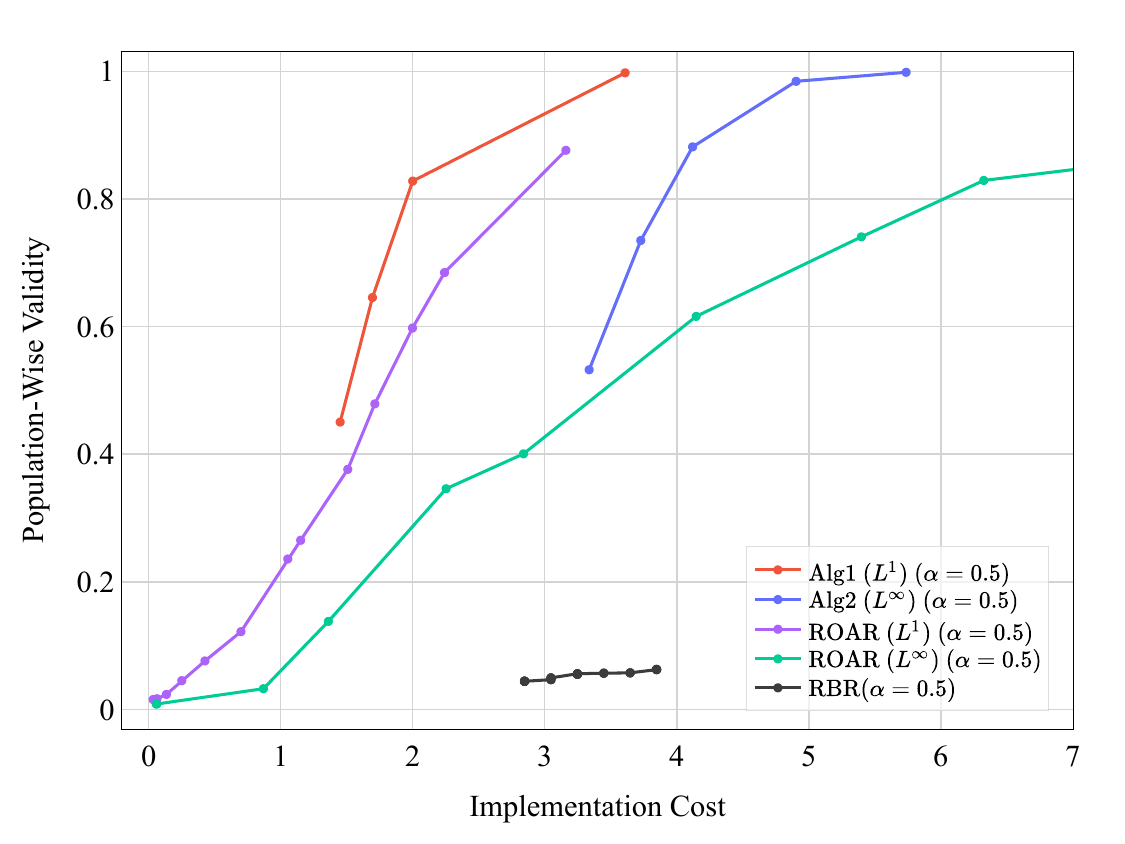}
        \caption{Population-Wise Validity, Logistic Regression}
        \label{fig:cost_validity_tradeoff_lr_sba_population_alpha_0.5_app}
    \end{subfigure}
    \begin{subfigure}[b]{0.45\textwidth}
        \centering
        \includegraphics[width=0.85\textwidth]{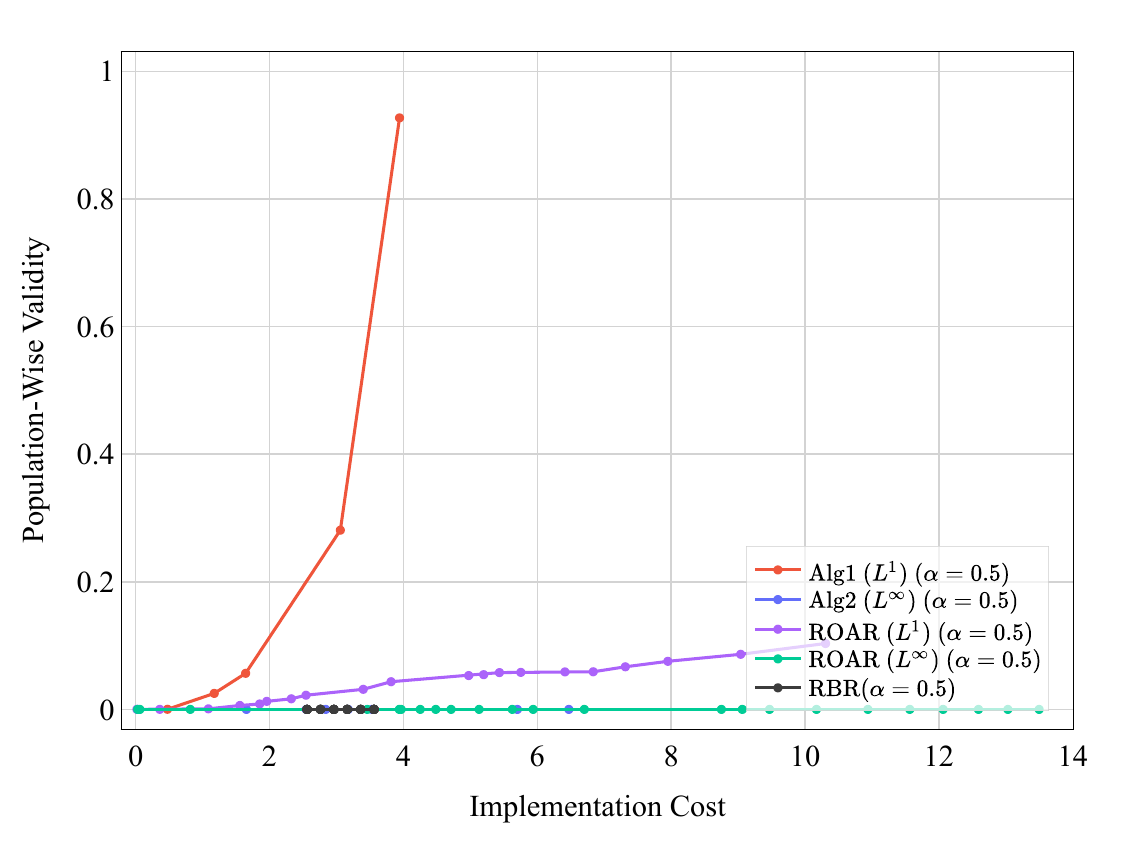}
        \caption{Population-Wise Validity, Neural Network}
        \label{fig:cost_validity_tradeoff_nn_sba_population_alpha_0.5_app}
    \end{subfigure}
    % CURRENT VALIDITY
    \begin{subfigure}[b]{0.45\textwidth}
        \centering
        \includegraphics[width=0.85\textwidth]{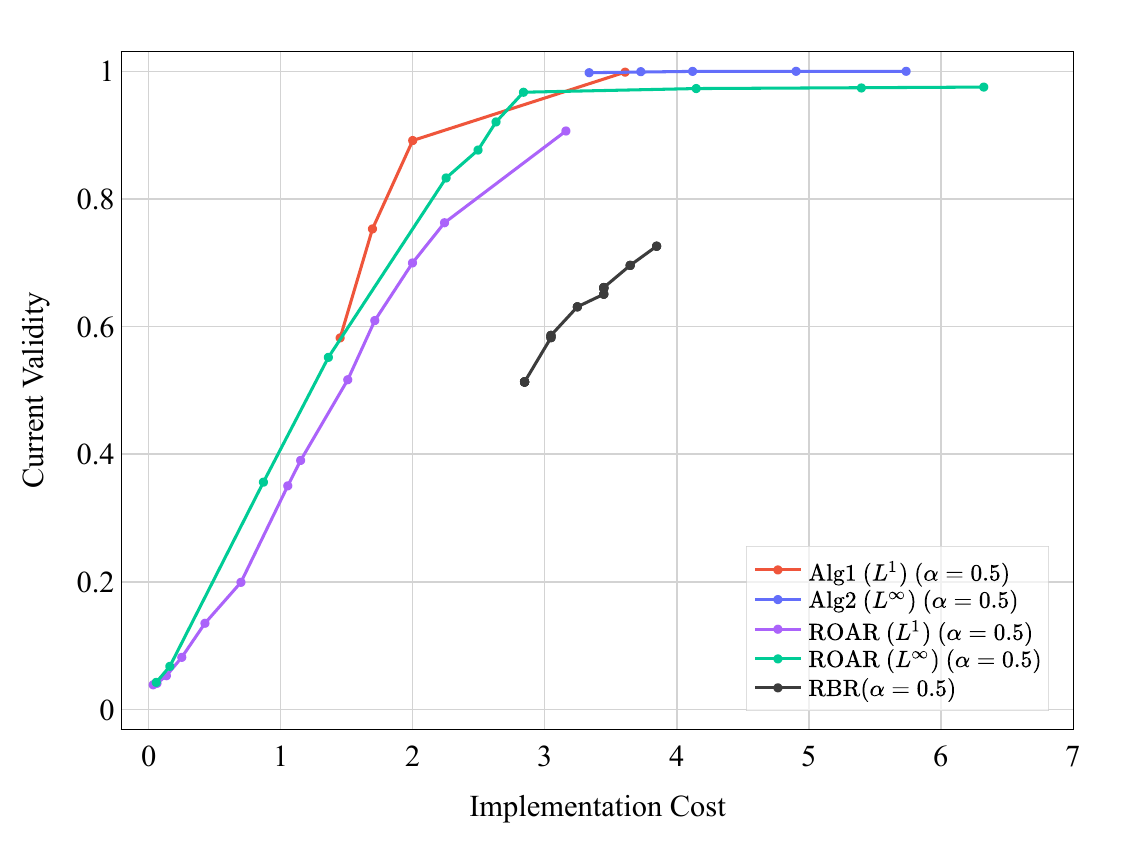}
        \caption{Current Validity, Logistic Regression}
        \label{fig:cost_validity_tradeoff_lr_sba_current_alpha_0.5_app}
    \end{subfigure}
    \begin{subfigure}[b]{0.45\textwidth}
        \centering
        \includegraphics[width=0.85\textwidth]{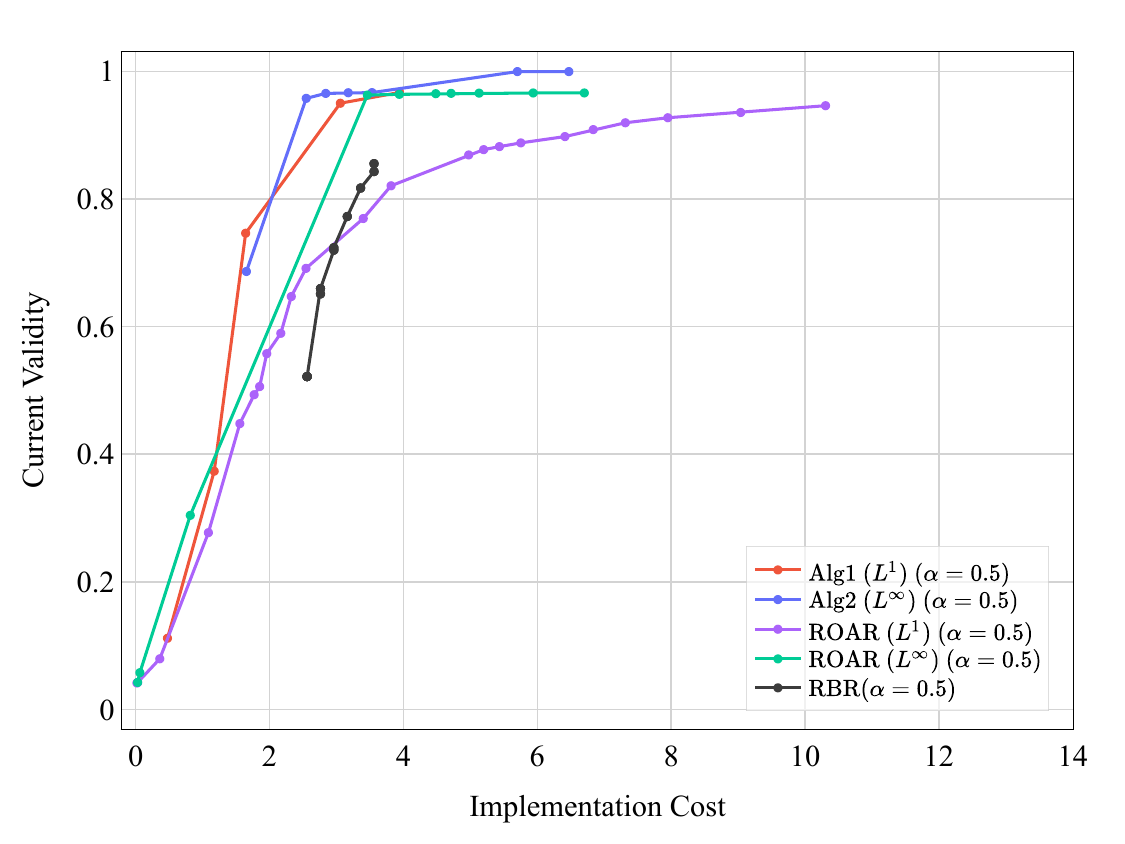}
        \caption{Current Validity, Neural Network}
        \label{fig:cost_validity_tradeoff_nn_sba_current_alpha_0.5_app}
    \end{subfigure}
    % FUTURE VALIDITY
    \begin{subfigure}[b]{0.45\textwidth}
        \centering
        \includegraphics[width=0.85\textwidth]{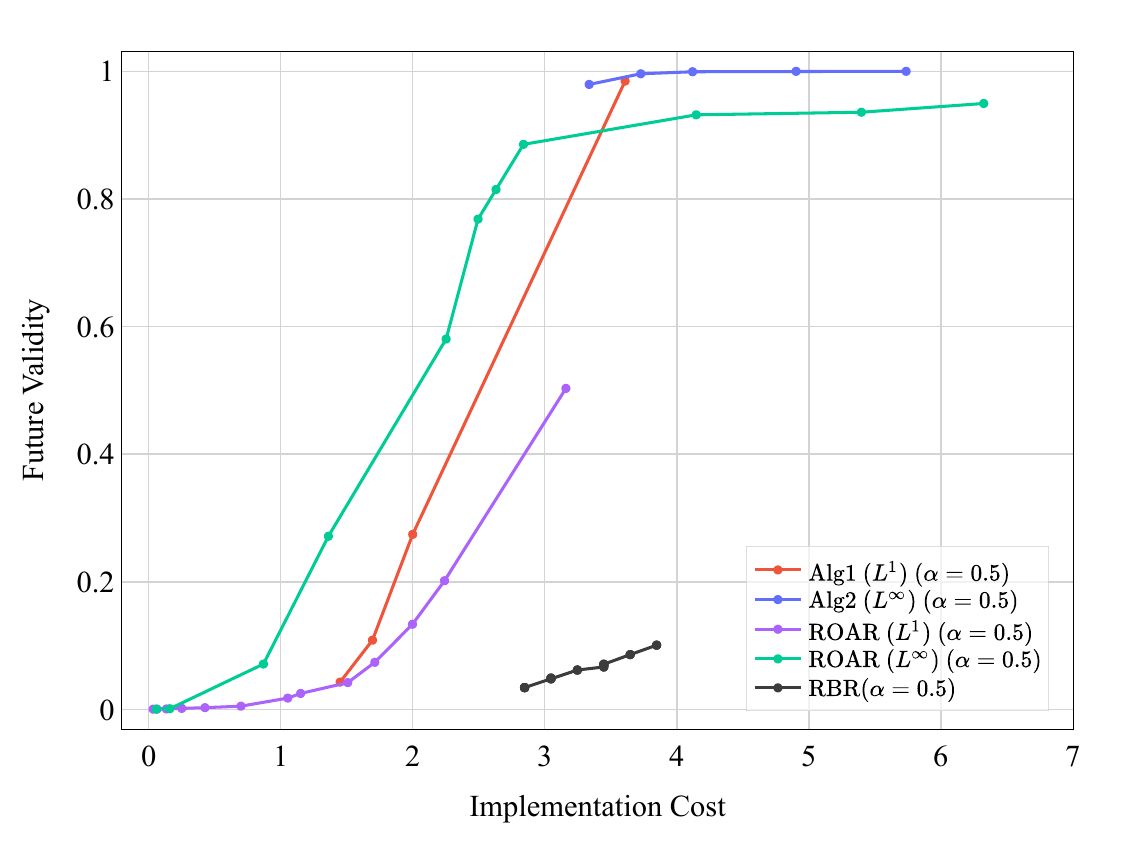}
        \caption{Future Validity, Logistic Regression}
        \label{fig:cost_validity_tradeoff_lr_sba_future_alpha_0.5_app}
    \end{subfigure}
    \begin{subfigure}[b]{0.45\textwidth}
        \centering
        \includegraphics[width=0.85\textwidth]{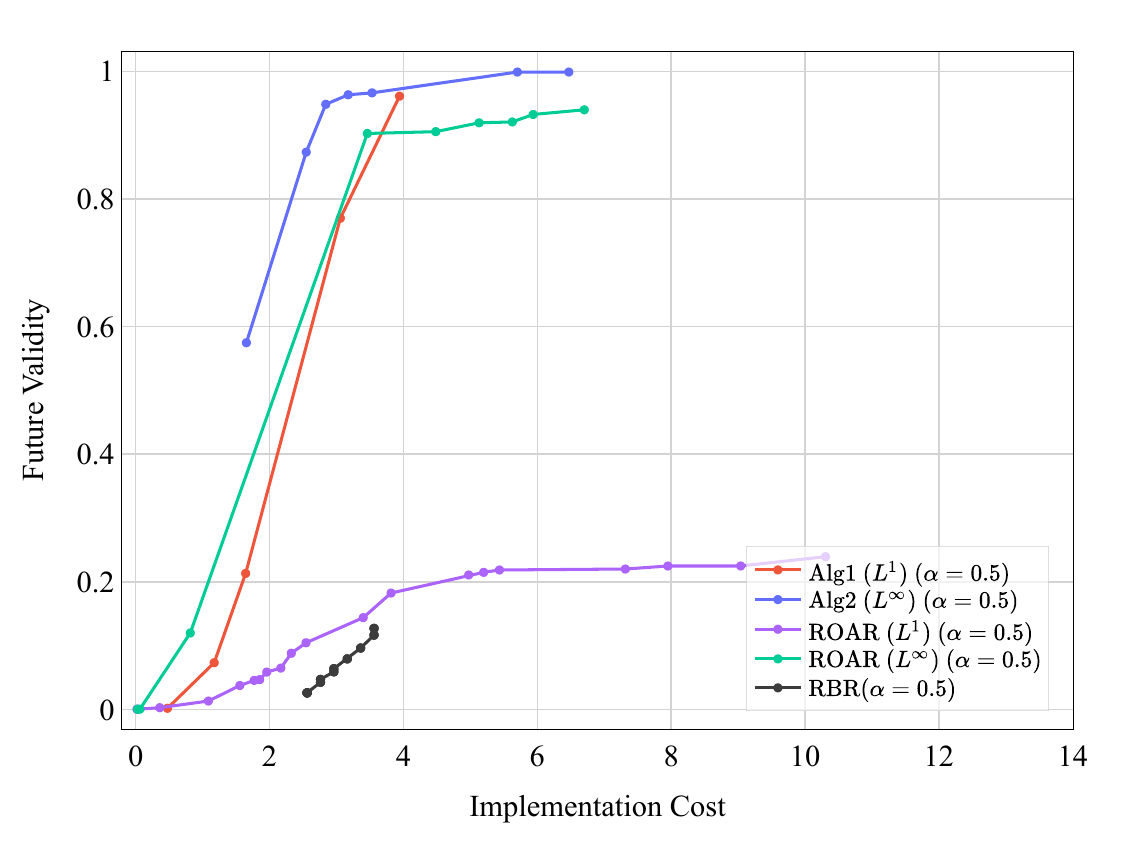}
        \caption{Future Validity, Neural Network}
        \label{fig:cost_validity_tradeoff_nn_sba_future_alpha_0.5_app}
    \end{subfigure}
    \caption{The frontier of the trade-off between validity and implementation cost on the Small Business Administration dataset with $\alpha=0.5$. The left and right columns correspond to logistic regression and neural network models. Each row corresponds to a different measure of validity. In each subfigure, curves show the trade-off for different algorithms. 
    \label{fig:cost_validity_tradeoff_sba_alpha_0.5_app}}
\end{figure*}

\subsection{Additional Details About Sparsity}
\label{sec:app-exp-sparse}
In this section, we provide the complete sparsity results for both datasets, models, and two distinct ways of measuring feature change. These results are presented in Figure~\ref{fig:sparsity_alpha_0.1_app} for $\alpha=0.1$ and Figure~\ref{fig:sparsity_alpha_0.5_app} for $\alpha=0.5$. 
\begin{figure*}[ht!]
    \centering
    \begin{subfigure}[b]{0.45\textwidth}
        \centering
        \includegraphics[width=0.85\textwidth]{Figures/sparsity-addictive-model_lr-dataset_german-alpha_0.1.pdf}
        \caption{German dataset, Logistic Regression, Additive}
        \label{fig:sparsity_add_lr_german_0.1_app}
    \end{subfigure}
    \begin{subfigure}[b]{0.45\textwidth}
        \centering
        \includegraphics[width=0.85\textwidth]{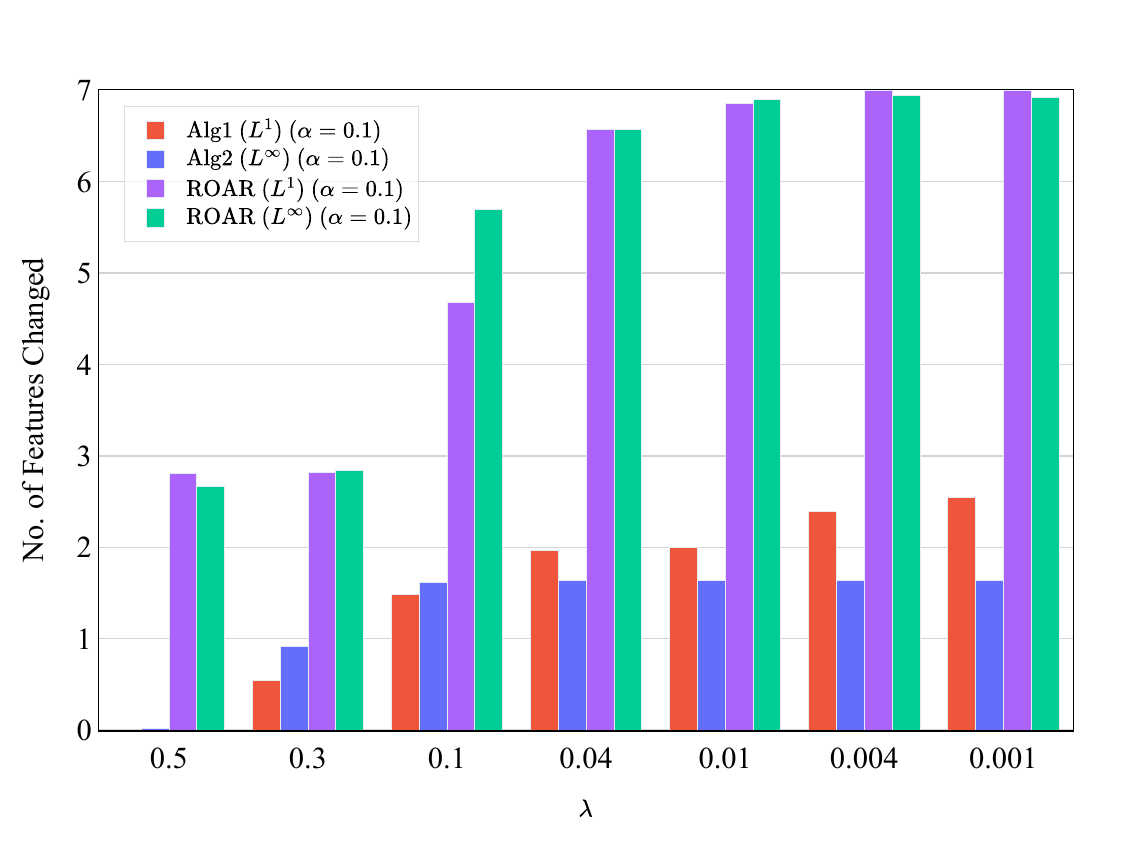}
        \caption{German dataset, Logistic Regression, Multiplicative}
        \label{fig:sparsity_multi_lr_german_0.1_app}
    \end{subfigure}
    \begin{subfigure}[b]{0.45\textwidth}
        \centering
        \includegraphics[width=0.85\textwidth]{Figures/sparsity-addictive-model_lr-dataset_sba-alpha_0.1.pdf}
        \caption{SBA Dataset, Logistic Regression, Additive}
        \label{fig:sparsity_add_lr_sba_0.1_app}
    \end{subfigure}
    \begin{subfigure}[b]{0.45\textwidth}
        \centering
        \includegraphics[width=0.85\textwidth]{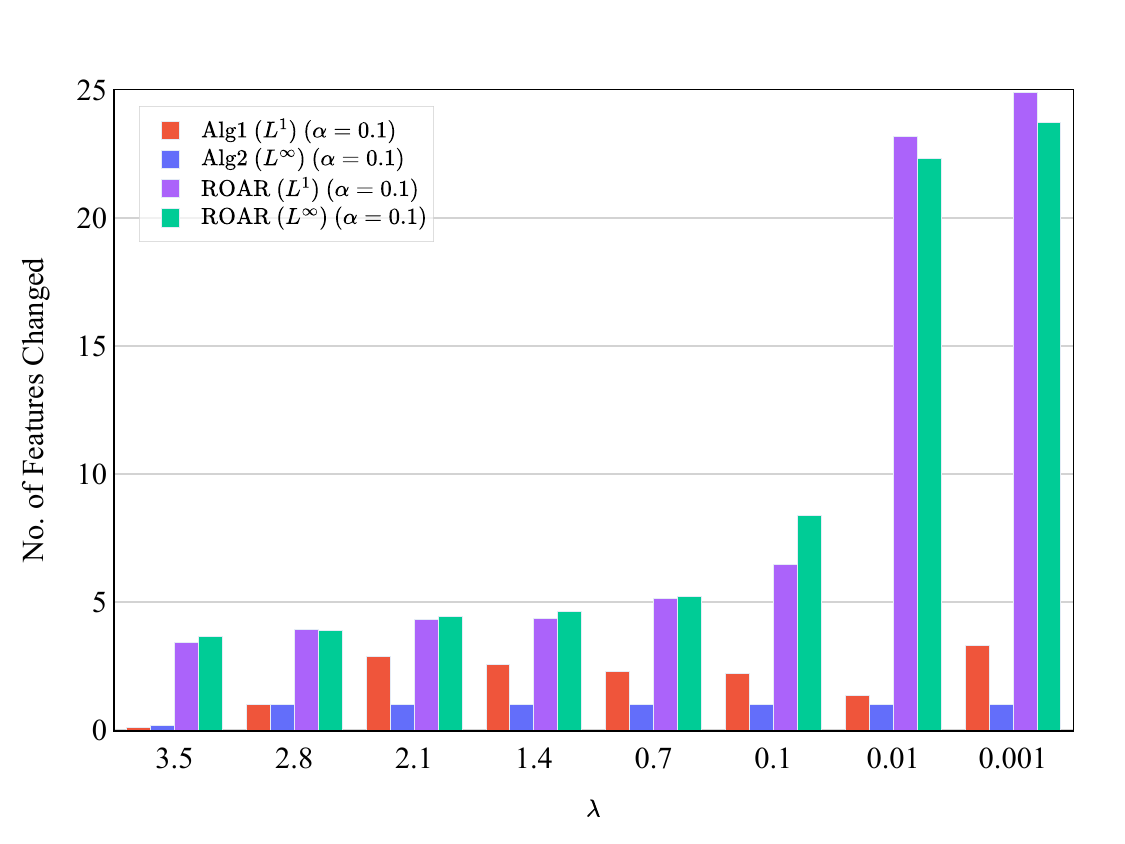}
        \caption{SBA Dataset, Logistic Regression, Multiplicative}
        \label{fig:sparsity_multi_lr_sba_0.1_app}
    \end{subfigure}
    \begin{subfigure}[b]{0.45\textwidth}
        \centering
        \includegraphics[width=0.85\textwidth]{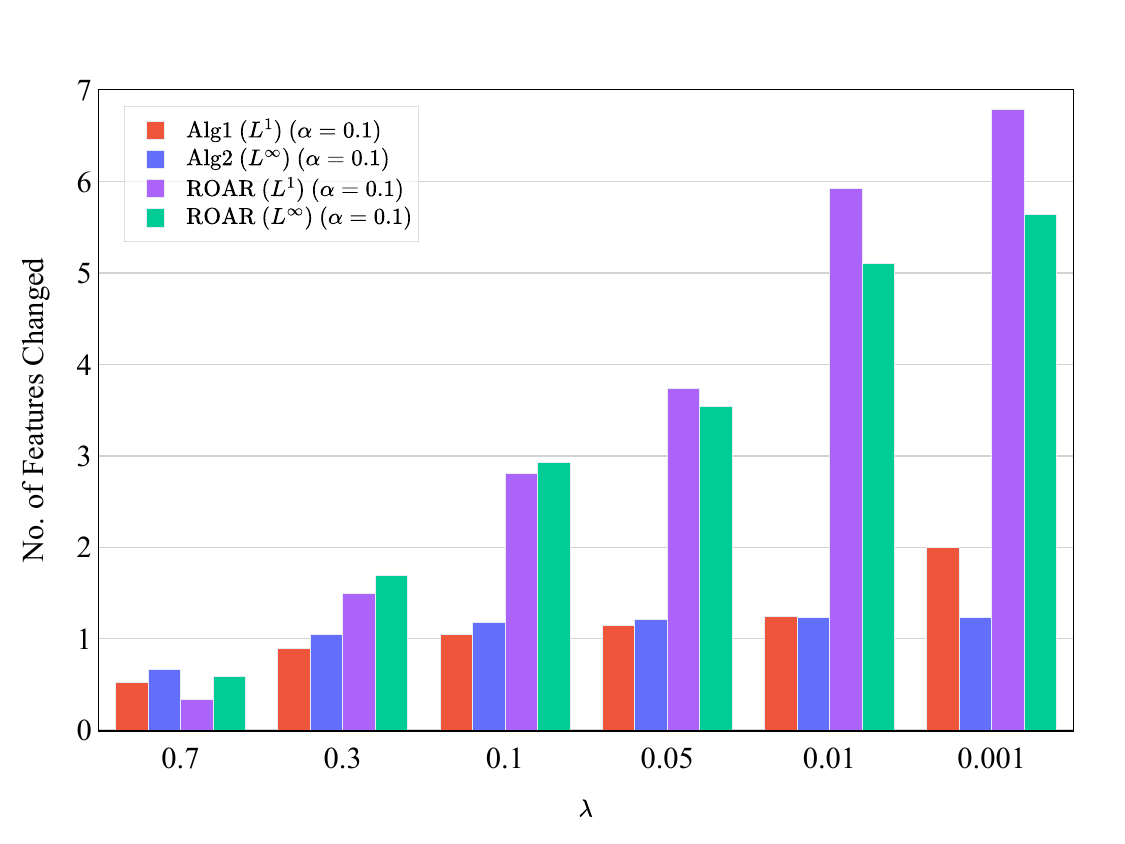}
        \caption{German dataset, Neural Network, Additive}
        \label{fig:sparsity_add_nn_german_0.1_app}
    \end{subfigure}
    \begin{subfigure}[b]{0.45\textwidth}
        \centering
        \includegraphics[width=0.85\textwidth]{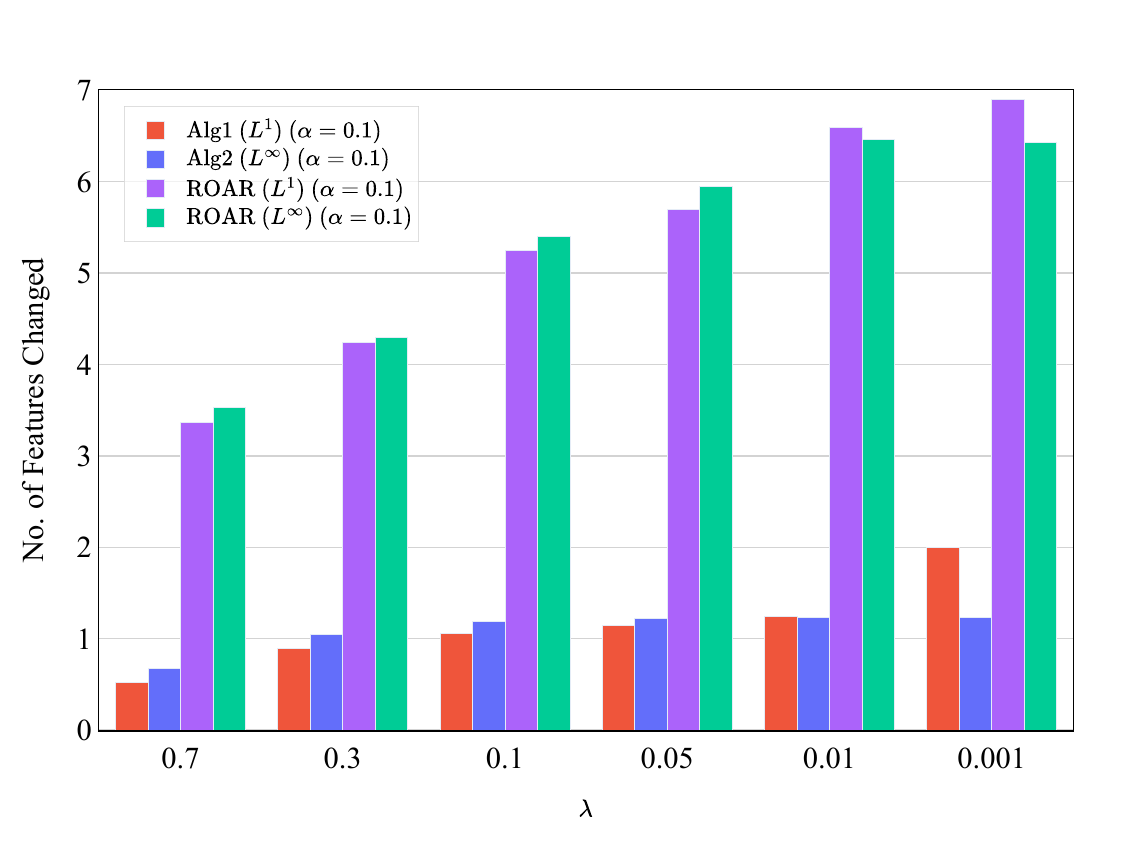}
        \caption{German dataset, Neural Network, Multiplicative}
        \label{fig:sparsity_multi_nn_german_0.1_app}
    \end{subfigure}
    \begin{subfigure}[b]{0.45\textwidth}
        \centering
        \includegraphics[width=0.85\textwidth]{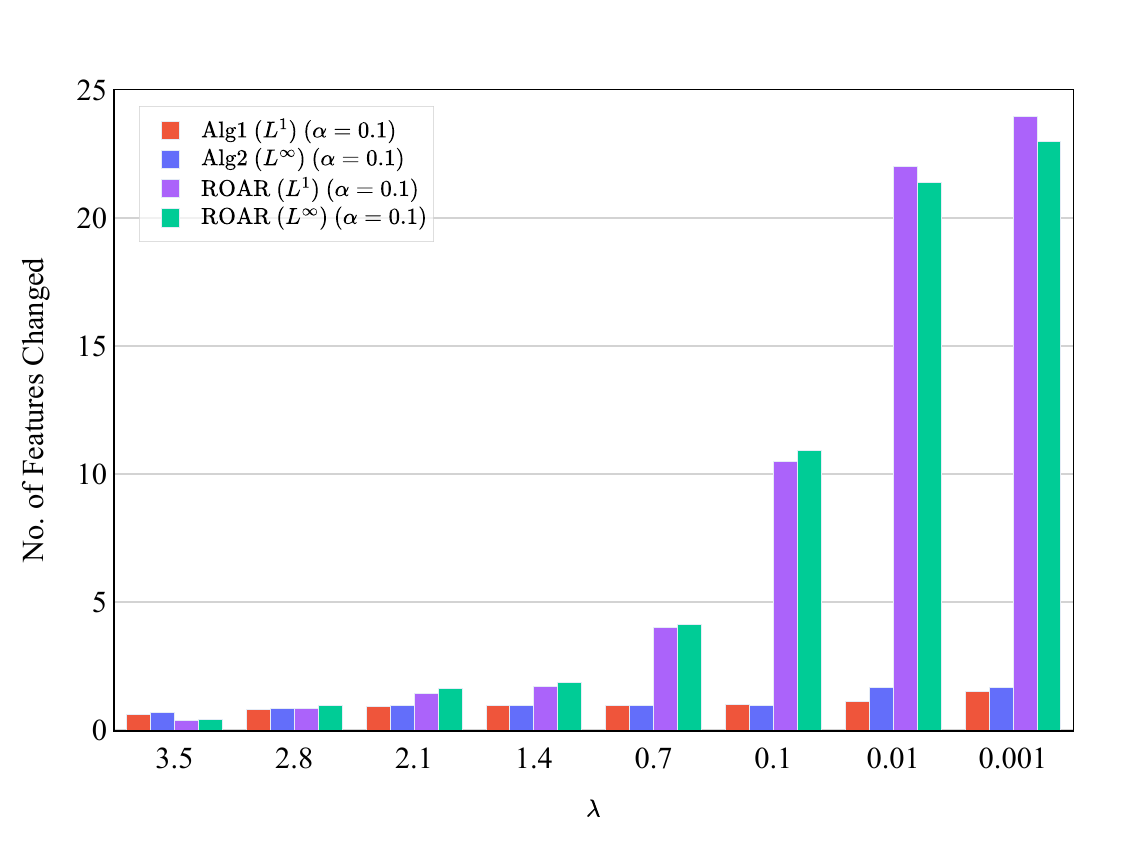}
        \caption{SBA Dataset, Neural Network, Additive}
        \label{fig:sparsity_add_nn_sba_0.1_app}
    \end{subfigure}
    \begin{subfigure}[b]{0.45\textwidth}
        \centering
        \includegraphics[width=0.85\textwidth]{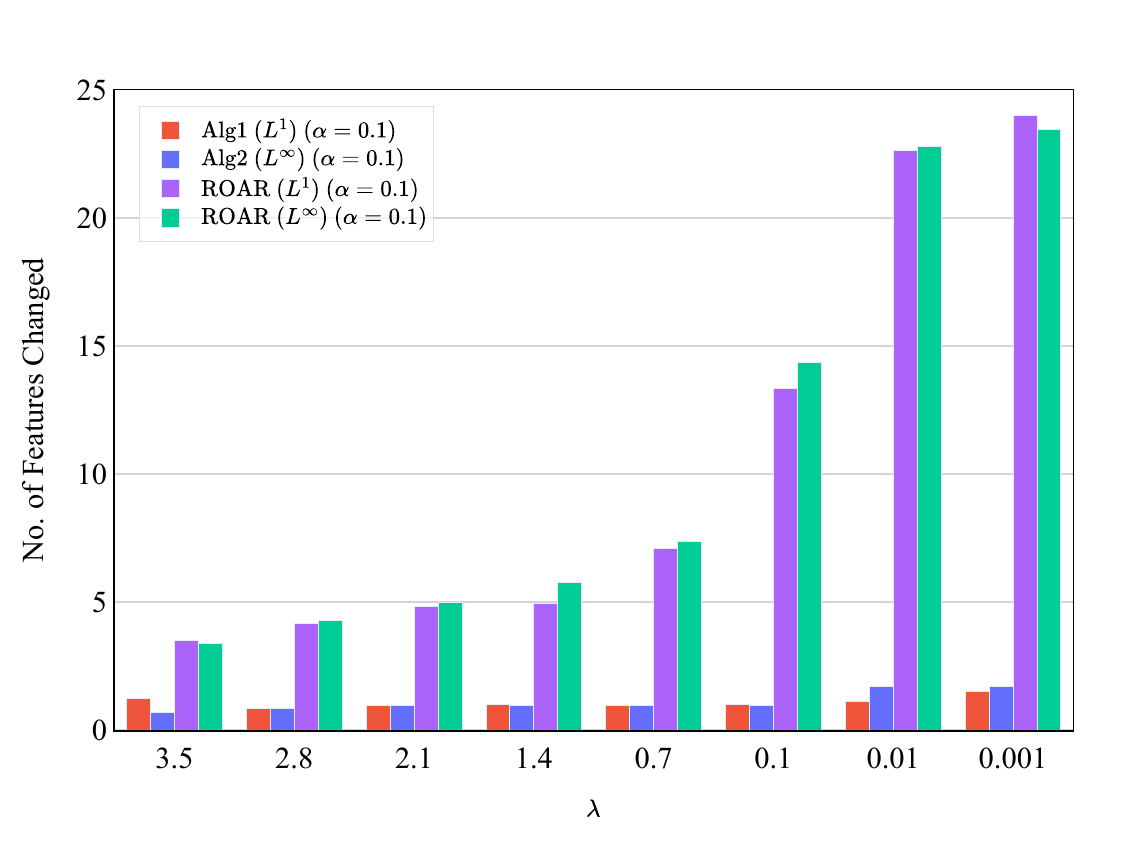}
        \caption{SBA Dataset, Neural Network, Multiplicative}
        \label{fig:sparsity_multi_nn_sba_0.1_app}
    \end{subfigure}
    \caption{Number of changed features for the German and Small Business Datasets with $\alpha=0.1$. Left and right columns correspond to measuring feature change in an additive and multiplicative manner. Each subfigure corresponds to a dataset and model combination. In each subfigure, bars depict the number of changed features for each of the algorithms at different $\lambda$ values. \label{fig:sparsity_alpha_0.1_app}}
\end{figure*}

\begin{figure*}[ht!]
    \centering
    \begin{subfigure}[b]{0.45\textwidth}
        \centering
        \includegraphics[width=0.85\textwidth]{Figures/sparsity-addictive-model_lr-dataset_german-alpha_0.5.pdf}
        \caption{German dataset, Logistic Regression, Additive}
        \label{fig:sparsity_add_lr_german_0.5_app}
    \end{subfigure}
    \begin{subfigure}[b]{0.45\textwidth}
        \centering
        \includegraphics[width=0.85\textwidth]{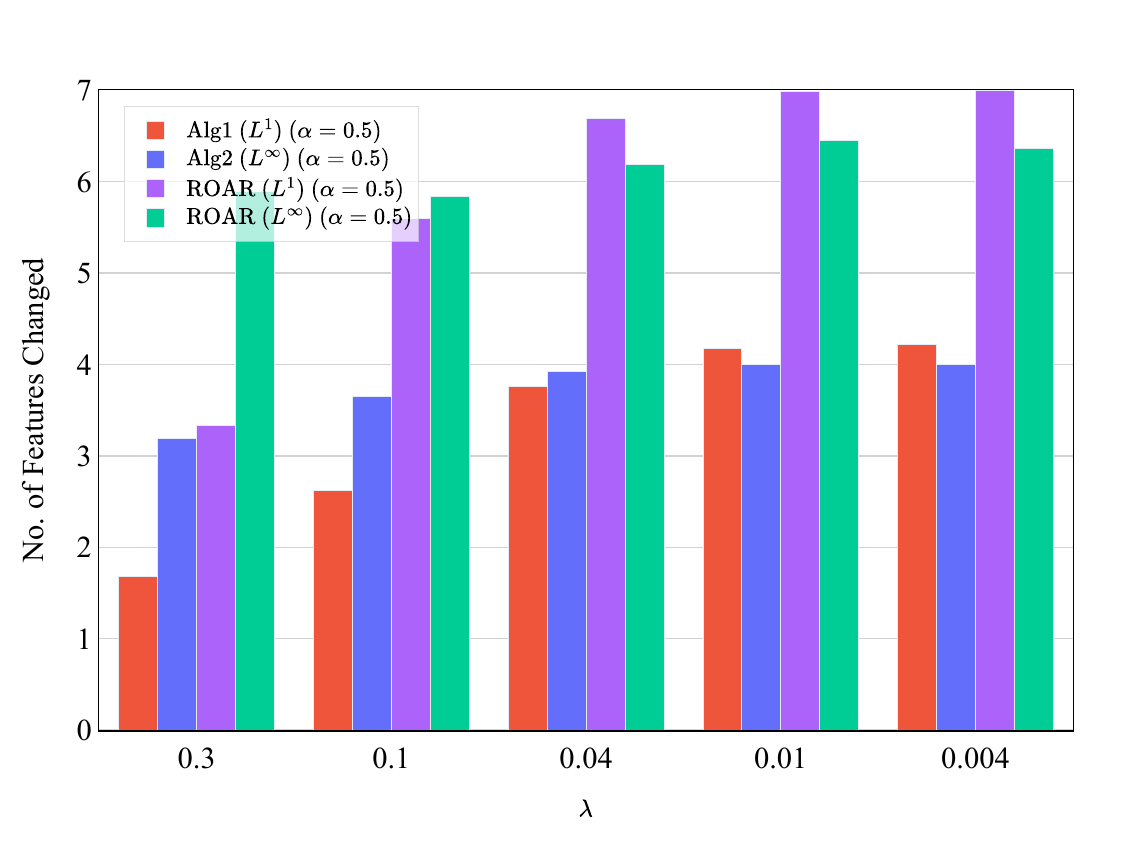}
        \caption{German dataset, Logistic Regression, Multiplicative}
        \label{fig:sparsity_multi_lr_german_0.5_app}
    \end{subfigure}
    \begin{subfigure}[b]{0.45\textwidth}
        \centering
        \includegraphics[width=0.85\textwidth]{Figures/sparsity-addictive-model_lr-dataset_sba-alpha_0.5.pdf}
        \caption{SBA Dataset, Logistic Regression, Additive}
        \label{fig:sparsity_add_lr_sba_0.5_app}
    \end{subfigure}
    \begin{subfigure}[b]{0.45\textwidth}
        \centering
        \includegraphics[width=0.85\textwidth]{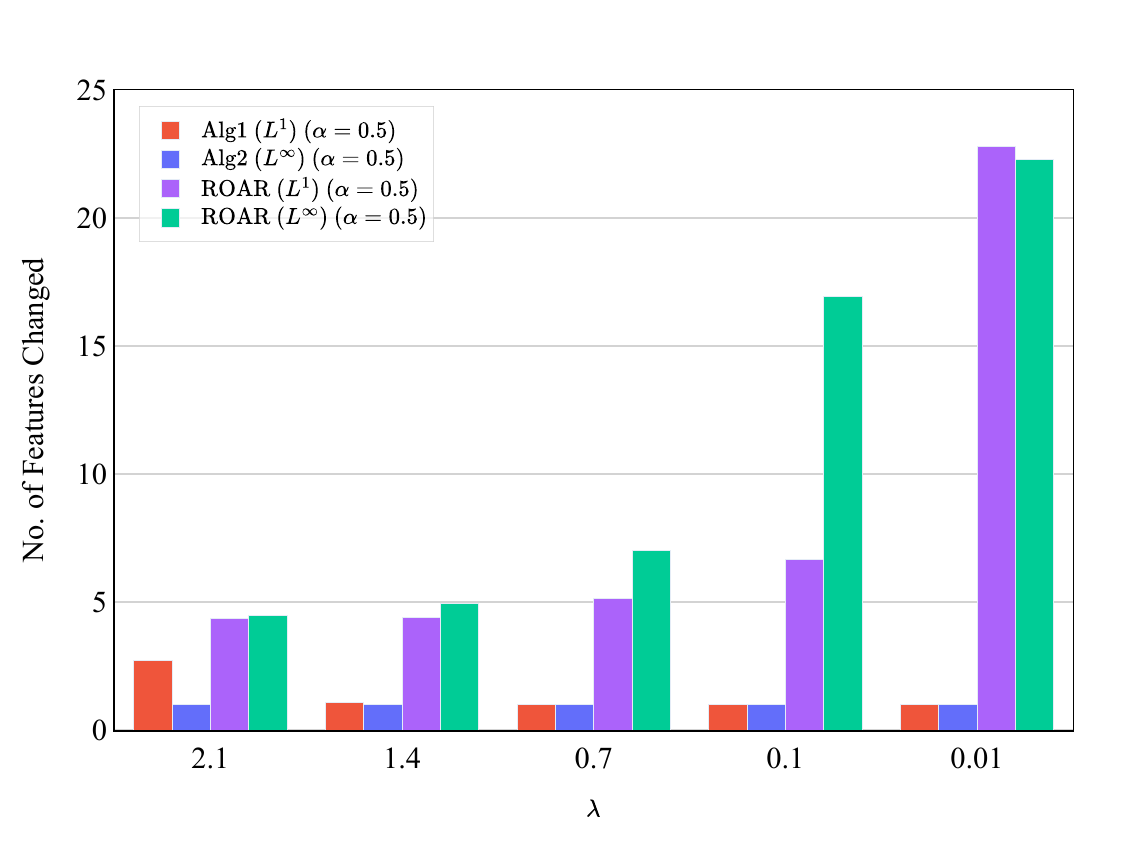}
        \caption{SBA Dataset, Logistic Regression, Multiplicative}
        \label{fig:sparsity_multi_lr_sba_0.5_app}
    \end{subfigure}
    \begin{subfigure}[b]{0.45\textwidth}
        \centering
        \includegraphics[width=0.85\textwidth]{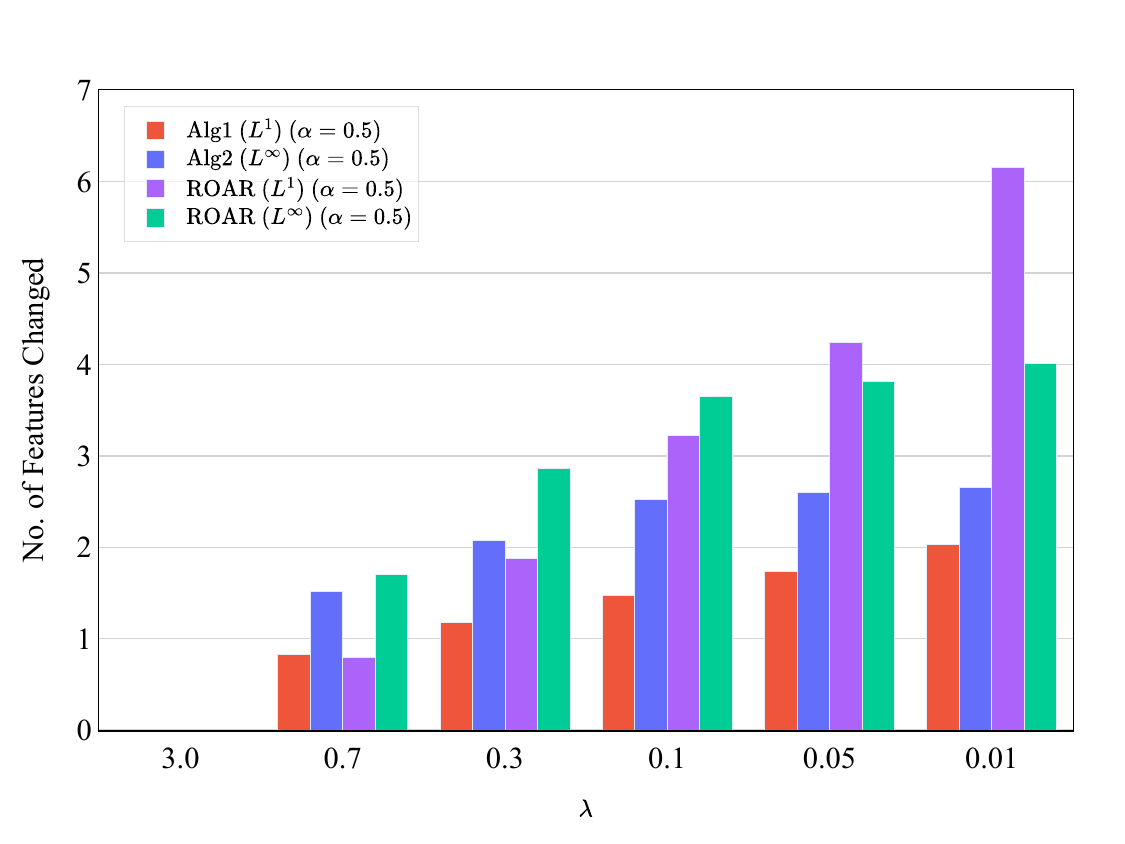}
        \caption{German dataset, Neural Network, Additive}
        \label{fig:sparsity_add_nn_german_0.5_app}
    \end{subfigure}
    \begin{subfigure}[b]{0.45\textwidth}
        \centering
        \includegraphics[width=0.85\textwidth]{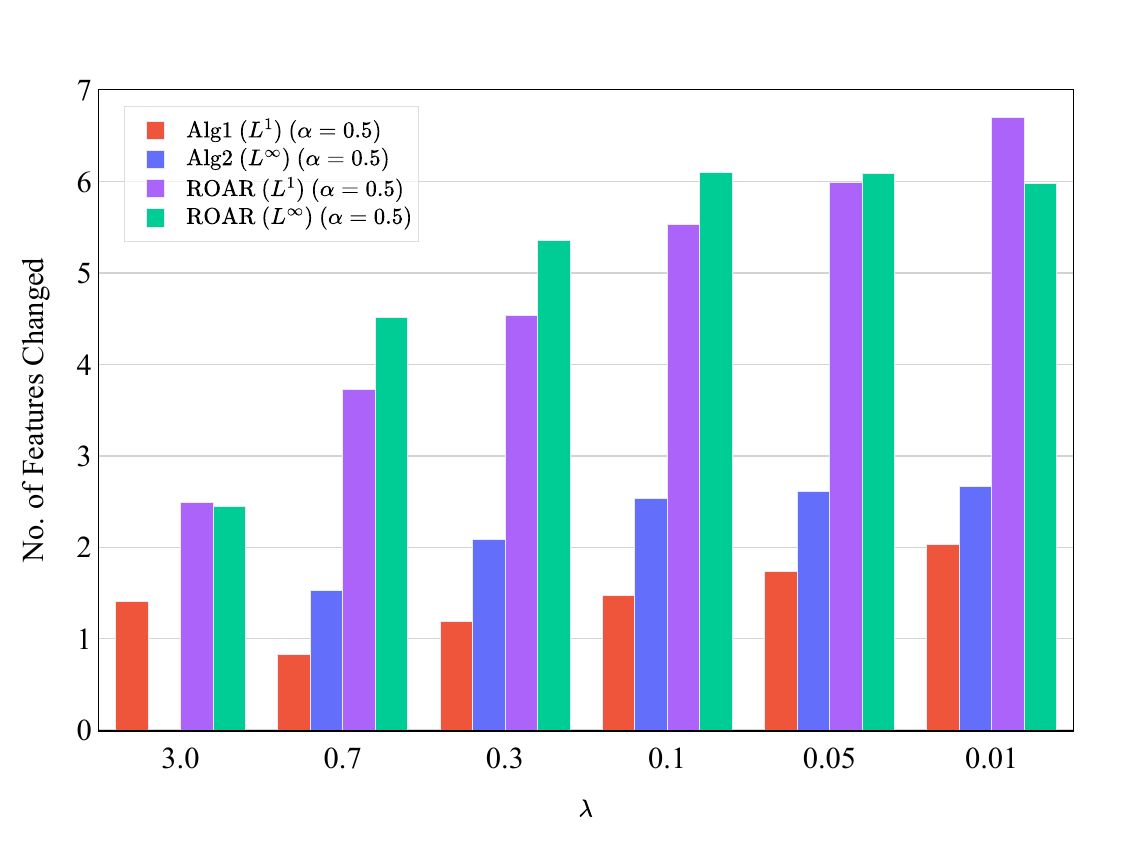}
        \caption{German dataset, Neural Network, Multiplicative}
        \label{fig:sparsity_multi_nn_german_0.5_app}
    \end{subfigure}
    \begin{subfigure}[b]{0.45\textwidth}
        \centering
        \includegraphics[width=0.85\textwidth]{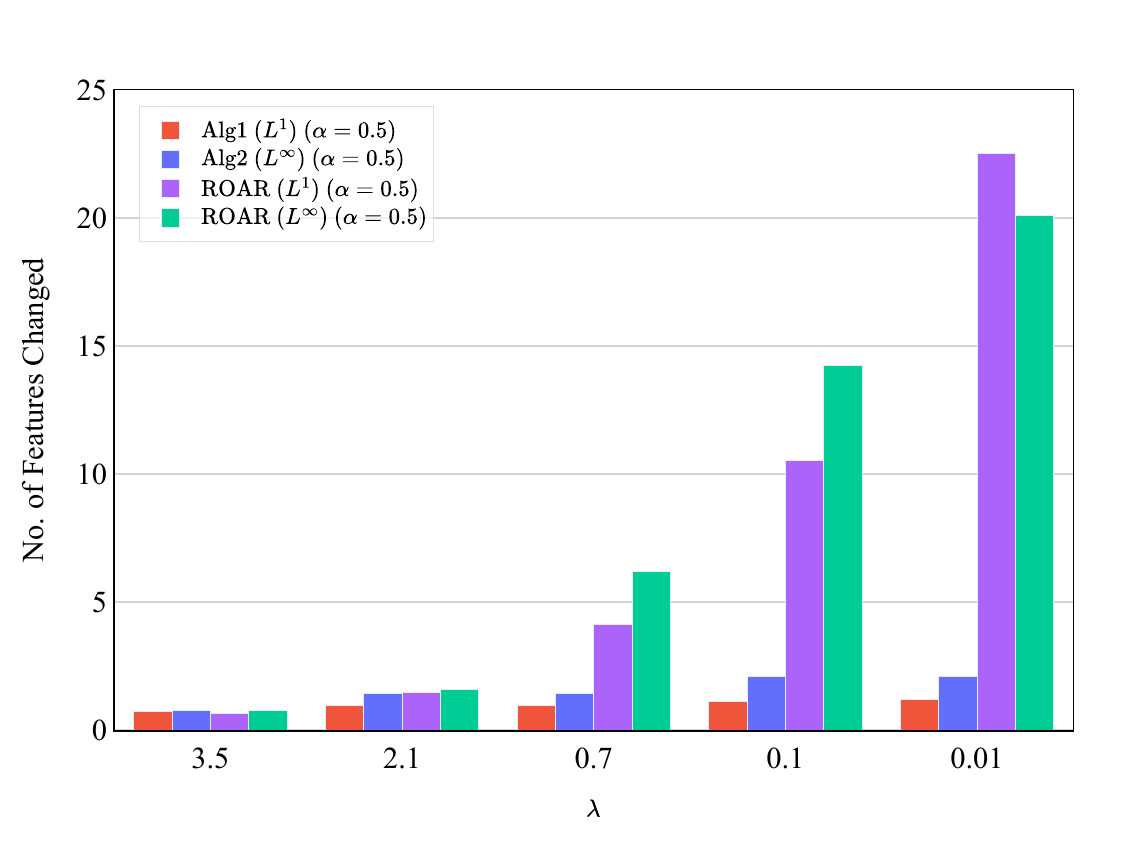}
        \caption{SBA Dataset, Neural Network, Additive}
        \label{fig:sparsity_add_nn_sba_0.5_app}
    \end{subfigure}
    \begin{subfigure}[b]{0.45\textwidth}
        \centering
        \includegraphics[width=0.85\textwidth]{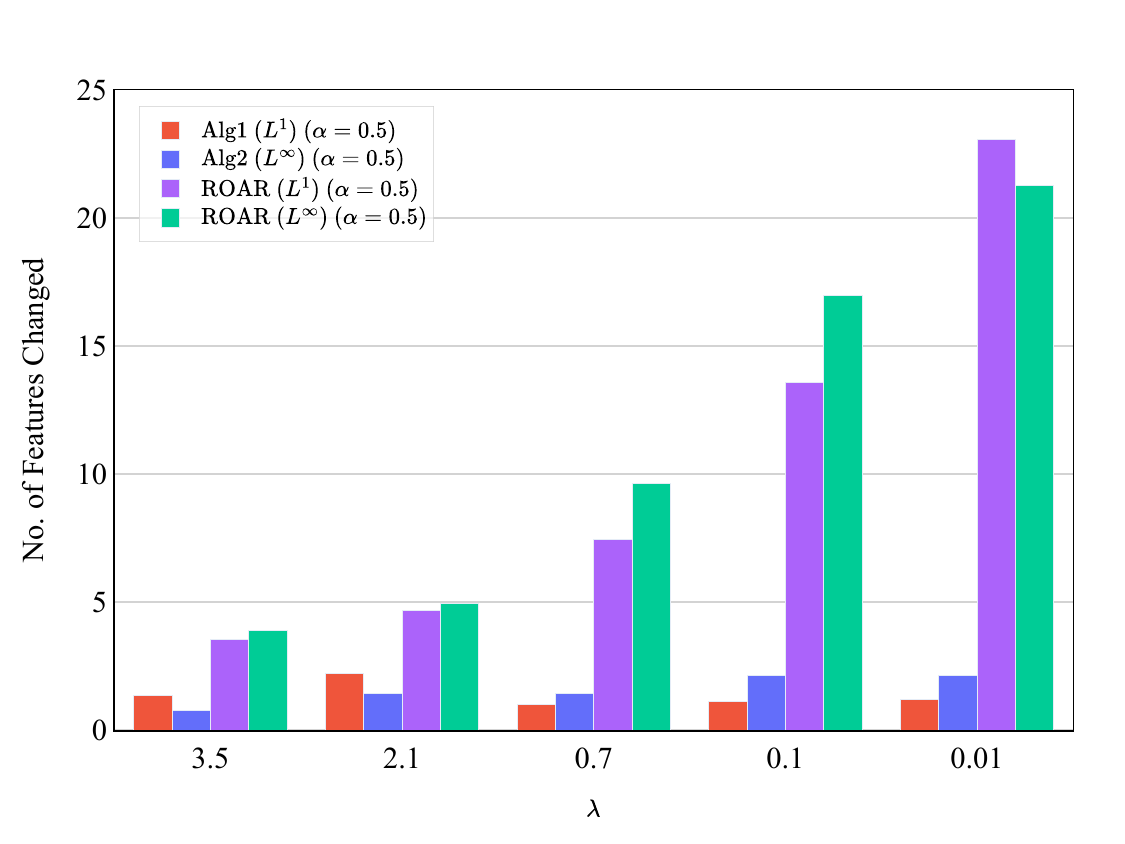}
        \caption{SBA Dataset, Neural Network, Multiplicative}
        \label{fig:sparsity_multi_nn_sba_0.5_app}
    \end{subfigure}
    \caption{Number of changed features for the German and Small Business Datasets with $\alpha=0.5$. Left and right columns correspond to measuring feature change in an additive and multiplicative manner. Each subfigure corresponds to a dataset and model combination. In each subfigure, bars depict the number of changed features for each of the algorithms at different $\lambda$ values. \label{fig:sparsity_alpha_0.5_app}}
\end{figure*}

\subsection{Additional Details About Feasibility}
\label{sec:app-exp-feasibility}
In this section, we provide the complete feasibility results for both datasets, models, and hardmax post-processing. These results are presented in Figure~\ref{fig:feasibility_alpha_0.1_app} for $\alpha=0.1$ and Figure~\ref{fig:feasibility_alpha_0.5_app} for $\alpha=0.5$. 

\begin{figure*}[ht!]
    \centering
    \begin{subfigure}[b]{0.45\textwidth}
        \centering
        \includegraphics[width=\textwidth]{Figures/cost_validity-instance_wise-lr-german-alpha_0.1-frontiers_feasible.pdf}
        \caption{German Credit dataset, Logistic Regression}
        \label{fig:feasibility_lr_german_alpha_0.1_app}
    \end{subfigure}
    \begin{subfigure}[b]{0.45\textwidth}
        \centering
        \includegraphics[width=\textwidth]{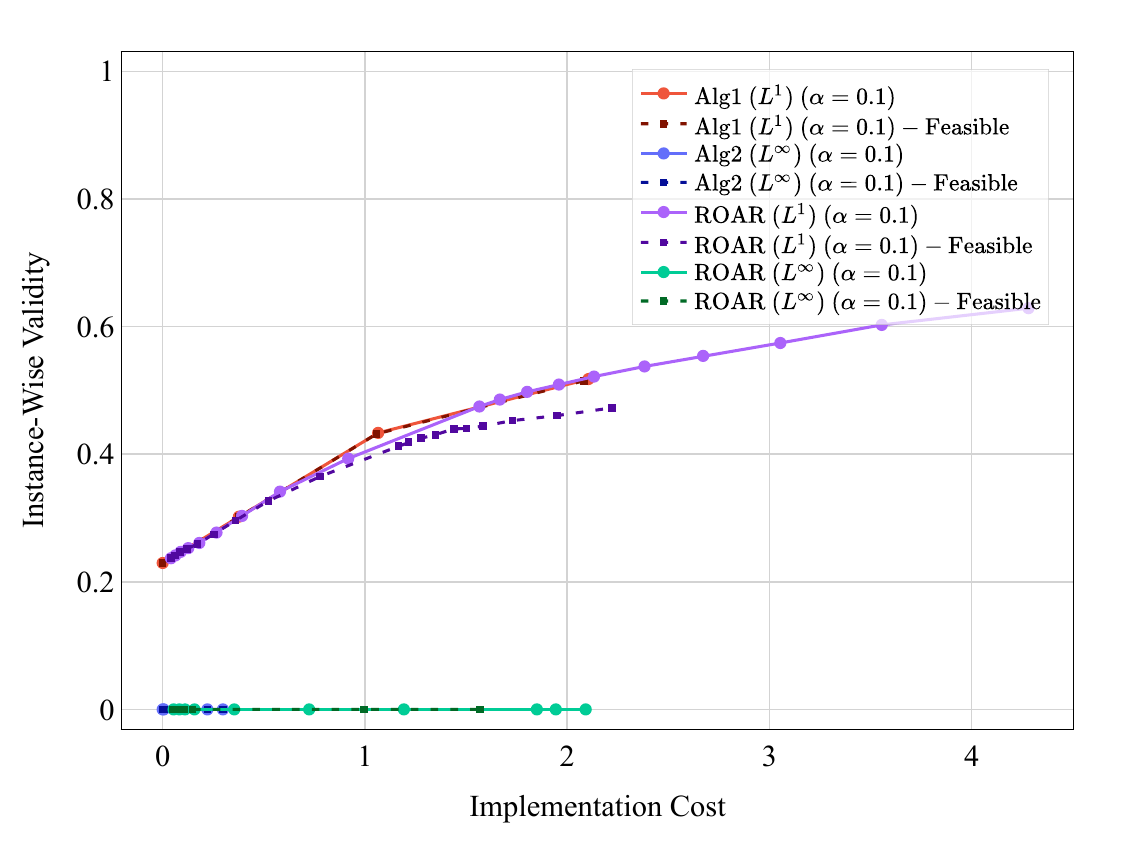}
        \caption{German Credit dataset, Neural Network}
        \label{fig:feasibility_nn_german_alpha_0.1_app}
    \end{subfigure}
    \begin{subfigure}[b]{0.45\textwidth}
        \centering
        \includegraphics[width=\textwidth]{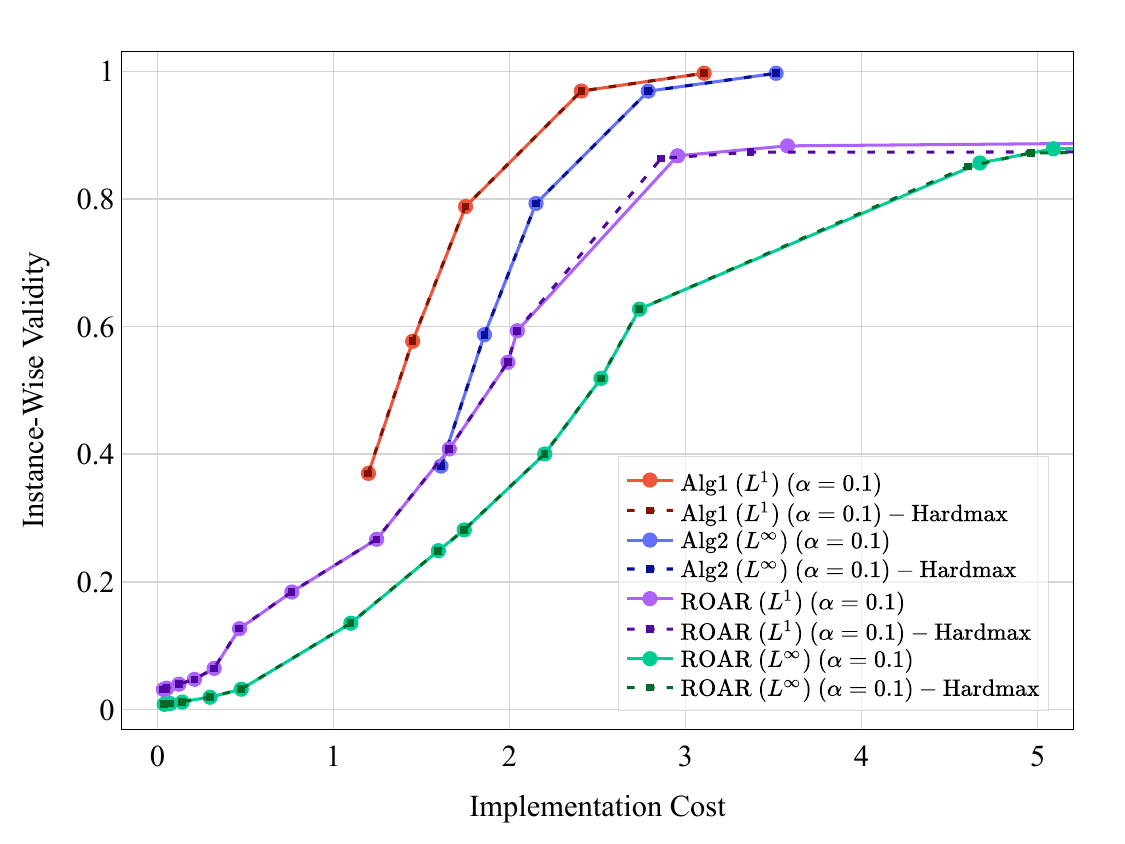}
        \caption{SBA Dataset, Logistic Regression}
        \label{fig:feasibility_lr_sba_alpha_0.1_app}
    \end{subfigure}
    \begin{subfigure}[b]{0.45\textwidth}
        \centering
        \includegraphics[width=\textwidth]{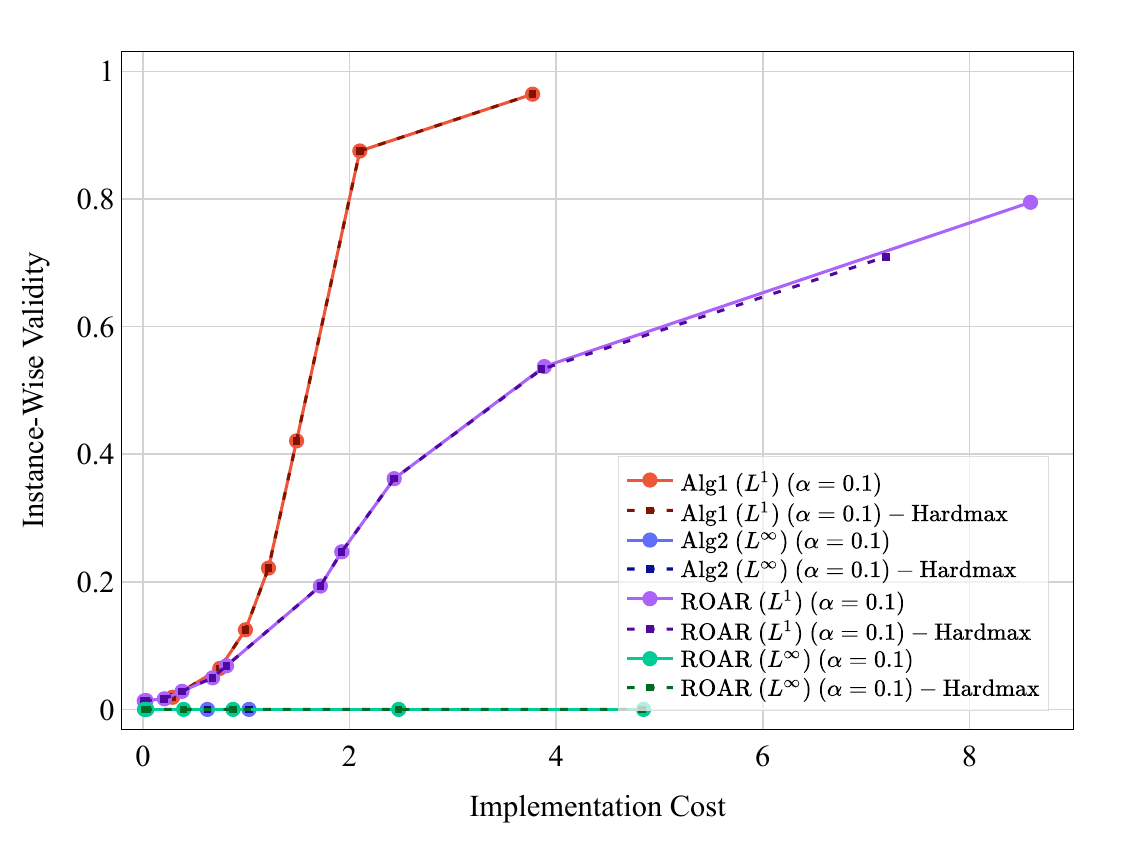}
        \caption{SBA Dataset, Neural Network}
        \label{fig:feasibility_nn_sba_alpha_0.1_app}
    \end{subfigure}
    \caption{The frontier of the trade-off between validity and implementation cost on the Small Business Administration dataset after post-processing with $\alpha=0.1$. The left and right columns correspond to logistic regression and neural network models. Each row corresponds to a different dataset. In each subfigure, curves show the trade-off for different algorithms. For each algorithm, solid and dashed lines depict the performance before and after hardmax post-processing is applied.
    \label{fig:feasibility_alpha_0.1_app}}
\end{figure*}

\begin{figure*}[ht!]
    \centering
    \begin{subfigure}[b]{0.45\textwidth}
        \centering
        \includegraphics[width=\textwidth]{Figures/cost_validity-instance_wise-lr-german-alpha_0.5-frontiers_feasible.pdf}
        \caption{German Credit dataset, Logistic Regression}
        \label{fig:feasibility_lr_german_alpha_0.5_app}
    \end{subfigure}
    \begin{subfigure}[b]{0.45\textwidth}
        \centering
        \includegraphics[width=\textwidth]{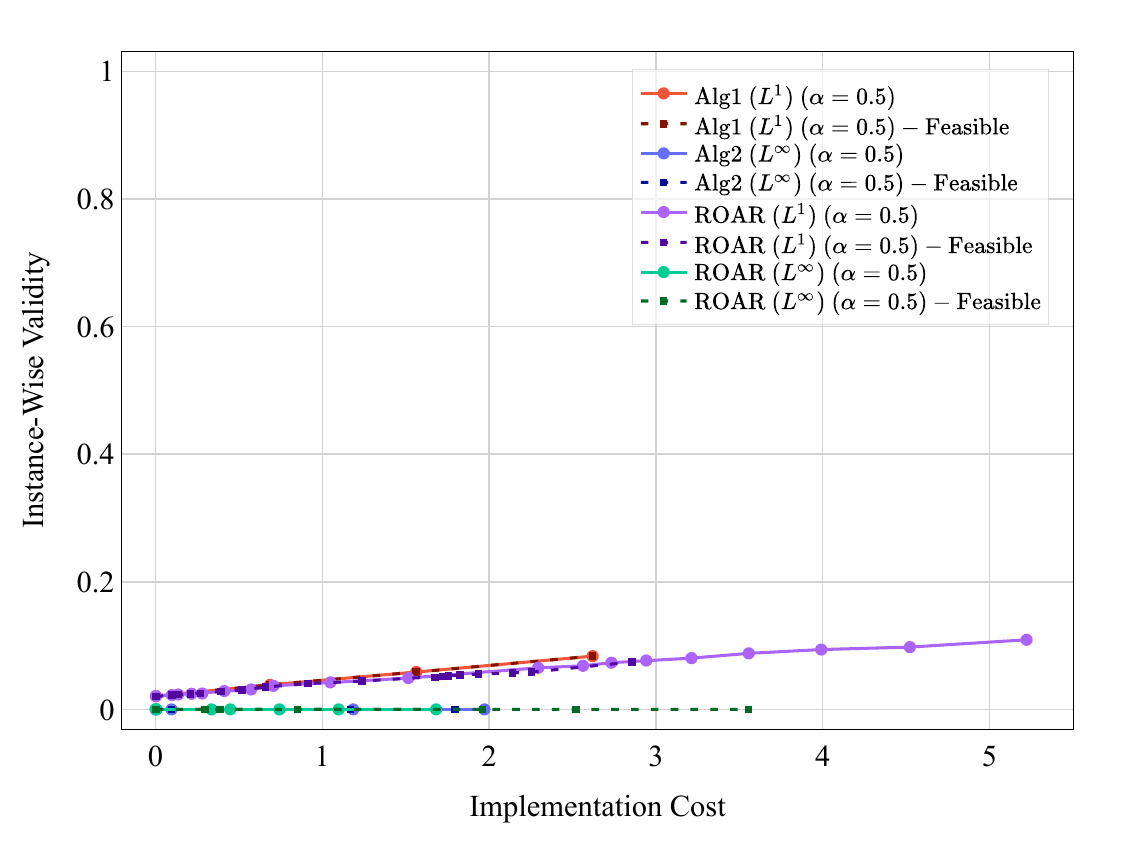}
        \caption{German Credit dataset, Neural Network}
        \label{fig:feasibility_nn_german_alpha_0.5_app}
    \end{subfigure}
    \begin{subfigure}[b]{0.45\textwidth}
        \centering
        \includegraphics[width=\textwidth]{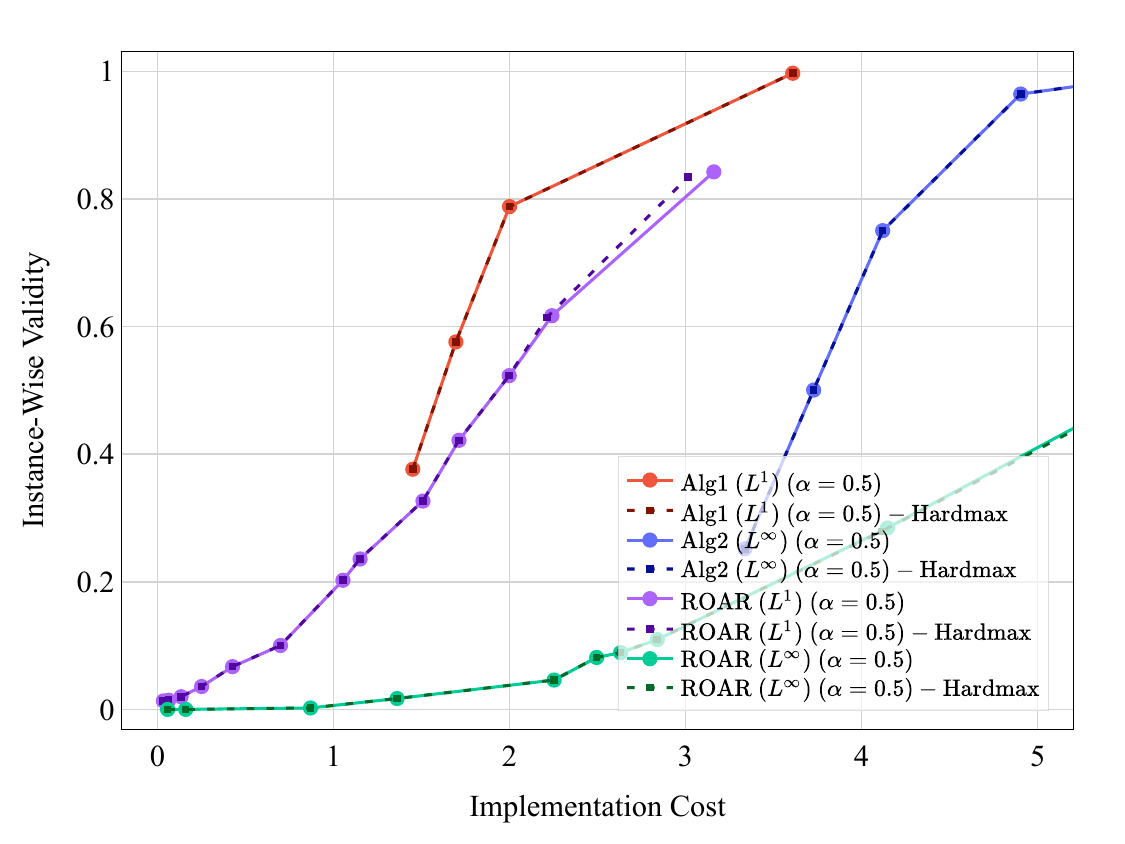}
        \caption{SBA Dataset, Logistic Regression}
        \label{fig:feasibility_lr_sba_alpha_0.5_app}
    \end{subfigure}
    \begin{subfigure}[b]{0.45\textwidth}
        \centering
        \includegraphics[width=\textwidth]{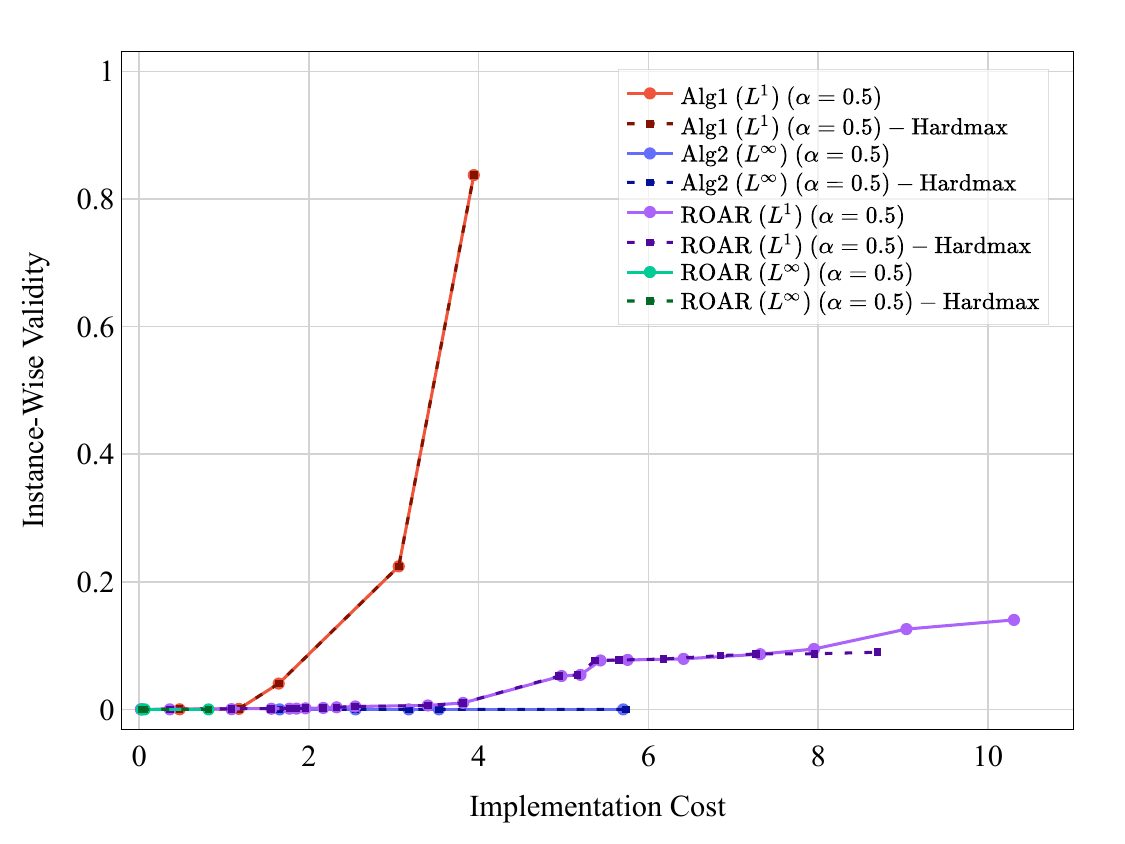}
        \caption{SBA Dataset, Neural Network}
        \label{fig:feasibility_nn_sba_alpha_0.5_app}
    \end{subfigure}
    \caption{The frontier of the trade-off between validity and implementation cost on the Small Business Administration dataset after post-processing with $\alpha=0.5$. The left and right columns correspond to logistic regression and neural network models. Each row corresponds to a different dataset. In each subfigure, curves show the trade-off for different algorithms. For each algorithm, solid and dashed lines depict the performance before and after hardmax post-processing is applied.
    \label{fig:feasibility_alpha_0.5_app}}
\end{figure*}

\end{document}